\documentclass[10pt,journal,compsoc]{IEEEtran}
\ifCLASSOPTIONcompsoc
  \usepackage[nocompress]{cite}
\else
  \usepackage{cite}
\fi
\ifCLASSINFOpdf
\else
\fi
\usepackage{url}
\usepackage{graphicx}
\usepackage{amsmath,amssymb} % define this before the line numbering.
\usepackage{amsthm}
\usepackage{color}

\usepackage[ruled]{algorithm}
\usepackage{algpseudocode}
\usepackage{caption}

\newcommand{\R}{\mathbb{R}}

\newcommand{\Z}{\mathbb{Z}}

\newcommand{\comment}[1]{ }

\newtheorem{defn}{Definition}
\newtheorem{prop}{Proposition}

\begin{document}
\title{SurfCut: Surfaces of Minimal Paths From Topological Structures}

\author{Marei Algarni and Ganesh Sundaramoorthi 
  \thanks{M.~Algarni and G.~Sundaramoorthi are with KAUST (King
    Abdullah University of Science \& Technology), Thuwal, Saudi
    Arabia. Email: \{marei.algarni, ganesh.sundaramoorthi\}@kaust.edu.sa } 
}

\IEEEtitleabstractindextext{%
\begin{abstract}
  We present \emph{SurfCut}, an algorithm for extracting a smooth,
  simple surface with an unknown 3D curve boundary from a noisy 3D
  image and a seed point. Our method is built on the novel observation
  that certain ridge curves of a function defined on a front
  propagated using the Fast Marching algorithm lie on the surface. Our
  method extracts and cuts these ridges to form the surface boundary.
  Our surface extraction algorithm is built on the novel observation
  that the surface lies in a valley of the distance from Fast
  Marching. We show that the resulting surface is a collection of
  minimal paths. Using the framework of cubical complexes and Morse
  theory, we design algorithms to extract these critical structures
  robustly.  Experiments on three 3D datasets show the robustness of
  our method, and that it achieves higher accuracy with lower
  computational cost than state-of-the-art.

\comment{Our method requires only a seed point, rather than a 3D
  boundary curve as in existing methods.}
\end{abstract}

}

\maketitle

\IEEEdisplaynontitleabstractindextext

% \fi
\IEEEpeerreviewmaketitle

\IEEEraisesectionheading{\section{Introduction}\label{sec:introduction}}

\IEEEPARstart{M}{inimal} path methods \cite{cohen1997global}, built on
the Fast Marching algorithm \cite{sethian1996fast} (see also
\cite{tsitsiklis1995efficient}), have been widely used in computer
vision. They provide a framework for extracting continuous curves from
possibly noisy images. For instance, they have been used in edge
detection \cite{kaul2012detecting} and object boundary detection
\cite{mille2015combination}, mainly in interactive settings as they
typically require user defined seed points. Because of their ability
to provide continuous curves, robust to clutter and noise in the
image, generalizations of these techniques to extract the equivalent
of edges in 3D images, which form surfaces, have been attempted
\cite{benmansour2009single,ardon2005new}.  These methods apply to
extracting a surface with a boundary that forms a curve, possibly in
3D, which we call a \emph{free-boundary}. Extraction of surfaces with
free-boundary is important because many edges form these surfaces, and
edges are fundamental structures that are prevalent in images. Some
applications include medical datasets (e.g., lung fissures, walls of
heart ventricles) \cite{grady2010minimal} and scientific imaging
datasets (e.g., fault surfaces in seismic images, an important problem
in the oil industry) \cite{hale2013methods}. In
\cite{grady2010minimal} an alternative method to extract such
surfaces, based on the theory of minimal surfaces
\cite{sullivan1990crystalline}, is provided. However, existing
approaches to surface extraction for surfaces with free-boundary have
a limitation - they require the user to provide the boundary of the
surface or other user laborious input.

In this paper, we use the Fast Marching algorithm and techniques from
computational topology to create an algorithm for extracting the
boundary of a surface from a 3D image and a single seed point, and an
algorithm to extract the surface.  Our main idea is to use Fast
Marching to ``smooth'' a local (possibly noisy) likelihood map of the
surface in a way that is guaranteed to preserve locations of critical
structures, and then extract the structures with methods, built from
computational topology, that guarantee correct topology. We show that
the resulting structures correspond to the surface of interest, and
the surface is a collection of minimal paths. Our method is applicable
to any imaging modality, and can be used to extract any simple surface
with boundary from an image that contains noisy local measurements
(possibly an edge map) of the surface. We demonstrate the method on
two applications - fault extraction from seismic images, and lung
fissure extraction from CT.

\comment{
We validate our
algorithm on seismic images for extracting fault surfaces, which form
surfaces with free-boundaries. This has wide ranging applications in
the oil industry \cite{hale2012fault}. Although we validate our method
with such images, 
}

Our contributions are: {\bf 1.}~We introduce the first algorithm, to
the best of our knowledge, to extract a closed 3D space curve forming
the boundary of a surface from a single seed point.  It is based on
extracting critical structures from a distance produced by Fast
Marching. {\bf 2.}~We introduce a new algorithm, based on extracting a
critical structure of the FM distance, to extract a surface given its
boundary and a noisy image. It produces a topologically simple surface
whose boundary is the given space curve. The surface is shown to be
formed from minimal paths. Both boundary and surface extraction have
$O(N \log N)$ complexity, where $N$ is the number of pixels. {\bf 3.}
We provide a fully automated algorithm using the algorithms above to
extract all such surfaces from a 3D image. {\bf 4.} We test our method
on challenging datasets, and we quantitatively out-perform comparable
state-of-the-art in free-boundary surface extraction.

\subsection{ Related Work}

\subsubsection{Surface Extraction}
Active surface methods
\cite{caselles1997geodesic,yezzi1997geometric,chan2001active}, based
on level set methods \cite{osher1988fronts}, their convex counterparts
\cite{pock2008convex}, graph cut methods
\cite{boykov2001interactive,rother2004grabcut}, and other image
segmentation methods partition the image into volumes and the surfaces
enclose these volumes. These methods have been used widely in
segmentation. However, they are not applicable to our problem since we
seek a surface, whose boundary is a 3D curve, that does not enclose a
volume nor partition the image.

Our method uses the Fast Marching (FM) Method
\cite{sethian1996fast}. This method propagates an initial surface
(e.g., a seed point) in an image in the direction of the outward
normal with speed proportional to a function defined at each pixel of
the image. The end result is a distance function, which gives the
shortest path length (measured as a path integral of the inverse
speed) from any pixel to the initial surface. The method is known to
have better accuracy than discrete algorithms based on Dijkstra's
algorithm. Shortest paths from any pixel to the initial surface can be
obtained from the distance function \cite{cohen1997global} (see also
\cite{ulen2015shortest}). This has been used in 2D images to compute
edges in images. A limitation of this approach is that it requires the
user to input two points - the initial and ending point of the
edge. In \cite{kaul2012detecting}, the ending point is automatically
detected. These methods are not directly applicable to extracting a
surface forming an edge in 3D.

Attempts have been made to use minimal paths to obtain edges that form
a surface. In \cite{ardon2005new,ardon2007new}, minimal paths are used
to extract a surface edge with a cylindrical topology, a topology
different from our problem. The user inputs the two boundary curves
(in parallel planes) of the cylinder and minimal paths joining the two
curves are computed conveniently using the solution of a regularized
transport partial differential equation. Surface extraction with less
intensive user input was attempted in
\cite{benmansour2009single}. There, a patch of a sheet-like surface is
computed with a user provided seed point and a bounding box, with the
assumption that the patch slices the box into two pieces. The
algorithm extracts a curve that is the intersection of the surface
patch with the bounding box using the distance function to the seed
point obtained with Fast Marching. Once this boundary curve is
obtained, the patch is computed using \cite{ardon2007new}. The obvious
drawbacks of this method are that only a patch of the desired surface
is obtained, and a bounding box, which may be cumbersome to obtain,
must be given by the user.

Another approach to obtaining a surface along image edges from its
boundary is minimal surfaces \cite{grady2006computing,
  grady2010minimal}. The minimal weighted area simple surface
interpolating the boundary is obtained by solving a linear
program. Faster implementations for minimal surfaces are explored in
\cite{grady2010minimal}, using algorithms for the minimum cost network
flow problem (e.g., \cite{goldberg1997efficient, kovacs2015minimum,
  brunsch2015smoothed, ford2015flows}). This significantly speeds up
the approach, although it requires an initial surface, and the
algorithm is dependent on it.  The main drawback of minimal surfaces
is that the user must input the boundary of the surface, which our
method addresses. It is also computationally expensive as we show in
experiments.

\comment{It is argued that minimal surfaces are more
natural extensions of the 2D shortest path problem to 3D.}

An approach for surface extraction that does not require user input is
\cite{schultz2010crease}. There, a matrix based on the local smoothed
Hessian matrix of the likelihood is used to generate a ridge in the
image near the desired surface. Then surface normals based on the
matrix are computed, which are used to generate several surfaces. This
method is convenient since it is fully automated. This approach has
been tailored to seismic images for extracting fault surfaces
\cite{hale2013methods}, and it is the state-of-the-art in that
field. Our method also smooths the likelihood, but in a way that
preserves locations of critical structures, resulting in a more
accurate surface. Also, our extraction of the critical structures, by
using tools from computational topology, guarantees a simple surface
topology.

\comment{However, it is sensitive to
noise as we show in experiments.}

\subsubsection{Computational Topology}
Our method is a discrete algorithm and is based on the framework of
cubical complexes
\cite{kovalevsky1989finite,kaczynski2006computational}. This framework
allows for performing operations analogous to topological operations
in the continuum. It has been used for thinning surfaces in 3D based
on their geometry \cite{chaussard2009surface} to obtain skeletons (or
medial representations
\cite{siddiqi1999shock,sebastian2004recognition, siddiqi2008medial}) of
geometrical shapes. This theory guarantees correct topology of
extracted structures. Our novel algorithms use concepts from cubical
complex theory. In contrast to \cite{chaussard2009surface}, our method
is designed to robustly extract ridges of a \emph{function} or data
defined on a surface (defined by Fast Marching), rather than
geometrical properties of a surface.

Our method uses a topological construction called the \emph{Morse
  complex} \cite{zomorodian2009topology} from Morse theory to extract
ridges on a manifold. There is a large literature that aims to compute
the Morse complex and an extension called the Morse-Smale Complex,
from discrete data \cite{edelsbrunner2001hierarchical,
  edelsbrunner2003morse, gyulassy2008practical,
  gyulassy2012computing}. Roughly, these Morse complexes describe the
behavior of the gradient flow of a function within regions. We use
cubical complexes to construct the Morse complex since they are
naturally suited for image data, defined on grids. Conceptually, our
algorithm for the Morse complex appears similar to
\cite{gyulassy2008practical}, even though the technical details and
notions of discrete topology are different. Our contribution is not to
provide another algorithm for the Morse complex, but to use the Morse
complex for the purpose of free-boundary surface extraction from
images.

\subsubsection{Extensions to Conference Paper}
A preliminary version of this manuscript has appeared in
\cite{algarni2016surfcut}. In this version, we have derived the
theoretical foundations: 1) we provide analytical arguments to show
that our algorithms correctly capture the surface of interest by
relating it to constructions in Morse theory, and 2) we relate the
surface extracted to minimal paths by showing the surface is formed by
collections of minimal paths, thus inheriting known regularity
properties from such paths. We extended our ridge extraction algorithm
to better deal with extraneous structures. We have also extended our
experiments to more datasets, including medical data, and extended the
quantitative comparison to minimal surface approaches.

\subsection{Overview of Method}

\def\fHeightO{1.2in}
\begin{figure}
  \centering
  {\scriptsize
    \begin{tabular}{c@{\hspace{0.03in}}c}
      Front Propagation & Ridge Extraction \\
      \includegraphics[totalheight=\fHeightO]{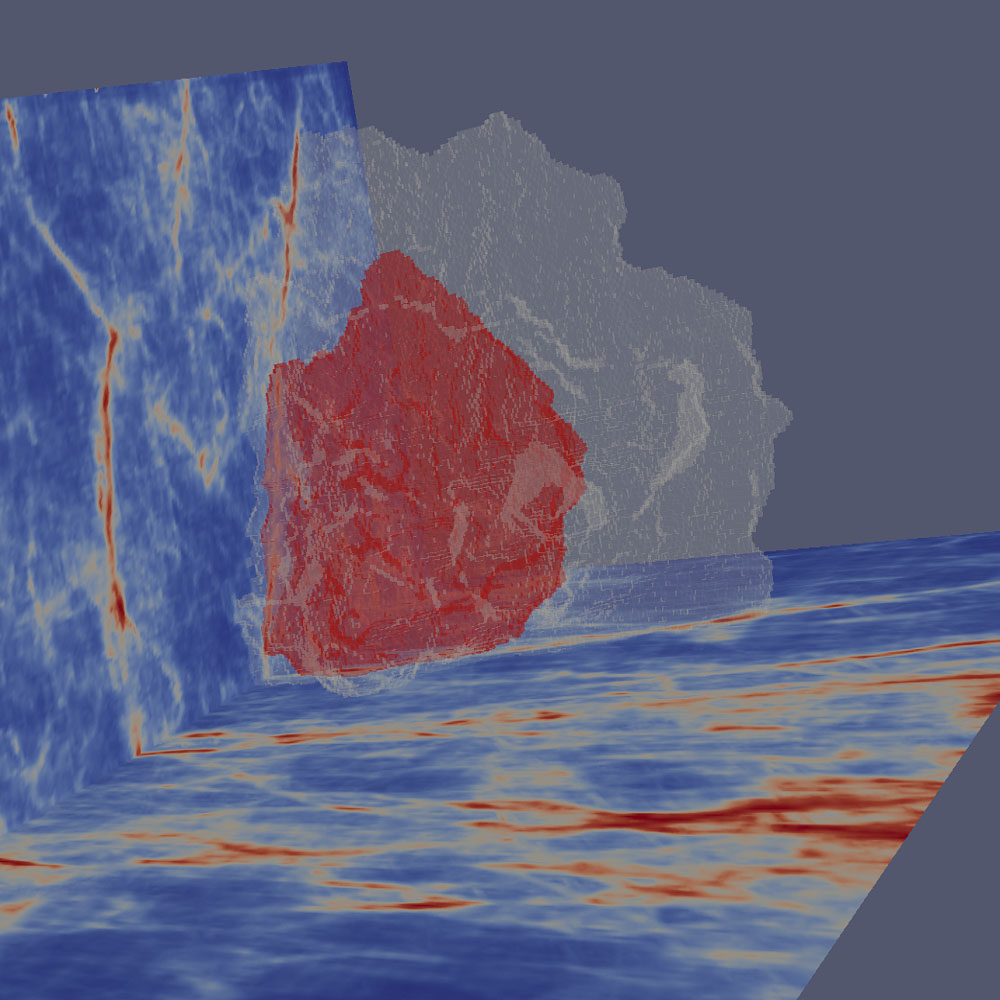} &
      \includegraphics[totalheight=\fHeightO]{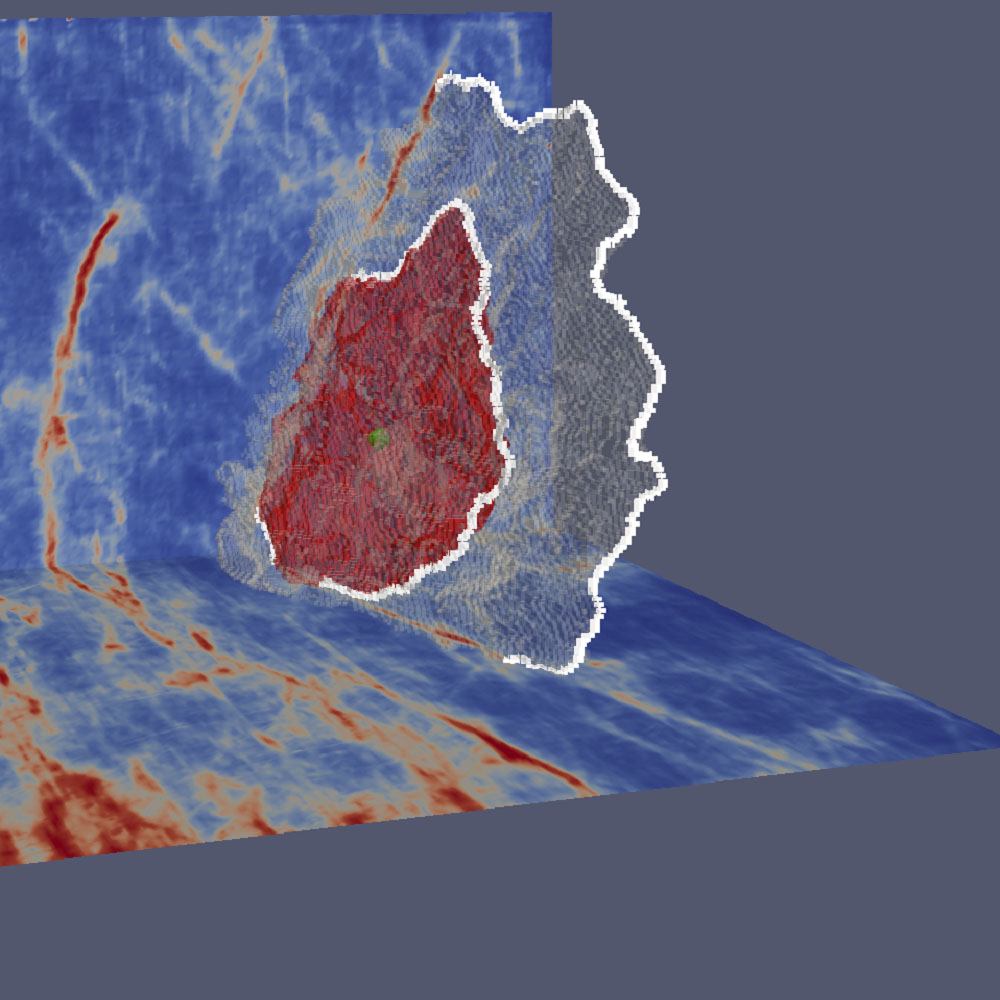} 
    \end{tabular}\\
    \begin{tabular}{c@{\hspace{0.03in}}c}
      Boundary Extraction & Surface Extraction \\
      \includegraphics[clip,trim=0 10 0 20,totalheight=\fHeightO]{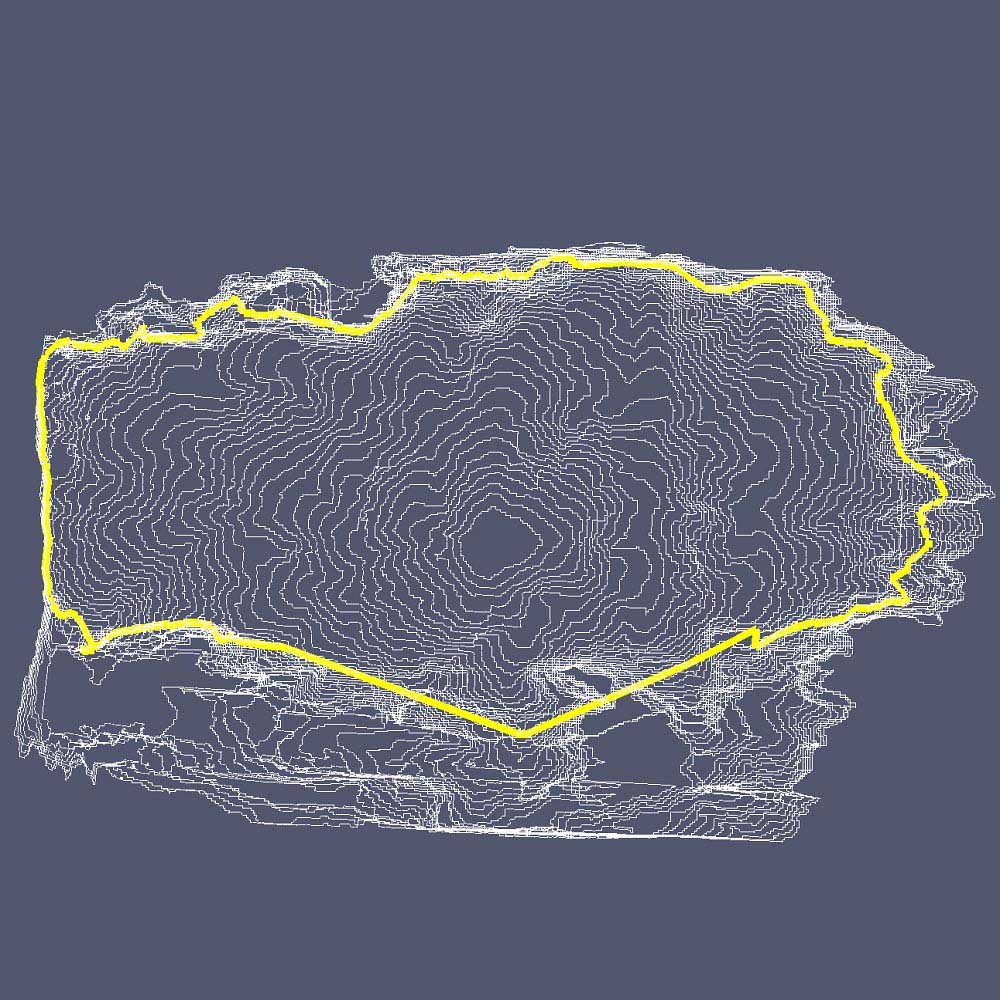} &
      \includegraphics[totalheight=\fHeightO]{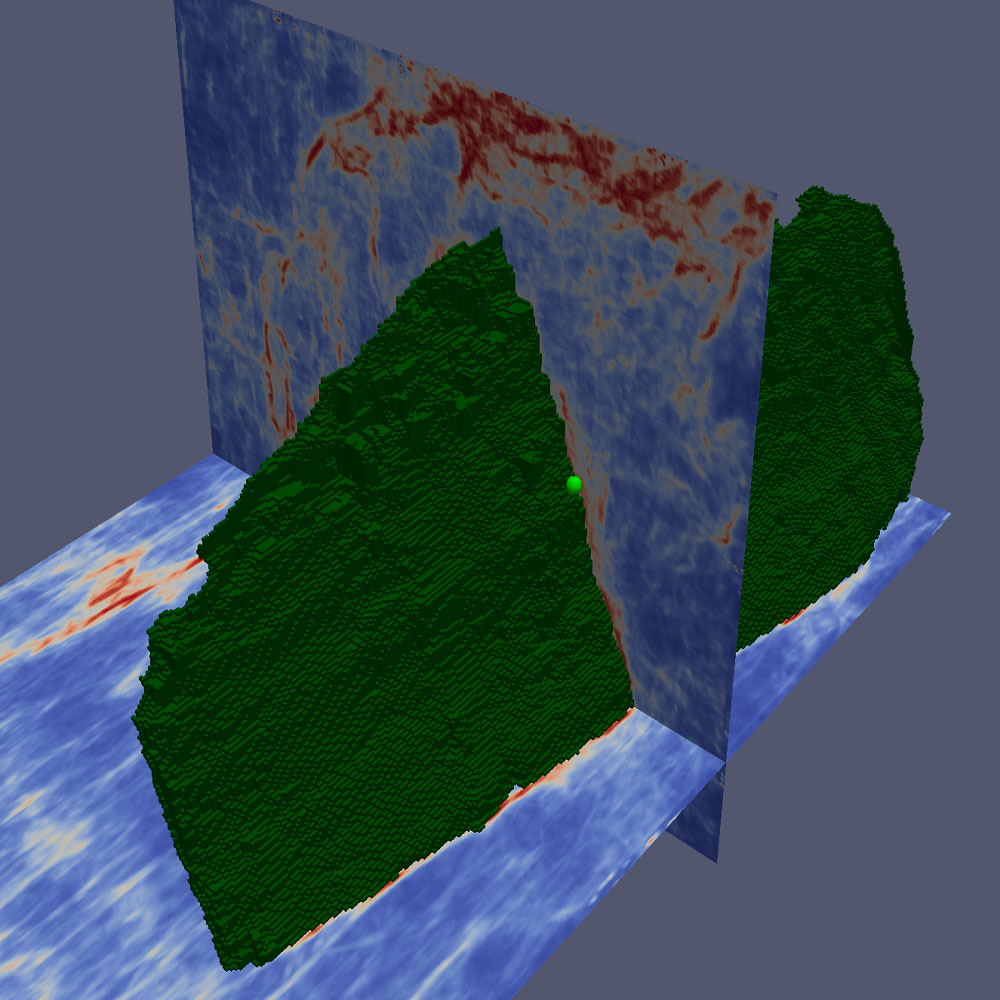} 
    \end{tabular}\\
}
\caption{ {\bf Overview of SurfCut}. Starting from a seed point on the
  surface, a front is propagated (top, left), ridges are extracted
  (top, right), a cut is performed forming the boundary (bottom,
  left), and the surface is extracted..}
\label{fig:overview}
%\vspace{-0.75cm}
\end{figure}

Our algorithm consists of the following steps (see
Figure~\ref{fig:overview}): {\bf i) Weighted Distance to Seed Point
  Computation}: From a given seed point on the surface, the Fast
Marching algorithm is used to propagate a front to compute shortest
path distance from any point in the image to the seed point
(Section~\ref{subsec:fast_marching}). {\bf ii) Ridge Curve
  Extraction}: At samples of the propagating front, the ridge curves
of the Euclidean path distance of minimal paths to the seed point are
computed (Section~\ref{subsec:curve_extraction}). These curves lie on
the surface of interest. {\bf iii) Surface Boundary Detection}: At
snapshots, a graph is formulated from curves from the previous step,
and is cut along locations where the Euclidean distance between points
on adjacent curves are small, resulting in the outer boundary of the
surface (Section~\ref{subsec:cut}). {\bf iv) Surface Extraction}:
Finally, the desired surface with boundary obtained from the last step
is computed (Section~\ref{subsec:surface_extraction}).

Our method requires notions from topology, which we review next. We
then proceed to our algorithm.

\section{Topological Preliminaries}

In this section, we present theory and notions from topology and
computational topology that will be relevant in subsequent sections in
designing and justifying our novel algorithms for surface extraction.

\subsection{Topological Structures}

Our algorithms extract topological structures from functions defined
on the image domain and manifolds embedded in the image. We give
formal definitions for these topological structures, \emph{ridges} and
\emph{valleys}, and then the Morse complex.

\subsubsection{Critical Structures}
Intuitively, ridge points of a function defined on a manifold
correspond to local maxima when restricted to sub-spaces of directions
rather than the whole space of possible directions. Similarly, valley
points correspond to local minima of a function when restricted to
sub-spaces of directions. We now give more formal definitions.  We
consider functions $h : M \subset \R^n \to \R$, defined on a $n-1$
dimensional manifold. For a point $x\in M$, we denote $T_x M$ to be
the tangent space of $M$ at $x$, which consists of all valid
directions at the point $x$ on $M$. We first define the \emph{critical
  points} of $h$ as the points $p$ on $M$ where the gradient vanishes,
i.e., $\nabla h(p) = 0$. Note that the gradient refers to the
\emph{intrinsic} gradient $\nabla h(x)\in T_x M$, i.e., it is defined
by the relation $d h(x)\cdot v = \nabla h(x) \cdot v$ for all
$v\in T_xM$ where $d h(x)\cdot v$ denotes the directional derivative
of $h$ at $x$ and the right hand side is the usual Euclidean dot
product. Ridges and valleys are formally defined by
\cite{eberly1994ridges} as follows.
 
\begin{defn}[Ridge and Valley] \label{defn:ridge}
  Let $h : M \subset \R^n\to \R$ where $M$ is an $n-1$ dimensional
  manifold. Let $\lambda_1 \leq \cdots \leq \lambda_{n-1}$ and
  $e_1, \ldots, e_{n-1} \in T_x M$, be eigenvalues and eigenvectors of
  the Hessian $Hh(x)$ at $x\in M$. Let $k<n-1$.
  \begin{itemize}
  \item A point $x\in M$ is a {\bf $n-1-k$ dimensional ridge
      point} of $h$ if $\lambda_{k} < 0$ and $\nabla h(x) \cdot
    e_m = 0$ for $m=1,\ldots, k$.
  \item A point $x\in M$ is a {\bf $n-1-k$ dimensional valley
      point}  of $h$ if $\lambda_{n-k} > 0$ and $\nabla h(x) \cdot
    e_m = 0$ for $m=n-k,\ldots, n-1$.
  \end{itemize}
\end{defn}

The conditions above ensure zero derivatives in a subspace of
directions, and the conditions on the Hessian ensure the function is
concave (for ridges) and convex (for valleys) in the appropriate
subspace. Differentiability in the definition is not needed, and there
are more generic conditions for continuous functions, e.g., that the
function value is higher at the ridge point than other points in a
sub-neighborhood corresponding to a subspace of directions.

\begin{figure}
  \centering
  \includegraphics[totalheight=1.6in]{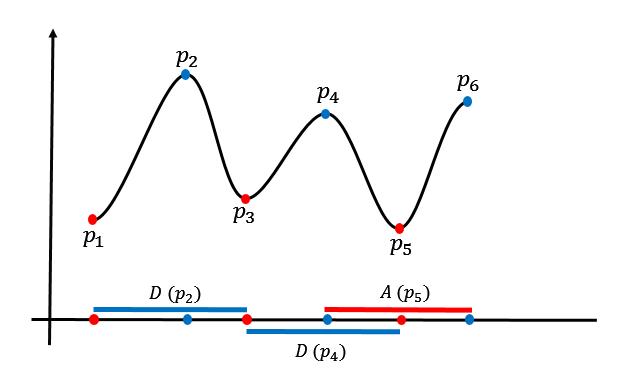}
  \caption{Illustration of some ascending and descending manifolds of a
    one-dimensional function.}
  \label{fig:morse_manifolds}
\end{figure}

\comment{
As that
literature can be fairly technical, we convey the main ideas without
excessive technicalities.
}

\subsubsection{Morse Complex}
Our algorithms for extracting the previous structures do not directly
use the differential definitions above, as they are not robust to
noise in the image. We will design our algorithms based on topological
constructions in \emph{Morse theory} \cite{milnor2016morse}. We
introduce basic notions from that literature, and the exact relation
to ridges and valleys will be left to subsequent sections when we
specify our algorithms. We define the \emph{ascending} and
\emph{descending} manifolds of a critical point as all points on a
path along the negative (positive, respectively) gradient direction
that leads to the given critical point. A path on a manifold $M$ is a
mapping $\gamma : [0,\infty) \to M$.  A gradient path is specified by
the differential equation $\gamma'(t) = \pm \nabla h(\gamma(t))$,
where $h$ is some function defined on $M$. Formally, the ascending and
descending manifolds of a critical point $p$ of $h$ are defined as
follows \cite{zomorodian2009topology}.
\begin{defn}[Ascending and Descending Manifolds]
  Let $h : M\to \R$ be a function and $p$ be a critical point of
  $h$. The {\bf ascending manifold} at $p$ is
  \begin{multline}
    A(p) = \{ x\in M \,:\, \mbox{ there exists $\gamma : [0,\infty)\to M$
      such that } \\ 
    \gamma(0) = x, \gamma(\infty) = p, \gamma'(t) = -\nabla
    h(\gamma(t)) \}.
  \end{multline}
  The {\bf descending manifold} at $p$ is
  \begin{multline}
    D(p) = \{ x\in M \,:\, \mbox{ there exists $\gamma : [0,\infty)\to M$
      such that } \\ 
    \gamma(0) = x, \gamma(\infty) = p, \gamma'(t) = \nabla
    h(\gamma(t)) \}.
  \end{multline}
\end{defn}
For instance, consider the function $h : \R^2 \to \R$ defined by
$h(x,y) = x^2+y^2$. Its ascending manifold at the critical point $0$
is $A(0) = \R^2$ as all negative gradient paths lead to the origin. Note also
that $D(0) = 0$. See Figure~\ref{fig:morse_manifolds} for a
visualization in the one-dimensional case.

The ascending manifolds of local minima decompose the manifold $M$
into disjoint sets. Similarly, the descending manifolds of all local
maxima decomposes the manifold $M$ into disjoint open sets. The latter
decomposition forms the \emph{Morse complex} of $h$, and the former is
the Morse complex of $-h$. We will use the Morse complex in future
sections.

\subsection{Cubical Complexes Theory}

We now introduce notions from cubical complex theory, which is the
basis for our algorithms in future sections. This theory defines
topological notions (and computational methods) for discrete data that
are analogous to topological notions in the continuum. The notion of
\emph{free pairs}, i.e., those parts of the data that can be removed
without changing topology of the data, is pertinent to our
algorithms. Since the algorithms we define require the extraction of
lower dimensional structures (ridge curves from surfaces, and valley
surfaces from volumes), it is important that the algorithms are
guaranteed to produce lower dimensional structures with correct
topology. The theory of cubical complexes (e.g.,
\cite{kaczynski2006computational,chaussard2009surface}) guarantees
such lower dimensional structures are generated with homotopy
equivalence to the original data.

Our data (either a curve, surface or volume) will be represented
discretely by a cubical complex. A cubical complex consists of basic
elements, called \emph{faces}, of $d$-dimensions, e.g., points
(0-faces), edges (1-faces), squares (2-faces) and cubes
(3-faces). Formally, a $d$-\emph{face} is the cartesian product of $d$
intervals of the form $(a,a+1)$ where $a$ is an integer. We can now
define a cubical complex (see Fig.~\ref{fig:Cubical_Complex}) as
follows.

\begin{defn}
  A $d$-dimensional {\bf cubical complex} is a finite set of faces of
  $d$-dimensions and lower such that every sub-face of a face in the
  set is contained in the set.
\end{defn}

\def\fHeightComplexFree{0.7in}
\begin{figure}
  \centering
  {\scriptsize
    \begin{tabular}{c@{\hspace{0.03in}}c@{\hspace{0.03in}}c@{\hspace{0.03in}}c}
      cubical complex & not cubical complex & free 1D faces & free 2D faces \\
      \includegraphics[totalheight=\fHeightComplexFree]{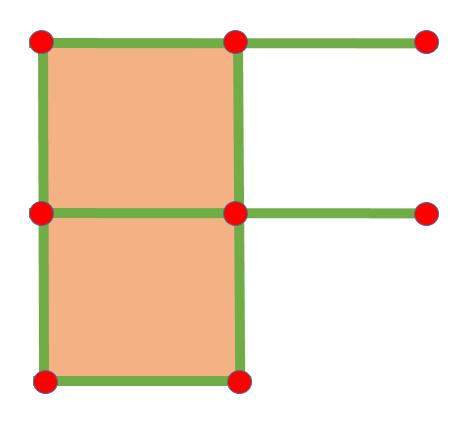} &
      \includegraphics[totalheight=\fHeightComplexFree]{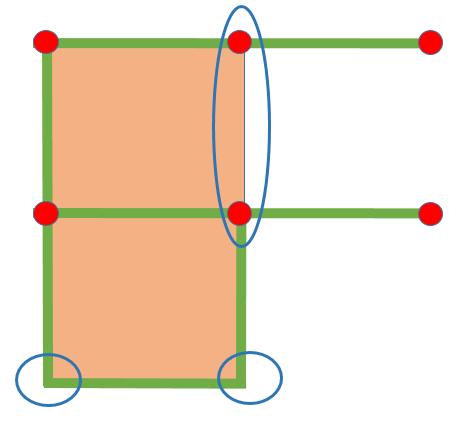} &
      \includegraphics[totalheight=\fHeightComplexFree]{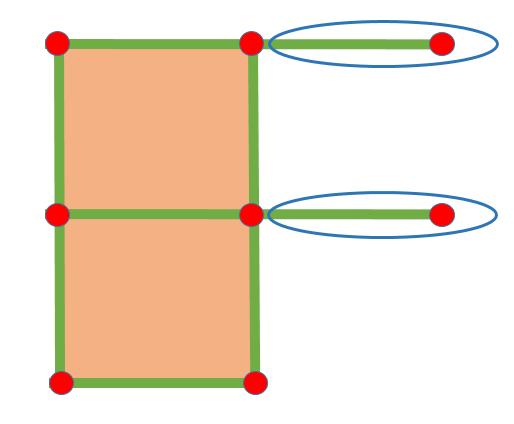} &
      \includegraphics[totalheight=\fHeightComplexFree]{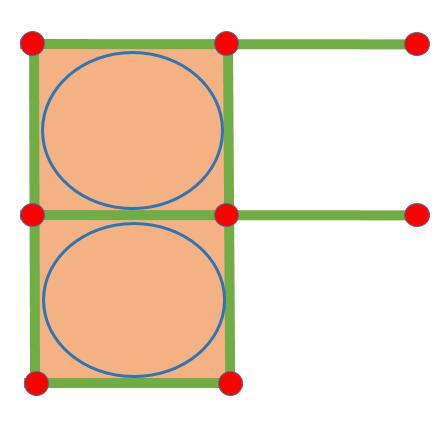}
    \end{tabular}\\
}
\caption{[Left two images]: Illustration of faces that
  form a cubical complex (left) and faces that do not form a
  cubical complex (0,1,2-faces are marked in red, green and orange).
  The missing 1-face and 0-faces circled in blue on the right are
  not in the complex, but they are sub-faces of other faces in the
  set. [Right two images]: Example of 1-face, 0-face free pairs,
  and 2-face, 1-face free pairs (circled in blue).}
\label{fig:Cubical_Complex}
\end{figure}

Our algorithms consist of simplifying cubical complexes by an
operation that is analogous to the continuous topological operation
called a \emph{deformation retraction}, i.e., the operation of
continuously shrinking a topological space to a subset. For example, a
punctured disk can be continuously shrunk to its boundary
circle. Therefore, the boundary circle is a deformation retraction of
the punctured disk, and the two are said to be \emph{homotopy
  equivalent}. We are interested in an analogous discrete operation,
whereby faces of the cubical complex can be removed while preserving
homotopy equivalence.  \emph{Free faces} (see
Fig.~\ref{fig:Cubical_Complex}), defined in cubical complex theory,
can be removed simplifying the cubical complex, while preserving a
discrete notion of homotopy equivalence. These are defined formally
as:

\begin{defn}
  Let $X$ be a cubical complex, and let $f, g \subset X$.
  
  $g$ is a {\bf proper face} of $f$ if $g\neq f$ and $g$ is a sub-face of $f$.

  $g$ is {\bf free} for $X$, and the pair $(g, f)$ is a {\bf free
    pair} for $X$ if $f$ is the only face of $X$ such that $g$ is a
  proper face of $f$. If $g$ is not free, it is called {\bf isthmus}.
\end{defn}

The definition provides a constant-time operation to check whether a
face is free. For example, if a cubical complex $X$ is a subset of the
3-dim complex formed from a 3D image grid, a 2-face is known to be
free by only checking whether only one 3-face containing the 2-face is
contained in $X$.

In the next section, we construct cubical complexes for the evolving
front produced from the Fast Marching algorithm, and retract this
front by removing free faces to obtain a lower dimensional ridge curve
that lies on the surface that we wish to obtain. We also retract a
volume to obtain a valley, which forms the surface of interest.

\section{Surface Boundary Extraction}

In this section, we present our algorithm for extracting the boundary
curve of a free-boundary surface from a possibly noisy local
likelihood map of the surface defined in a 3D image. The algorithm
consists of retracting the fronts (closed surfaces) generated by the
Fast Marching algorithm to obtain ridge curves on the surface of
interest. We therefore review Fast Marching in the first sub-section
before defining our novel algorithms for surface extraction.

\subsection{Fronts Localized to the Surface With Fast Marching}

\label{subsec:fast_marching}

We use the Fast Marching Method \cite{sethian1996fast} to generate a
collection of fronts that grow from a seed point and are localized to
the surface of interest. We denote by
$\phi : \Omega \subset \R^3 \to \mathbb{R}^+$, a possibly noisy
function defined on each pixel of the given image grid. It has the
property that (in the noiseless situation) a small value of $\phi(x)$
indicates a high likelihood of the pixel $x$ belonging to the surface
of interest.

Fast Marching solves, with complexity $O(N\log N)$ where $N$ is the
number of pixels, a discrete approximation to $U : \Omega \subset \R^3
\to \mathbb{R}^+$, the solution of the eikonal equation:
\begin{equation} \label{eq:eikonal}
  \begin{cases}
    |\nabla U(x)| = \phi(x) & x\in \Omega \backslash \{ p \}
    \\
    U(p) = 0 &
  \end{cases}
\end{equation}
where $\nabla$ denotes the spatial gradient (partials in all
coordinate directions), and $p \in \Omega$ denotes an initial seed
point. For our situation, $p$ will be required to lie somewhere on the
surface of interest. The function $U$ at a pixel $x$ is the weighted
minimum path length along any path from $x$ to $p$, with weight
defined by $\phi$. $U$ is called the weighted distance. Minimal paths
can be recovered from $U$ by following the gradient descent of $U$
from any $x$ to $p$. A front (a closed surface, which we hereafter
refer to as a front to avoid confusion with the free-boundary surface)
evolving from the seed point at each time instant is equidistant (in
terms of $U$) to the seed point and is iteratively approximated by
Fast Marching. As noted by \cite{cohen1997global}, a positive constant
added to the right hand side of \eqref{eq:eikonal} may be used to
induce smoothness of paths. The front, evolving in time, moves in the
outward normal direction with a speed proportional to
$1/\phi(x)$. Fronts can be alternatively obtained by thresholding $U$
at the end of Fast Marching. The solution of \eqref{eq:eikonal} is
continuous, and can be approximated as smooth since the solution is a
viscosity solution \cite{crandall1983viscosity}, and so a limit of
smooth functions.

\subsection{Contours on the Surface from Front Ridges}

\label{subsec:curve_extraction}

If we choose the seed point $p$ to be on the free-boundary surface of
interest, the front generated by Fast Marching will travel the fastest
when $\phi$ is small (i.e., along the surface) and travel slower away
from the surface, and thus the front is elongated along the surface at
each time instant (see Figure~\ref{fig:fmm-front}). Our algorithm is
based on the following observation: points along the front at a time
instant that have traveled the furthest (with respect to Euclidean
path length), i.e., traveled the longest time, compared to nearby
points, lie on the surface of interest. This is because points
traveling along locations where $\phi$ is low (on surface) travel
the fastest, tracing out paths that have large arc-length.

This property can be more easily seen in the 2D case (see
Figure~\ref{fig:fmm-front}): suppose that we wish to extract a curve
rather than a surface from a seed point, using Fast Marching to
propagate a front. At each time, the points on the front that travel
the furthest with respect to Euclidean path length lie on the 2D curve
of interest. This has been noted in 2D by \cite{kaul2012detecting}. In
3D (see Figure~\ref{fig:fmm-front}), we note this generalizes to
\emph{ridge points} of Euclidean minimal path length $U_E$ (defined
next) are on the surface of interest. The Euclidean minimal path
length $U_E$ is defined as follows. Define a front
$F = \partial \{ x\in \Omega \,:\, U(x) \leq D \}$ where $\partial$
denotes the boundary operator. The function $U_E : F\to \R^+$ is such
that $U_E(x) $ is the Euclidean path length of the minimal weighted
path (w.r.t to the distance $U$) from $x$ to $p$.

Computationally, $U_E$ is easy to obtain by keeping track of another
function $U_E : \Omega \to\R^+$ in Fast Marching for $U$. One follows
the ordered traversal of points according to Fast Marching in solving
for $U$, and simultaneously updates the value of $U_E$ based on a
discretization of \eqref{eq:eikonal} with $\phi$ chosen equal to
$1$. This gives the Euclidean length of minimal paths determined from
$U$.

The fact that ridge points lie on the surface is visualized in the
right of Figure~\ref{fig:fmm-front}. Points on the intersection of the
surface and the front are such that in the direction orthogonal to the
surface, the minimal paths have Euclidean lengths that decrease.  This
is because $\phi$ becomes large in this direction, thus minimal paths
travel slower in this region, so they have lower Euclidean path
length. Along the surface, at the points of intersection of the
surface and front, the path length may increase or decrease, depending
on the uniformity of $\phi$ on the surface. This implies points on the
intersection of the front and surface are ridge points of $U_E | F$.

\def\fHeightFront{1.0in}
\begin{figure}
\centering
\includegraphics[clip,trim=90 0 0 10,totalheight=\fHeightFront]{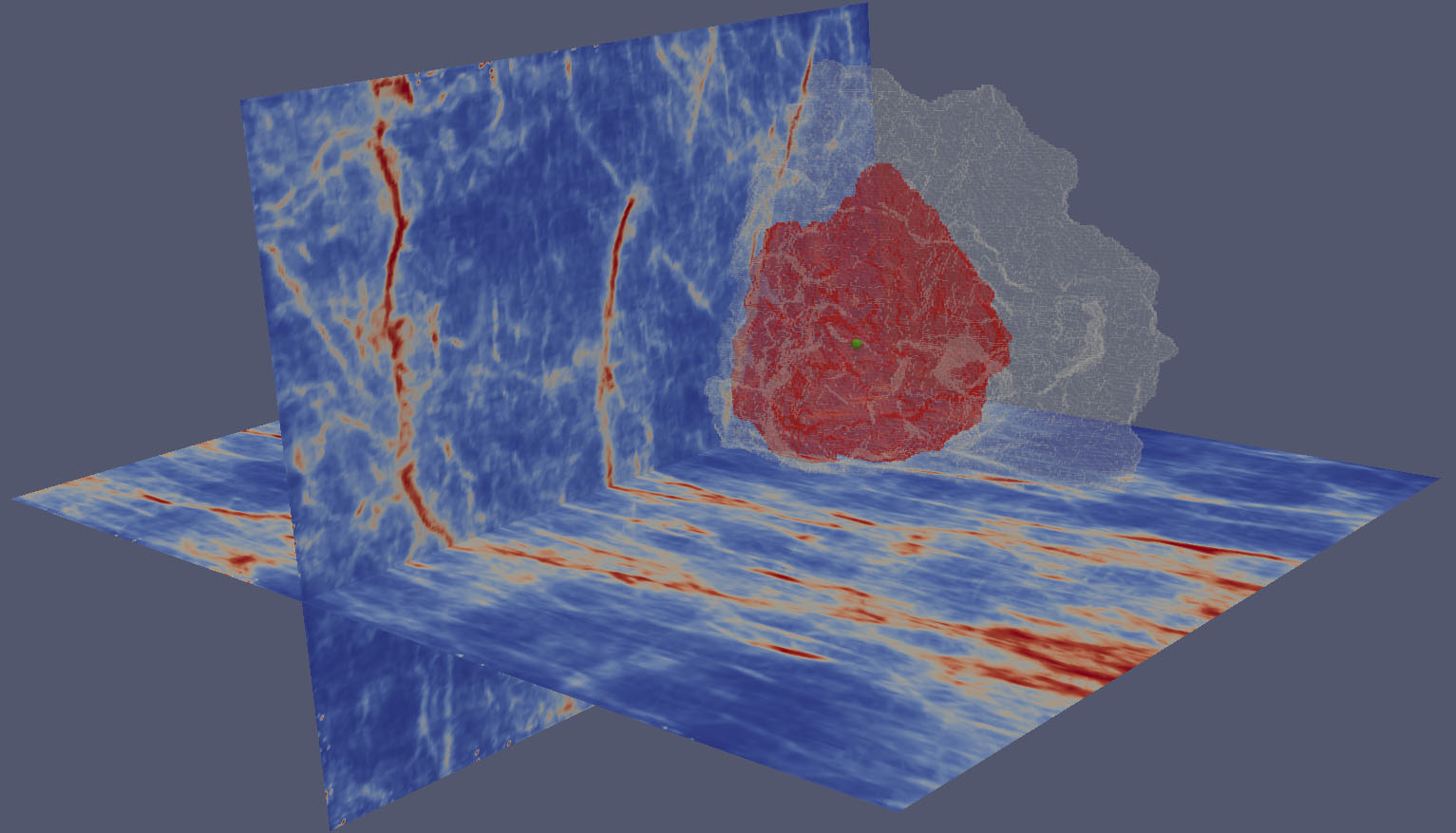}
\includegraphics[clip,trim=25 0 0 0,totalheight=\fHeightFront]{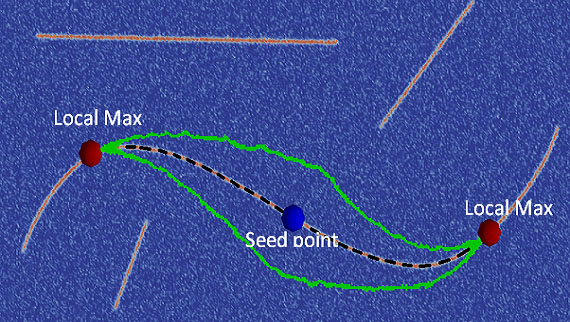} \\ \vspace{1mm}
\includegraphics[clip,trim=50 0 0 0,totalheight=\fHeightFront]{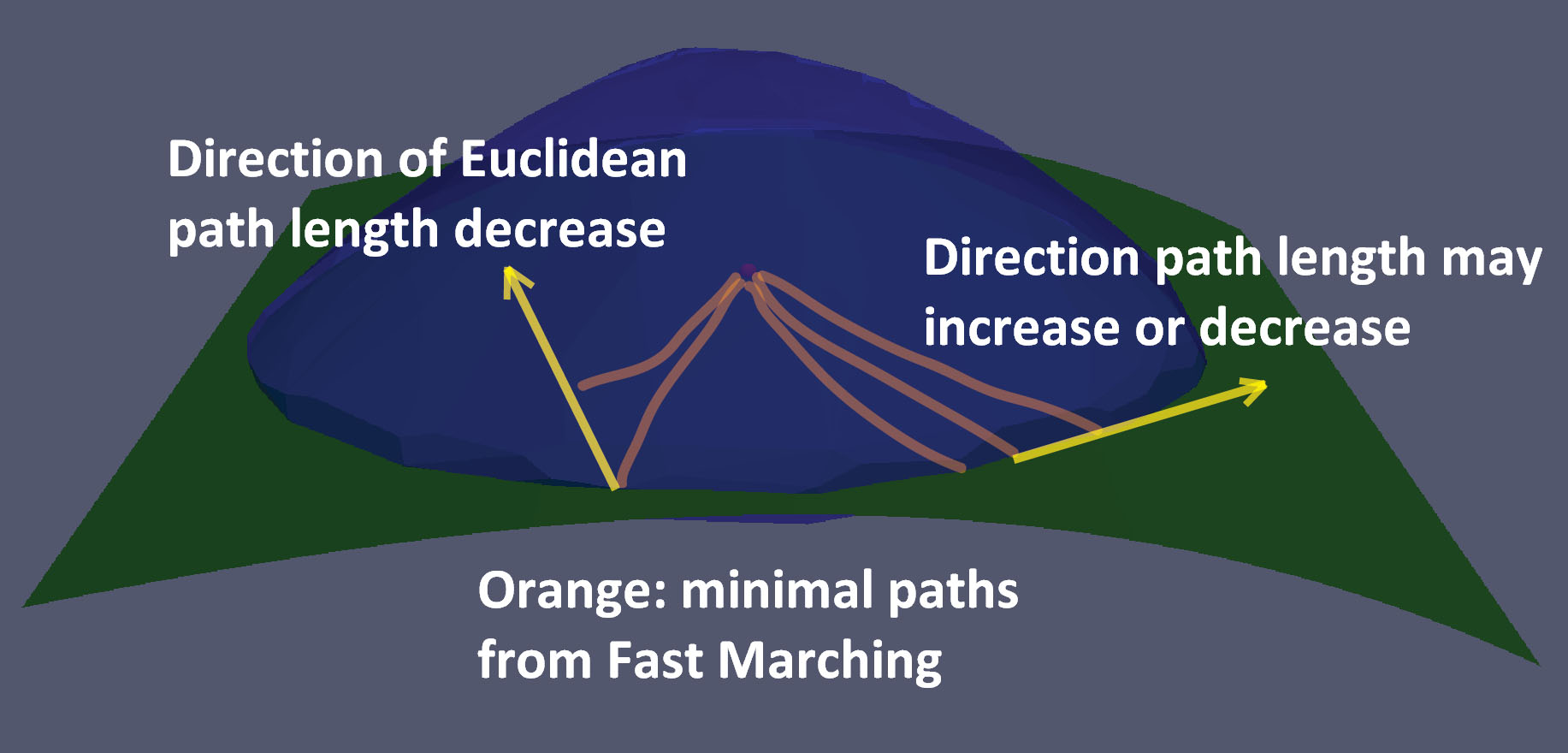}
\caption{[Top left]: The evolving Fast Marching (FM) front at two
  different time instances in orange and white. The function $1/\phi$
  evaluated at $x$ is the likelihood of surface passing through $x$,
  and is visualized (red - high values, and blue - low values).  The
  fronts are localized near the surface of interest. Ridge points of
  $U_E$, the Euclidean path length of minimal weighted paths, lie on
  the surface of interest.  [Top right]: This is more easily seen in 2D
  where the local maxima of the Euclidean path length (red balls) of
  minimal paths (dashed) are seen to lie on the curve of interest. The
  green contour is a snapshot of the front. [Bottom]: Schematic in 3D
  with front (blue), surface (green), and several minimal paths
  (orange). Orthogonal to the surface where the surface intersects the
  front, the Euclidean path length decreases. Along the surface, the
  path lengths may increase or decrease. This indicates ridge points.}
\label{fig:fmm-front}
\end{figure}

\comment{
points on the surface being
ridge points of $U_E$ defined on the front $F$ produced by Fast
Marching.
}

We now give an analytic argument that common points to the surface and
the front are ridge points.
\begin{prop} \label{prop:ridge_surface}
  Suppose $S\subset \Omega$ is a smooth surface and $p\in S$. Consider
  the front $F = \{ U = D \}$ and suppose $x\in S \cap F$ then $x$ is
  a ridge point of $U_E : F \to \R$, where $U_E(y)$ is defined as the
  Euclidean length of the minimal path from $y$ to $p$. We assume
  that locally $\phi$ is larger on $S$ than points not on $S$.
\end{prop}
\begin{proof}
  Let $x\in S \cap F$ and let $N$ be a normal vector to $S$ at $x$. We
  choose a neighborhood $V_x \subset \Omega$ around $x$ so that $S$ is
  approximately flat and  $\phi$ is approximated as 
  \[
    \phi(x) = 
    \begin{cases}
      K_1 & x\notin S \cap V_x \\
      K_2 & x \in S \cap V_x
    \end{cases},
  \]
  where $K_1 > K_2 >0$, which are constants. Let us consider a point
  $y = x + \varepsilon N$, where $\varepsilon>0$ is small, and the
  minimal path from $y$ to $p$ (see
  Figure~\ref{fig:ridge_prf_schematic}).
  \begin{figure}
    \centering
    {\scriptsize
      \includegraphics[totalheight=1.1in]{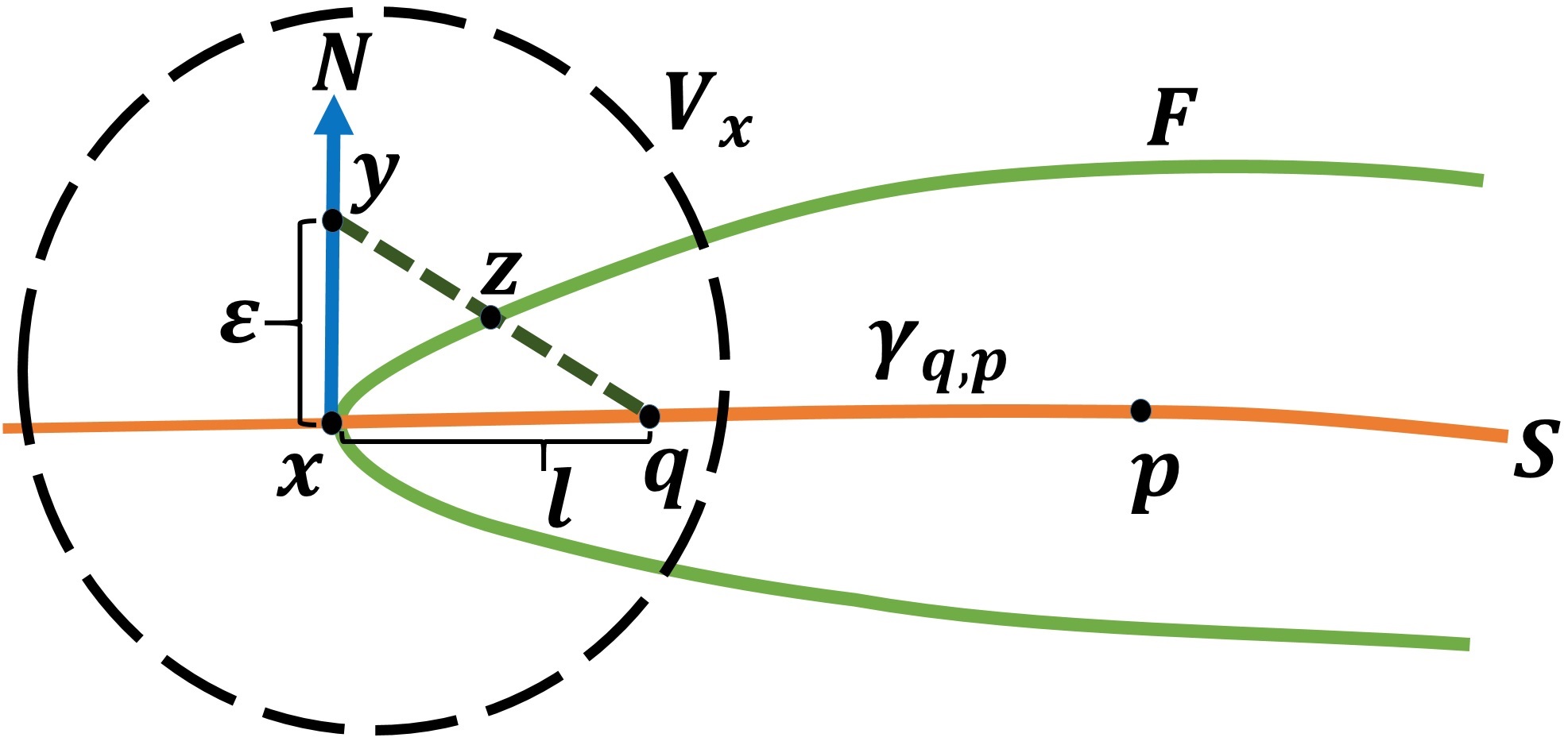}
    }
    \caption{Schematic of quantities in the proof of
      Proposition~\ref{prop:ridge_surface}.}
    \label{fig:ridge_prf_schematic}
    \comment{ that points on the
      front that lie on the surface are ridge points of Euclidean path
      length}
  \end{figure}
  We note that minimal paths within $V_x \backslash S$ will be
  straight lines as $\phi$ is uniform in that region. For
  $\varepsilon>0$ small enough, we can find $q \in S$ on the minimal
  path from $x$ to $p$ so that the minimal path from $y$ to $p$ is the
  straight line path from $y$ to $q$ appended to the minimal path from
  $q$ to $p$. We note that if we let $\ell = |x-q|$ then
  \[
    U(x) = U(q) + K_2 \ell.
  \]
  Also, 
  \[
    U(y) = U(q) + K_1 \sqrt{ \ell^2 + \varepsilon^2 },
  \]
  and any point $z$ on the line between $y$ and $q$ will have 
  \[
    U(z) = U(q) + t K_1 \sqrt{ \ell^2 + \varepsilon^2 }
  \]
  where $t\in (0,1)$. If we search for the point $z$ on the line
  between $q$ and $y$ on the front $F$, which has $U(z)=U(x)$, we find
  that
  \[
    t = \frac{K_2}{K_1} \frac{ \ell }{ \sqrt{ \ell^2 + \varepsilon^2 } }
    < 1.
  \]
  Therefore, the Euclidean length of the minimal path from $z$ to $p$
  is 
  \[
    U_E(z)  = \frac{K_2}{K_1} \ell + \mbox{len}(\gamma_{q,p})
    % \frac{ \ell }{ \sqrt{ \ell^2 + \varepsilon^2 } }
  \]
  where $\mbox{len}(\gamma_{q,p})$ is the length of the minimal path
  from $q$ to $p$. Notice this has less length than the path from $x$
  to $p$, which is $U_E(x) = \ell +
  \mbox{len}(\gamma_{q,p})$. Therefore, $U_E(z) < U_E(x)$. So moving
  in the direction $N$ along $F$ reduces the Euclidean length of
  minimal paths. This same argument holds for any $z$ within $V_x$
  along the direction $-N$ from $x$. This implies that $x\in F\cap S$
  is a one-dimensional ridge point of $U_E$.
\end{proof}
This tells us that points of the front that are on the surface must be
ridge points, and so we restrict our attention to ridge points on the
front as possible points on the surface.

\subsection{Ridge Curve Extraction Using the Morse Complex}

Since computing ridges directly from Definition~\ref{defn:ridge},
using differential operators, is sensitive to noise, scale spaces
\cite{lindeberg1998edge,kolomenkin2013multi} are often used. However,
that approach, while being more robust to noise, may distort the data,
and it is often difficult to obtain a connected curve as the
ridge. Therefore, we derive a robust method by making use of the Morse
complex and cubical complex theory to extract the ridge of interest
from the data $U_E$. Cubical complex theory guarantees the correct
topology of the desired ridge (as a 1-dimensional closed curve).

\comment{
Let $h : M \subset \R^n \to \R$ be a function from an
$n-1$ dimensional manifold $M$ to the real line. In our case, the
manifold $M$ is the Fast marching front at a particular time instant,
and the function $h$ is the Euclidean length of minimal paths ending
on the front, $d_E$. 
}

{\bf Relation Between Ridges and Morse Complex}: In the following
proposition, we note that certain ridges of a smooth function can be
computed by computing ascending manifolds. We assume that $M$ is a
2-manifold.
\begin{prop}
  Boundaries of ascending manifolds of $h$ are ridges of $h$.
\end{prop}
\begin{proof}
  Suppose that $x\in \partial A(p_1)$ then for any neighborhood $V_x$
  sufficiently small around $x$, we have that
  $\partial A(p_1) \cap V_x$ divides $V_x$, i.e.,
  $V_x = [V_x \cap A(p_1)] \cup [V_x \cap A(p_2)]$ ($p_1\neq p_2$) for
  the case when $V_x$ intersects two ascending manifolds. Note that
  $-\nabla h(y) \cdot N_2 > 0$ for $y \in V_x \cap A(p_2)$ where $N_2$
  is the inward normal to $\partial A(p_2)$ when $V_x$ is small
  enough. If this were not the case, then paths following the negative
  gradient would intersect the boundary $\partial A(p_2)$, which is
  not the case since they flow into $p_2$. By a similar argument,
  $-\nabla h(y) \cdot N_2 < 0$ for $y \in V_x \cap A(p_1)$. Since the
  function $h$ is assumed smooth and thus the gradient is continuous,
  we must have that $\nabla h(x) \cdot N_2 = 0$. Further, the function
  is decreasing away from $x$ along the directions $\pm N_2$ as points
  in $V_x \backslash \{x\}$ belong to ascending manifolds. Therefore,
  the point $x$ is a local maximum in the direction $N_2$. Ridges
  satisfy this property.  Hence, boundaries of the ascending manifolds
  are ridges.
\end{proof}

{\bf Algorithm for Ridges via Morse Complex}: Next, we specify a
discrete algorithm to determine the Morse complex of $-U_E|F$. The
boundaries of ascending manifolds can then be used to extract the
relevant ridge. We retract the front to the ridge curve by an ordered
removal of free faces based on lowest to highest ordering based on
$U_E|F$.

Given a front $F$, obtained by thresholding the distance $U$, the
two-dimensional cubical complex $C_F$ of the front
is constructed as follows. Let $\Z_n = \{ 0, 1, \ldots, n-1 \}$ be a
sampling of a coordinate direction of the image. Then
\begin{itemize}
\item $C_F$ contains all 2-faces $f$ in $\Z_n^3$ between any 3-faces
  $g_1, g_2$ with the property that one of $g_1,g_2$ has all its
  0-sub-faces with $U < D$ and one does not.
\item Each face $f$ of $C_F$ has cost equal to the average of $U_E$
  over 0-sub-faces of $f$.
\end{itemize}

Our algorithm for Morse complex extraction and boundaries of the
ascending manifolds is given in
Algorithm~\ref{alg:ridge_extraction}. The algorithm creates holes at
local minima of the function $U_E|F$ defined on 1-faces by removing
the adjacent 2-faces. It then removes free faces in increasing order
of $U_E|F$ so as to preserve homotopy equivalence. The removed points
associated with a local minimum form the ascending manifold for the
local minimum. The faces that cannot be removed without breaking
homotopy equivalence, i.e., the isthmus faces, form the boundaries of
the ascending manifolds. The algorithm removes all 2-faces and
preserves only isthmus 1-faces, and hence the remaining structure of
$C_F$ is one dimensional. Further, since the algorithm preserves
homotopy equivalence, the remaining structure at the end of the
algorithm is connected. This is a clear advantage over computation of
ridges from differential operators, which does not guarantee
connectedness. A heap is used to keep track of the faces in order. The
computational complexity of this extraction is therefore $O(N\log N)$
where $N$ is the number of pixels, an over-estimate since the faces in
the complex are significantly lower than the number of pixels.

\begin{algorithm}
  \begin{algorithmic}[1]
    \Procedure{Morse Complex}{$C_F$, $U_E$}
    \State \Comment{{\it $C_F = $ cubical 2-complex, $U_E = $ cost on 1-faces in $C_F$}}
    \State $\mbox{id} \gets 0$
    \State Create heap of 1-faces ordered by $U_E$ (min at top)
    \Repeat 
    \State Remove 1-face $g$ from heap
    \If { $g$ is a subset of two faces $f_1$ and $f_2$ in $C_F$}
    \State Remove $g, f_1, f_2$ from $C_F$
    \State $l(f_1) \gets l(f_2) \gets \mbox{id}$, $\mbox{id} \gets
    \mbox{id} +1$
    \State \Comment{ {\it new id for ascending manifold; hole
        at local min}}
    \ElsIf { $(g,f)$ is a free pair in $C_F$}
    \State Remove $g,f$ from $C_F$
    \State $l(f) \gets l(f_{adj})$ where $f_{adj} \supset g$ and
    $f_{adj} \notin C_F$
    \State \Comment{ {\it labels face same as adjacent face containing $g$} }
    \ElsIf { $(f,g)$ is a free pair in $C_F$ }
    \State Remove $g,f$ from $C_F$
    \ElsIf {$g$ is isthmus}
    \State $l(g) = \{ l(f_1), l(f_2) \}$ where $f_1,f_2 \supset g$ 
    \State \Comment{{\it label is unordered list}}
    \EndIf
    \Until{ heap is empty }
    \State \textbf{return} $C_F$, $l$ \Comment{{\it Ridges, labels for
      2-faces, ridges}}
    \EndProcedure
  \end{algorithmic}
  \caption{Morse Complex Extraction}
  \label{alg:ridge_extraction}
\end{algorithm}

Ideally, in the case of clean data $\phi$, the \emph{function} $U_E$
defined on the front would have a rather simple topology, indeed a
volcano structure (see left image in
Fig.~\ref{fig:ridge_curve_extraction}), where the ridge separates the
inside of the volcano from the outside. The two minimum of $U_E$ on
each side of the ridge would correspond to points away from the
surface in the direction of the surface normal. In this case, the
previous algorithm would produce the inside of the volcano, and the
outside as two components of the complex, and the boundary between
them as the ridge, as desired. However, due to noise other ridge
structures besides the main ridge of interest can be extracted.

Fortunately, we can simplify the extracted collection of ridges from
the previous algorithm by applying the algorithm iteratively. We
construct a new complex with a 2-face for each ascending manifold
computed, and a 1-face connecting 2-faces if two corresponding
ascending manifolds have intersecting boundaries. Each 1-face in this
new complex is assigned a value to be the average of 1-faces in the
common boundary between ascending manifolds. The Morse complex of this
simplified complex is then computed, and the process is repeated until
only one loop remains. The algorithm is given in
Algorithm~\ref{alg:highest_ridge_curve}. Figure~\ref{fig:noisy_example}
shows an example run through this algorithm.

\begin{algorithm}
  \begin{algorithmic}[1]
    \Procedure{HighestRidge}{$C_F$, $U_E$}
    \State \Comment{ {\it $C_F$ cubical 2-complex of Fast Marching front} }
    \State \Comment{ {\it Euclidean trajectory length $U_E$ defined on
      1-faces} }
    \Repeat
    \State ( $C_F'$, $l$ ) = {\sc Morse Complex} ( $C_F$, $U_E$ )
    \State Create 2-cubical complex $C_F''$ with 
    \State \indent a 2-face $f$ for each unique 2-face id in $l$
    \State \indent a 1-face $g$ for each unique 1-face id in $l$
    \State \indent $g$ joins $f_1$ and $f_2$ if $l(g) = \{ l(f_1),  l(f_2) \}$
    \For {$g$ each 1-face in $C_F''$ }
    \State $R = \{ g' \in C_F' \,:\, l(g') = l(g) \}$ \Comment{{\it a ridge}}
    \State $U_E'(g) \gets$ average of $U_E$ along $R$
    \EndFor
    \State $C_F \gets C_F''$, $U_E \gets U_E'$
    \Until{ no degree three 1-faces in $C_F'$}
    \State \textbf{return} $C_F'$
    \EndProcedure
  \end{algorithmic}
  \caption{Highest Ridge Curve Extraction}
  \label{alg:highest_ridge_curve}
\end{algorithm}

\def\fHeightDeform{0.48in}
\begin{figure}
  \centering
  {\scriptsize
    \includegraphics[clip,trim=30 0 30 0, totalheight=\fHeightDeform]{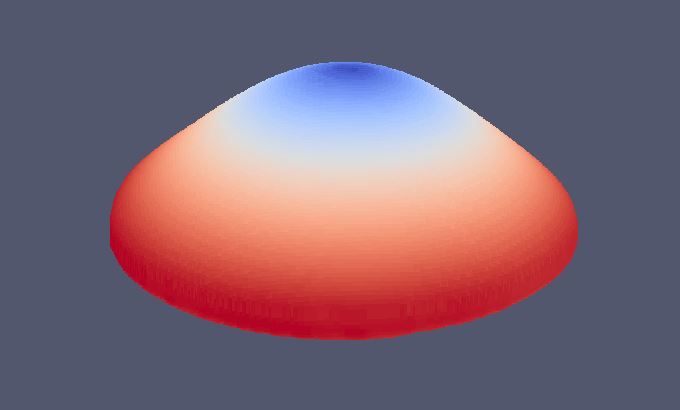}
    \includegraphics[totalheight=\fHeightDeform]{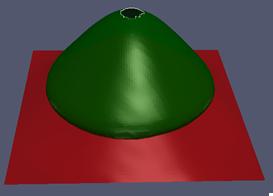}
    \includegraphics[totalheight=\fHeightDeform]{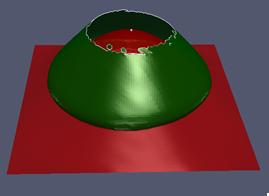}
    \includegraphics[totalheight=\fHeightDeform]{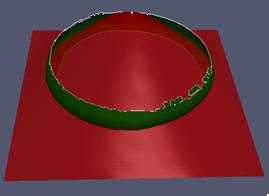}
    \includegraphics[totalheight=\fHeightDeform]{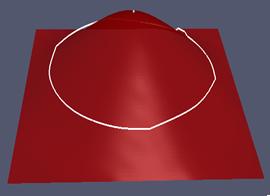}
  }
  \caption{ [1st image]: Front color coded with Euclidean path length
    $U_E$ (top view). Red indicates high values. The bottom view (not
    shown) is a symmetric flip. Topologically, $U_E$ forms a volcano
    structure (ridge, i.e., top of volcano, is darkest red), and
    inside the volcano is blue. [Subsequent images]: Illustration of
    iterations (from left to right) of
    Algorithm~\ref{alg:ridge_extraction} on noise-less data to obtain
    the ridge curve (white) on the Fast Marching front (green) by
    computing the Morse complex of $U_E$. The ridge curve lies on the
    surface of interest (red).}
  \label{fig:ridge_curve_extraction}
\end{figure}

\def\fHeightNoisy{0.77in}
\begin{figure}
  \centering
  {\scriptsize
    \includegraphics[totalheight=\fHeightNoisy]{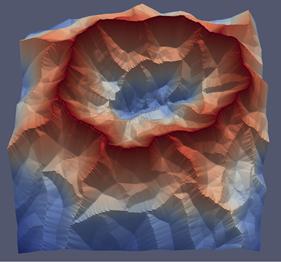}
    \includegraphics[totalheight=\fHeightNoisy]{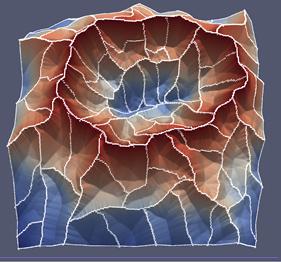}
    \includegraphics[totalheight=\fHeightNoisy]{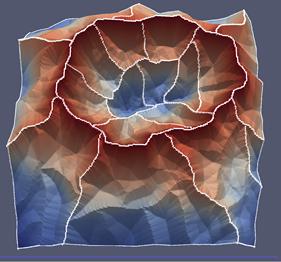}
    \includegraphics[totalheight=\fHeightNoisy]{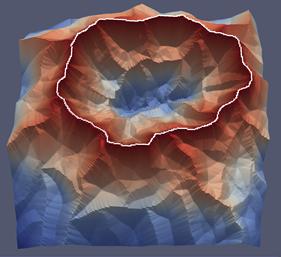}
  }
  \caption{Illustration of Algorithm~\ref{alg:highest_ridge_curve}
    operating on noisy data to obtain the highest ridge.}
  \label{fig:noisy_example}
\end{figure}

\comment{
\begin{algorithm}
  \begin{algorithmic}[1]
    \State Input: 2D cubical complex $C_F$ of FM front and
    Euclidean distance $U_E$ 
    \State Sort $C_F$ 1-faces from min to max based on $U_E$
    \Repeat 
    \State Let $g$ be the minimum value 1-face in $C_F$
    \If {$(f,g)$ or $(g,f)$ is a free pair in $C_F$ for some $f$}
    \State Remove $f$ and $g$ from $C_F$
    \EndIf
    \Until{no free pairs in $C_F$}
  \end{algorithmic}
  \caption{Ridge Curve Extraction}
  \label{alg:curve_extraction}
\end{algorithm}
}

An example of ridge curves detected for multiple fronts is shown in
Figure~\ref{fig:curve_extraction}. This procedure of retracting the
Fast Marching front to form the main ridge is continued for different
fronts of the form $\{ U < D \}$ with increasing $D$. This forms many
curves on the surface of interest.  In practice, in our experiments,
$D$ is chosen in increments of $\Delta D = 20$, until the stopping
condition is achieved, and this typically results in $10-20$ ridge
curves extracted. The next sub-section describes the stopping
criteria.

\def\fHeightCut{1.0in}
\begin{figure}
  \centering
  \includegraphics[totalheight=\fHeightCut]{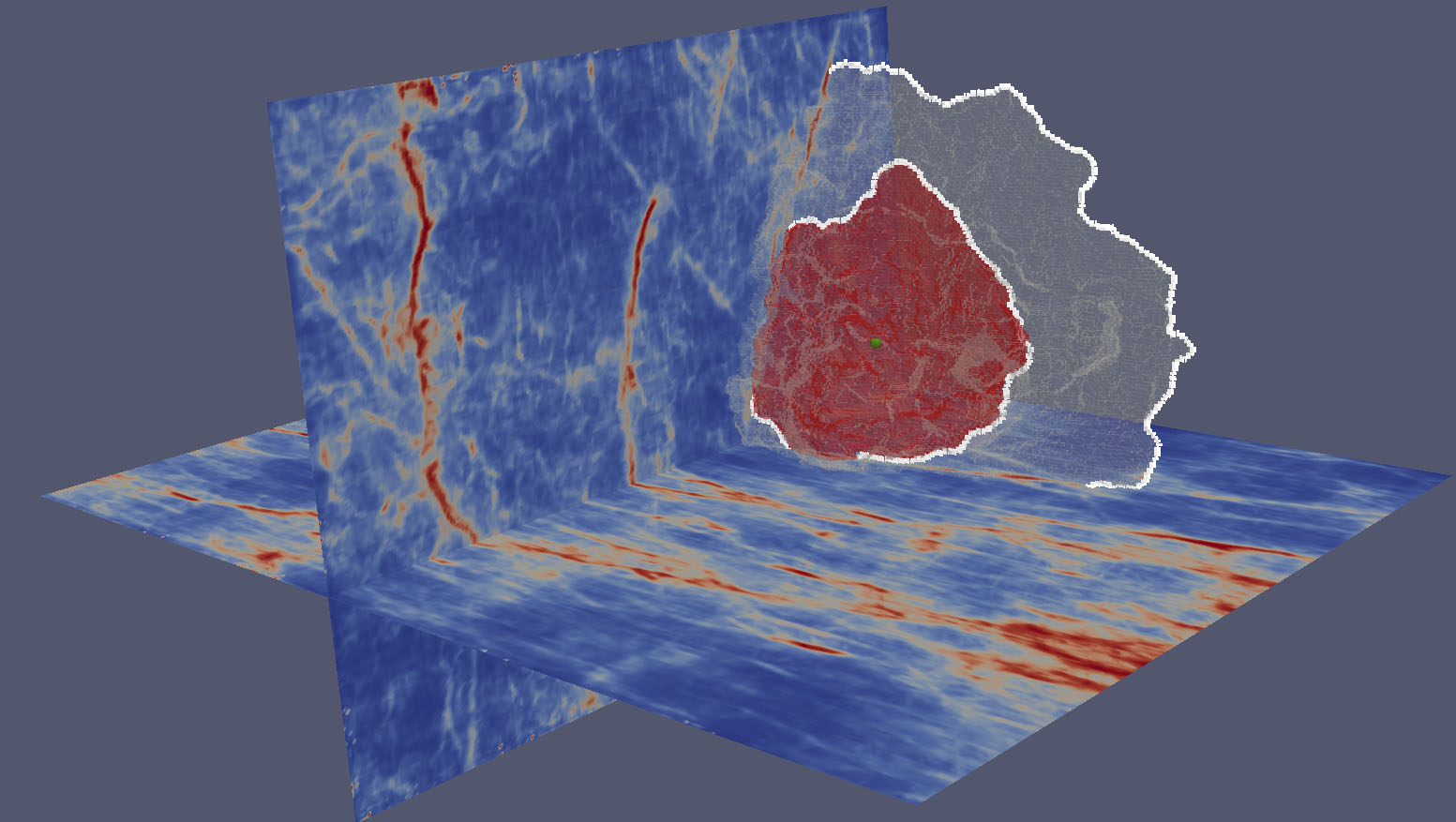}
  \includegraphics[clip,trim=30 0 0 0,totalheight=\fHeightCut]{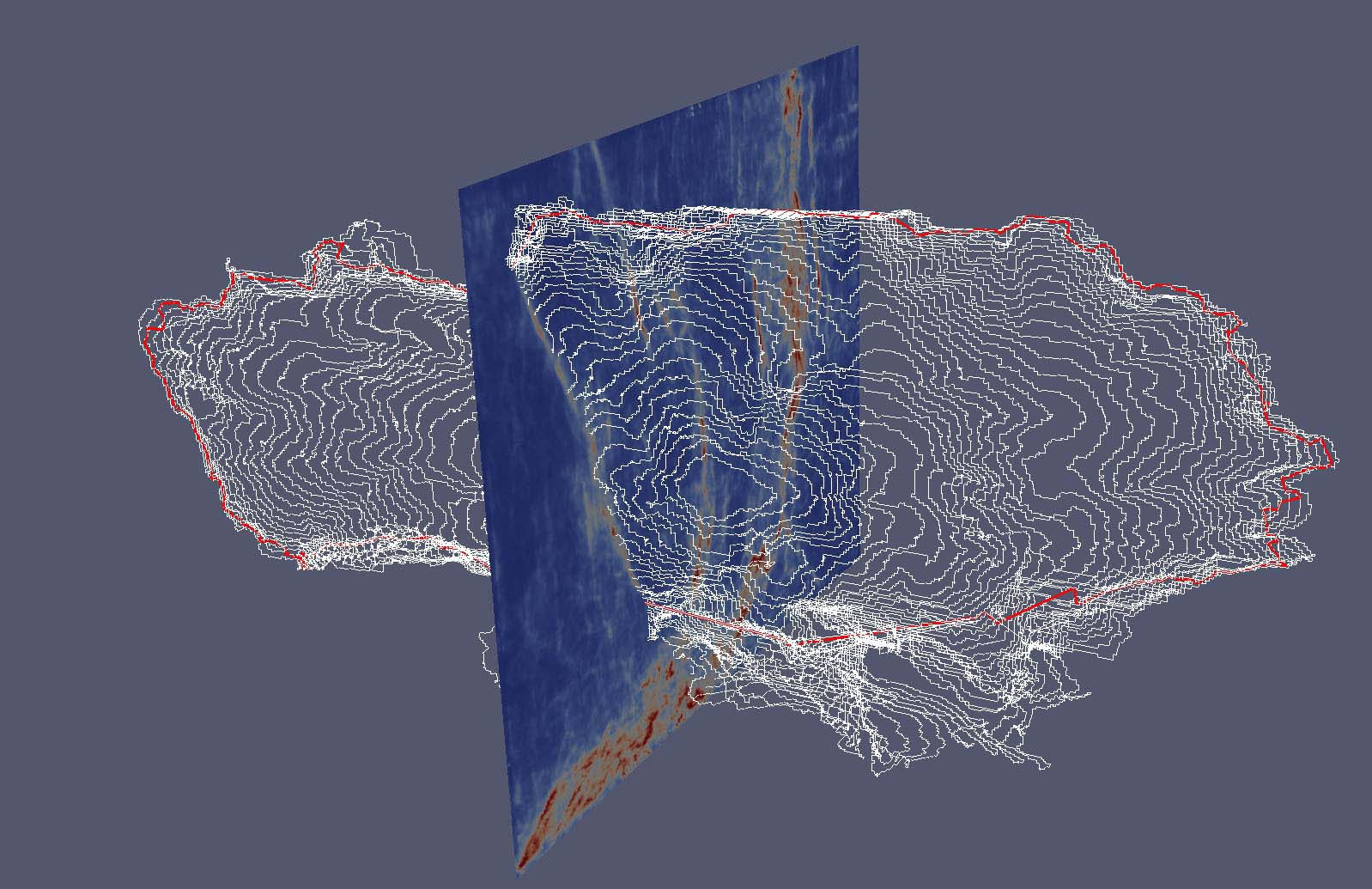}
  \caption{[Left]: Ridge curve (white) extraction by retracting the Fast
    Marching front at two instants. [Right]: An example cut (red) of
    ridge curves, forming the surface boundary. Notice that the cut
    matches with the end of high $1/\phi$ (bright areas).}
  \label{fig:curve_extraction}
\end{figure}

\subsection{Stopping Criteria and Surface Boundary Extraction}

\label{subsec:cut}

To determine when to stop the process of extracting ridge curves, and
thus obtain the outer boundary of the surface of interest, we make the
following observation. Parts of the curves generated from the previous
section move slowly, i.e., become close together with respect to
Euclidean distance at the boundary of the surface. This is because the
speed function $1/\phi$ becomes small outside the surface. Hence, for
the curves $c_i$ generated, we aim to detect the locations where the
distance between points on adjacent curves becomes small. To formulate
an algorithm robust to noise, we formulate this as a graph cut problem
\cite{boykov2001interactive}.

We define the graph $G$ as follows:
\begin{itemize}
\item vertices $V$ are 0-faces in all the 1-complexes $c_i$ formed from ridge
  extraction 
\item edges $E$ are $(v_1,v_2)$ where $v_1, v_2\in V$ are such that
  $v_1, v_2$ are connected by a 1-face in some $c_i$ or $v_1$ is a
  0-face in $c_i$ and $v_2$ is the closest (in terms of Euclidean
  distance) 0-face in $c_{i+1}$ to $v_1$
\item a cost $|v_j-v_k|$ is assigned to each edge
  $(v_j,v_k)$ where $v_j$ and $v_k$ belong to different $c_i$ (so that
  the min cut will be where adjacent curves are close)
\item for edges $(v_j,v_k)$ such that $v_j$ and $v_k$ belong to the
  same $c_i$, the cost is the minimum Euclidean distance between
  segment $(v_j,v_k)$ and segments on $c_{i+1}$
\item the source is the seed point $p$, and the sink is the last ridge
  curve $c_l$
\end{itemize}

We wish to obtain a cut of $G$ (separating $G$ into two disjoint sets)
with minimum total cost defined as the sum of all costs along the
cut. In this way, we obtain a cut of the ridge curves along locations
where the distance between adjacent ridge curves is small. The process
of obtaining ridge curves from the Fast Marching front is stopped when
the cost divided by the cut size is less than a pre-specified
threshold. This cut then forms the outer boundary of the surface. The
computational cost of the cut (compared to other parts of the
algorithm) is negligible as the graph size is typically less than
$0.5\%$ of the image.  Figure~\ref{fig:curve_extraction} shows an
example of a cut that is obtained. Figure~\ref{fig:sythetic_example}
shows a synthetic example.

\section{Surface Extraction}

\label{subsec:surface_extraction}

We now present our algorithm for surface extraction.  Given the
surface boundary curve determined from the previous section, we
provide an algorithm that determines a surface going through locations
of small $\phi$ and whose boundary is the given curve. Our algorithm
uses the cubical complex framework and has complexity $O(N\log N)$.

\comment{
Although there is another algorithm, \cite{grady2010minimal}, for
this task, it is computationally expensive as we show in
Section~\ref{sec:expts}.
}

\subsection{Valley Extraction Algorithm and Rationale}

We show now that the surface of interest lies in a valley of
$U : \Omega\to\R^+$, the weighted minimal path length.

\begin{prop} \label{prop:surface_valley}
  Suppose $S \subset \R^3$ is a smooth surface and $p\in S$.  Let
  $\phi : \Omega \subset \R^3 \to \R^+$ be a function with low values
  on $S$ and higher values outside (locally). Then $S$ is a valley of
  $U$, where $U$ is the solution of the eikonal equation with
  $U(p)=0$.
\end{prop}
\begin{proof}
  We show that for $x\in S$, $U$ decreases away from $x$ in
  the direction $\pm N$, the normals to the surface at $x$. For a
  small enough neighborhood $V_x$ around $x$, we may assume that $S$
  is flat and that $\phi$ is approximated by
  \[
    \phi(x) = 
    \begin{cases}
      K_1 & x\notin S \cap V_x \\
      K_2 & x \in S \cap V_x
    \end{cases},
  \]
  where $K_1 >> K_2 > 0$. We also assume (for now) that $x$ close
  enough to $p$ so that $p$ lies in $V_x$. In this case, we see that
  \[
    U(x) \approx U(p) + K_2 |x-p| = L K_2,
  \]
  as the minimal path from $p$ to $x$ is approximately a straight line
  path on the surface, as the surface is nearly flat in $V_x$.
  Let $y = x \pm \varepsilon N$ for $\varepsilon > 0$ sufficiently
  small. We now consider the minimal path from $y$ to $p$. Note
  outside the surface, the path must be nearly a straight line as
  $\phi$ is constant. Similarly, on the surface the minimal path must be a
  straight line. We see that the minimal path is a straight line
  between $y$ and some point $z$ on the line joining $x$ to $p$ and then
  the straight line between $z$ and $p$ (see
  Figure~\ref{fig:surface_valley}). Therefore, 
  \[
    U(y) = \min_{\ell } K_2(L-\ell) + K_1\sqrt{ \ell^2 +\varepsilon^2 }
  \]
  where $\ell$ is the length of the segment between $x$ and $z$. The
  minimizer is $\ell = \varepsilon / \sqrt{1 - r^2}$, where $r =
  K_2/K_1 < 1$. This yields that 
  \begin{align*}
    U(y) &= L K_2 - \frac{ \varepsilon }{ \sqrt{1-r^2} } K_2 +
           \varepsilon \frac{ \sqrt{2-r^2} }{ \sqrt{1-r^2} } K_1 
           \\
         &= LK_2 + \frac{ \varepsilon }{ \sqrt{1-r^2} } [ K_1\sqrt{2-r^2} -
           K_2 ] > L K_2,
  \end{align*}
  where the last inequality follows from the fact that $\sqrt{2-r^2} >
  1$ and $K_1 > K_2$. Therefore, $U(y) > U(x)$ and so we see a local
  minimum in the direction $N$, which implies $x$ lies in a 2-d valley of
  $U$. We may now apply the same argument using $x$ to play the role
  of $p$, and show that all points in a neighborhood of $x$ on the
  surface are on a valley. We may continue in this way to show all
  points on the surface are on the valley.
\end{proof}

\begin{figure}
  \centering
  {\scriptsize
    \includegraphics[totalheight=1.2in]{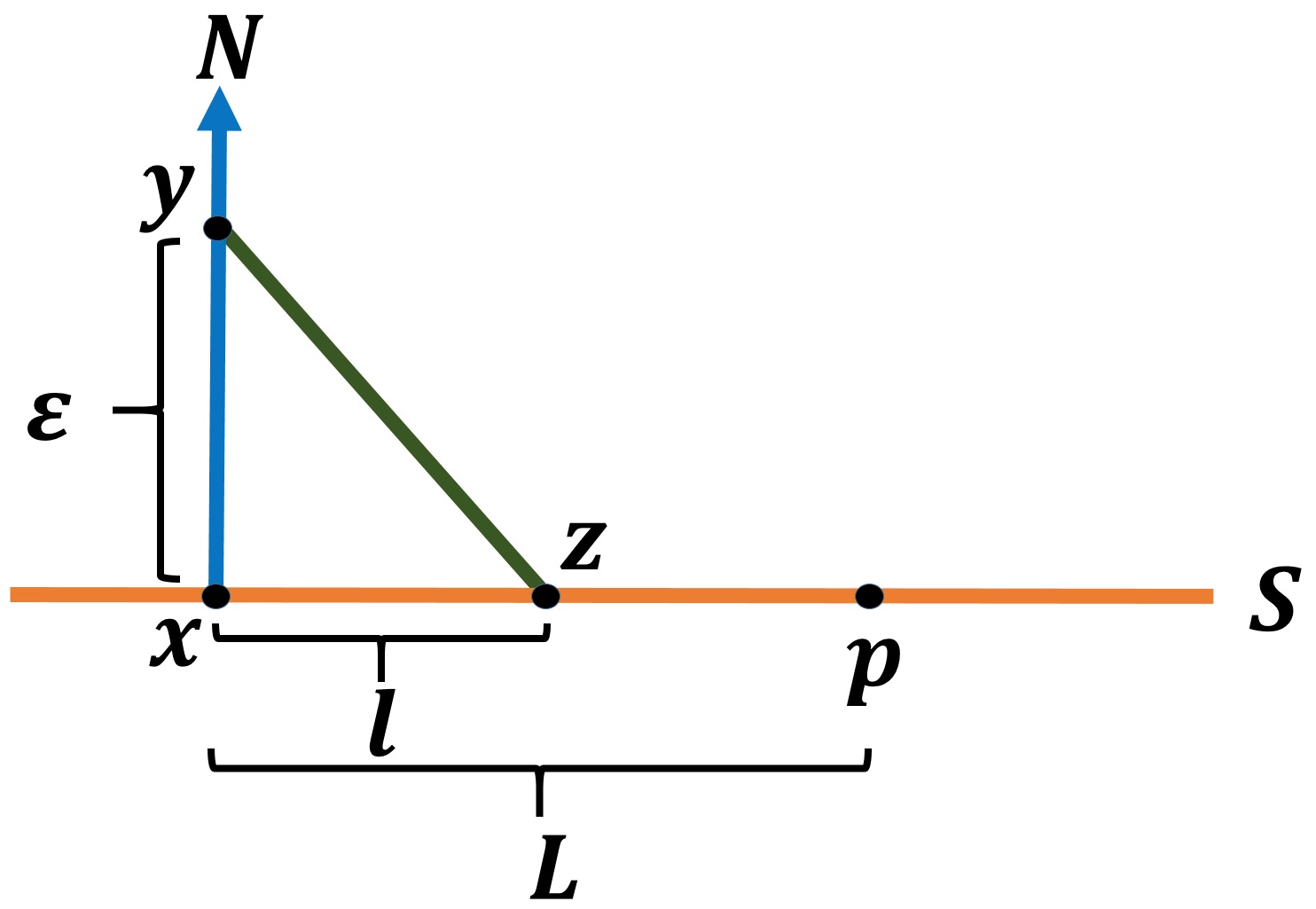}
  }
  \caption{Quantities defined in the proof of
    Proposition~\ref{prop:surface_valley}. }
  \label{fig:surface_valley}
\end{figure}

{\bf Algorithm}: We can use the above fact to design an algorithm for
extracting the surface. We may perform a deformation retraction of
$V_0 = \{ U \leq T(0) \}$ where $T(0)$ is chosen to enclose the entire
surface, and $T(t)$ is a decreasing function of $t$. At each time, the
points of the level set $L_t = \{ U = T(t) \}$ that retain the
homotopy equivalence to $V_t$ are removed from $V_t$. We further
impose that the boundary of the surface must not be removed from
$V_t$. This way, all points that are on the surface are retained. One
can show this with an inductive argument. Assume for a given time $t$,
the union of all retained sets is a 2-dim set $S_{t-}$ ($t-$ is just
before $t$) that is on the surface, and so
$V_{t-} = S_{t-} \cup \{ U\leq T(t) \}$. Note that the latter set in
the union is a volume. A point $x\in \partial S_{t^{-}}$ with
$U(x)=T(t)$ cannot be removed. Since $x$ is on the surface, which by
the proposition is a valley point, the normal to the surface at $x$ is
tangent to $L_t$, and $U$ is strictly increasing along the
normal. Therefore, removing point $x$ disconnects $V_{t-}$, not
preserving homotopy equivalence. Therefore,
$V_t = S_{t} \cup \{ U < T(t) \}$ where $S_t$ contains all points on
$\partial S_{t^{-}}$.

\comment{
This results in fronts that have large distance
$U$ from the seed point being removed first. By the constraint, only
the parts of the fronts that do not touch the boundary can be
removed. As the removal progresses, faces are removed on either side
of the surface. This creates a ``wrapping'' effect around the surface
of interest, which have small values of $U$. Near the end of the
algorithm, points on the surface cannot be removed without creating a
hole, so no faces are free, and thus the algorithm stops.}

This procedure can be accomplished with an analogous algorithm in the
discrete case. We retract the cubical complex of the image with the
constraint that the boundary curve $1$-faces cannot be removed. We
accomplish this retraction by an ordered removal of free faces based
on weighted path length $U$. The algorithm is described in
Algorithm~\ref{alg:surface_extraction}.
\begin{algorithm}
  \begin{algorithmic}[1]
    \Procedure{ValleyExtract}{$C_I$, $U$, $\partial S$}
    \State \Comment{ {\it $C_I =$ cubical 3-complex of image, $U=$ FM
        distance} }
    \State \Comment{ {\it $\partial S = $ boundary of surface (1-complex) } }
%    \State Sort faces of $C_I$ based on $U$ in decreasing order
    \State Create heap of 2-faces ordered by $U$ (max at top)
    \Repeat
%    \State Let $g$ be the 2-face with largest weight
    \State Remove 2-face $g$ from heap
    \If{$(g,f)$ is a free pair in $C_I$ for some $f$}
    \State Remove $f$ and $g$ from $C_I$
    \ElsIf{$(f,g)$ is a free pair in $C_I$ for some $f$ and $g\cap \partial S
      = \emptyset$}
    \State Remove $f$ and $g$ from $C_I$
    \EndIf
    \Until {heap is empty}
    \State \textbf{return} $C_I$ \Comment{{\it 2-cubical complex of Valley}}
    \EndProcedure
  \end{algorithmic}
  \caption{Surface Extraction from Boundary of Surface}
  \label{alg:surface_extraction}
\end{algorithm}
Figure~\ref{fig:valley_extraction} shows the evolution from
Algorithm~\ref{alg:surface_extraction} to extract the surface from the
data used in
Figure~\ref{fig:curve_extraction}. Figure~\ref{fig:sythetic_example}
shows a synthetic example of the evolution of this algorithm.

\def\fHeightExtractSurfaceItr{0.8in}
\begin{figure}
  \centering
  {\scriptsize
    \includegraphics[width=\fHeightExtractSurfaceItr]{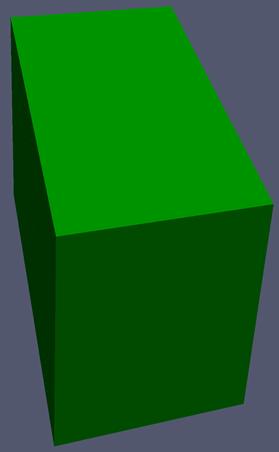}
    \includegraphics[width=\fHeightExtractSurfaceItr]{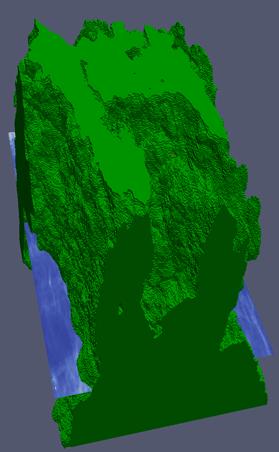}
    \includegraphics[width=\fHeightExtractSurfaceItr]{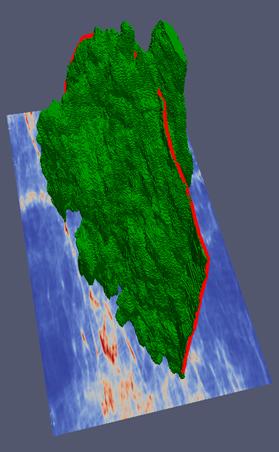}
    \includegraphics[width=\fHeightExtractSurfaceItr]{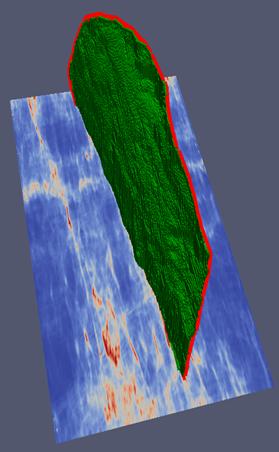}
  }
  \caption{Illustration of valley extraction by
    Algorithm~\ref{alg:surface_extraction}, which retracts the volume
    while preserving 1-faces on the surface boundary (red). This gives
    the surface of interest.}
  \label{fig:valley_extraction}
\end{figure}

\comment{
\def\fHeightSl{1.8in}
\begin{figure}
  \centering
  \includegraphics[totalheight=\fHeightSl]{./figures/fault_intersection/f02_slice}
  \includegraphics[totalheight=\fHeightSl]{./figures/fault_intersection/f02_surfcut}
  \includegraphics[totalheight=\fHeightSl]{./figures/fault_intersection/f02_gt} \\ \vspace{1mm}
  \includegraphics[totalheight=\fHeightSl]{./figures/fault_intersection/f01_slice}
  \includegraphics[totalheight=\fHeightSl]{./figures/fault_intersection/f01_surfcut}
  \includegraphics[totalheight=\fHeightSl]{./figures/fault_intersection/f01_gt} \\ \vspace{1mm}
  \includegraphics[totalheight=\fHeightSl]{./figures/fault_intersection/f04_slice}
  \includegraphics[totalheight=\fHeightSl]{./figures/fault_intersection/f04_surfcut}
  \includegraphics[totalheight=\fHeightSl]{./figures/fault_intersection/f04_gt}
  \caption{Validation on slices (first column) that intersect the
    surface computed with SurfCut (green, second column) passes
    through locations of high likelihood of the true surface (red
    regions) and ground truth (white, third column).}
  \label{fig:surface-from-bc}
\end{figure}
}

\subsection{Valley: Surface of Minimal Paths}

We now relate the valley that is extracted by our algorithm to minimal
paths. We show that the valley, and thus the surface extracted, is a
surface formed from a collection of minimal paths to $p$. First, we
show that the gradient path starting from a point in the valley stays
in the valley.

\comment{ The case of 1-dimensional valleys is similar. }

\begin{prop} \label{prop:valley_grad}
  Suppose $x \in M$ is a valley point of $h : M \to \R$, then the path
  $\gamma$ determined by the gradient descent of $h$ with initial
  condition $x$ lies on the valley of $h$ containing $x$.
\end{prop}
\begin{proof}
  For simplicity, we assume $M=\R^3$ and that the valley is
  two-dimensional. By definition of a 2D valley in $\R^3$, we have
  that $\nabla h(x)\cdot N_x = 0$ and $N_x^T H h(x) \cdot N_x > 0$ for
  some unit direction $N_x\in \R^3$ where $Hh$ denotes the Hessian.
  For every neighborhood $V_x$ of $x$ sufficiently small, there exists
  $y\in V_x$ such that $\nabla h(y)\cdot N_y = 0$ and
  $N_y^THh(y)\cdot N_y > 0$ for some $N_y$. If that were not the case,
  then $x$ would be a isolated critical point, which is not the
  case. By smoothness of $h$, $N$ is a smooth function. Let $S$ be the
  points the satisfy the conditions on the gradient and Hessian in
  $V_x$.

  We consider the path $\gamma$ defined by the gradient descent of $h$
  starting from $x$. Then by definition of $\gamma$ and Taylor
  expansion of $h$,
  \[
    \nabla h[ \gamma(\Delta t) ] \approx \nabla h[ x - \Delta t \nabla h(x)
    ] \approx \nabla h(x) - \Delta t H h(x) \cdot \nabla h(x).
  \]
  Taking the dot product of the above with
  $N_{ \gamma(\Delta t) } \approx N_x$, by a Taylor expansion, we have
  \[
    \nabla h[ \gamma(\Delta t) ] \cdot N_{ \gamma(\Delta t) }  \approx
    - \Delta t N_x^T Hh(x) \cdot \nabla h(x).
  \]
  Note that $N_x^THh(x) = \lambda N_x^T$ with $\lambda >0$ since $N_x$
  is an eigenvector of $Hh(x)$ by definition of valley. Since
  $N_x^T\nabla h(x) = 0$ by the definition of valley, we have that
  $\nabla h[ \gamma(\Delta t) ] \cdot N_{ \gamma(\Delta t) } \approx
  0$. Also,
  $N^T_{ \gamma(\Delta t) } Hh[ \gamma(\Delta t) ] \cdot N_{
    \gamma(\Delta t) } >0$ as $\gamma(\Delta t)\in V_x$. Therefore,
  $\gamma(\Delta t)$ is also in the valley, and thus continuing this
  way, we can show that the path $\gamma$ formed from the gradient
  descent is also in the valley.
\end{proof}

\def\fHeightSS{0.75in}
\def\fHeightS{0.85in}
\begin{figure}
  \centering
  {\footnotesize
    \begin{tabular}{c@{\hspace{0.05in}}c@{\hspace{0.06in}}c@{\hspace{0.06in}}c}
      Ridges & Final Cut & Surface & Ground truth \\
      \includegraphics[clip,trim=90 150 90 130, totalheight=\fHeightSS]{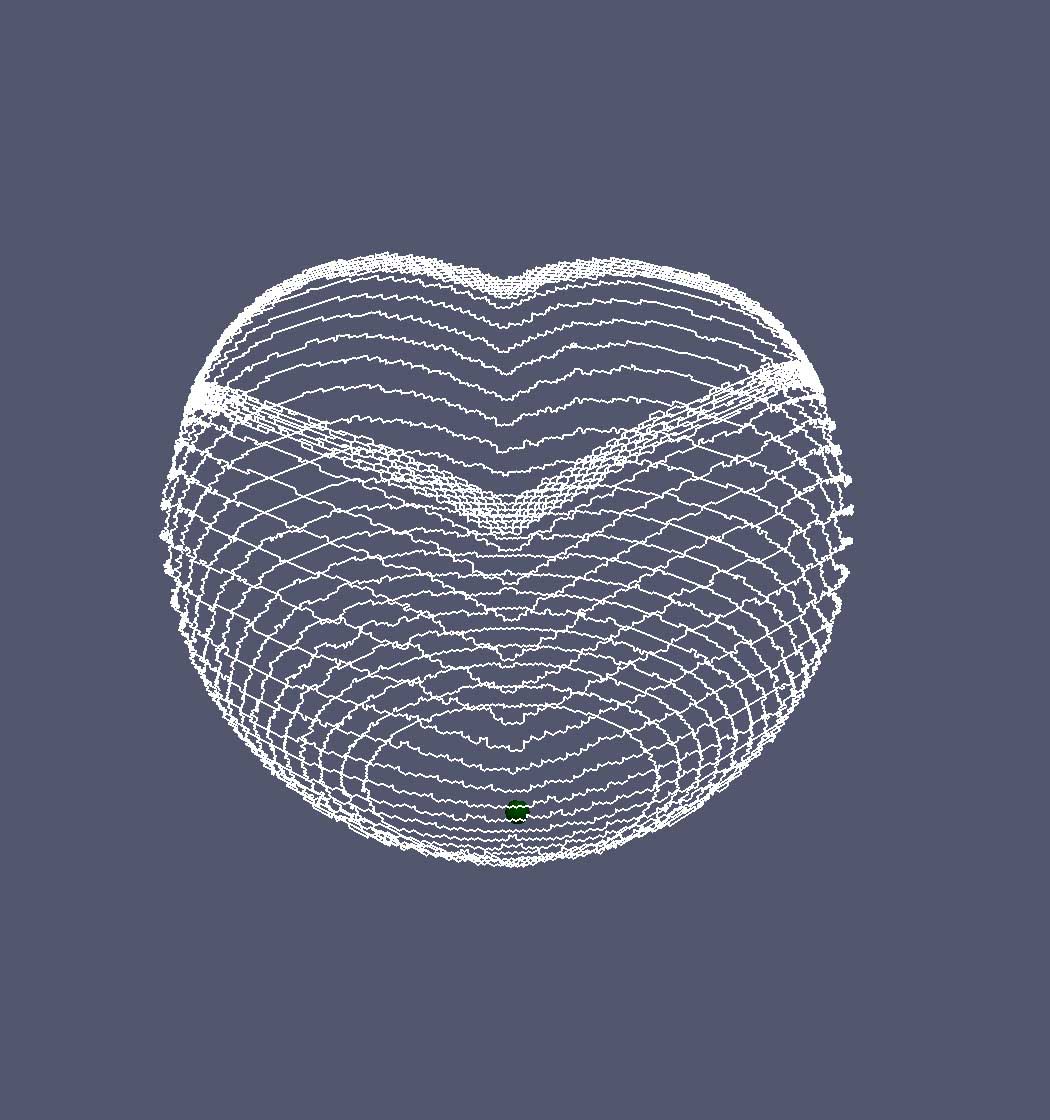} &
      \includegraphics[clip,trim=90 150 90 130, totalheight=\fHeightSS]{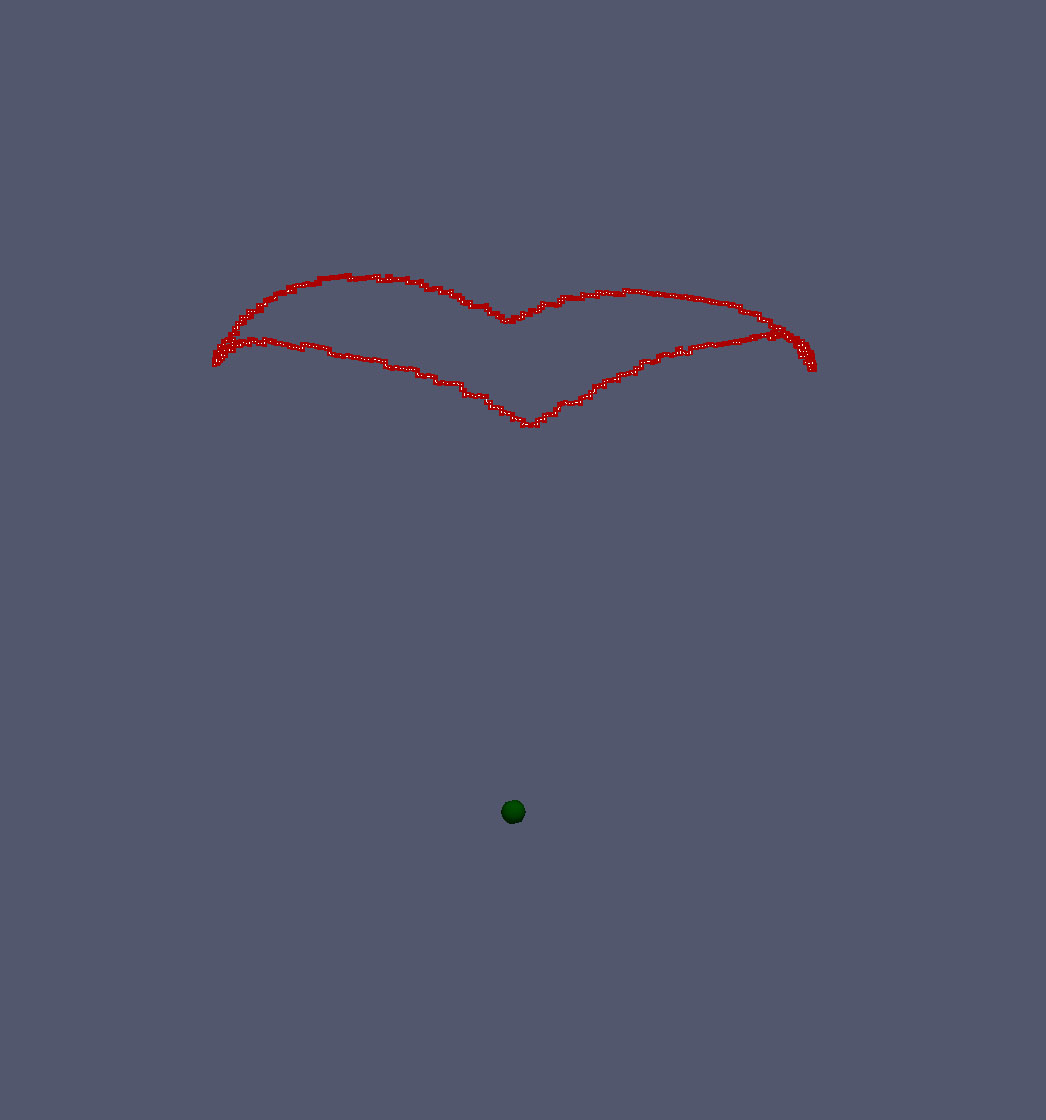} &
      \includegraphics[clip,trim=90 150 90 130, totalheight=\fHeightSS]{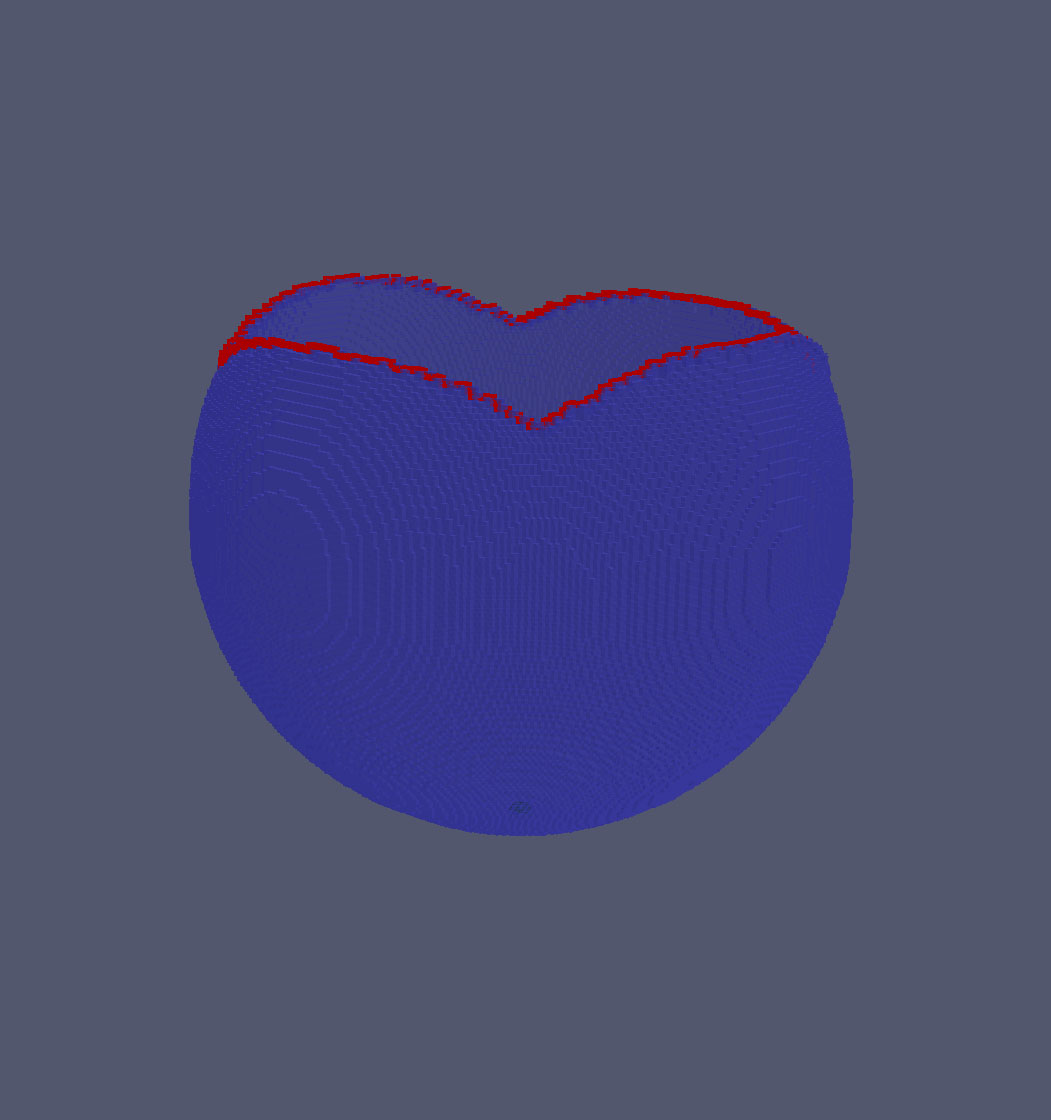} &
      \includegraphics[clip,trim=90 150 90 130, totalheight=\fHeightSS]{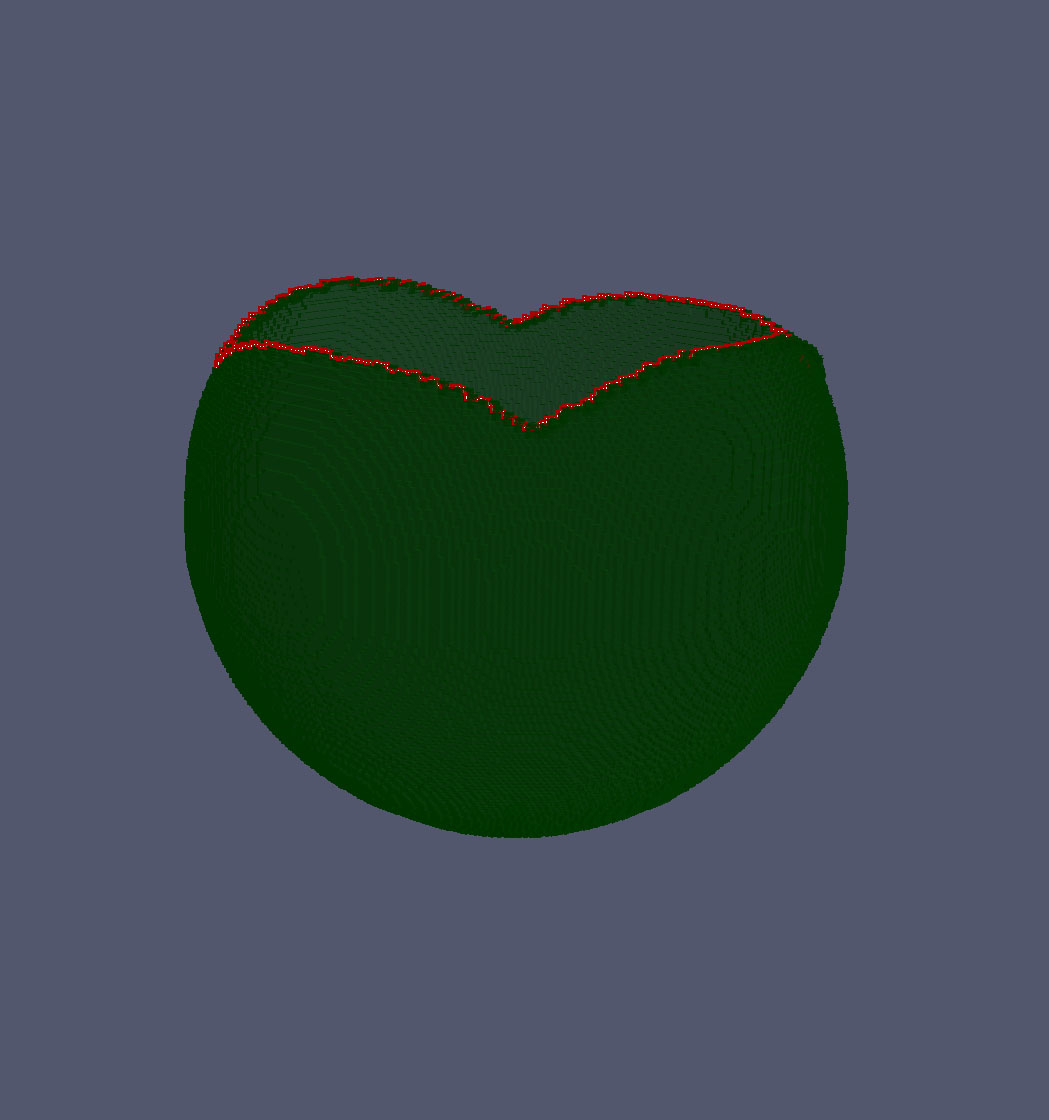} 
    \end{tabular}\\
    Removal of Faces in Image (Algorithm~\ref{alg:surface_extraction}) to Extract Surface  $\rightarrow$ \\
    \begin{tabular}{c@{\hspace{0.05in}}c@{\hspace{0.06in}}c@{\hspace{0.06in}}c}
    \includegraphics[clip,trim=90 90 90 90,totalheight=\fHeightS]{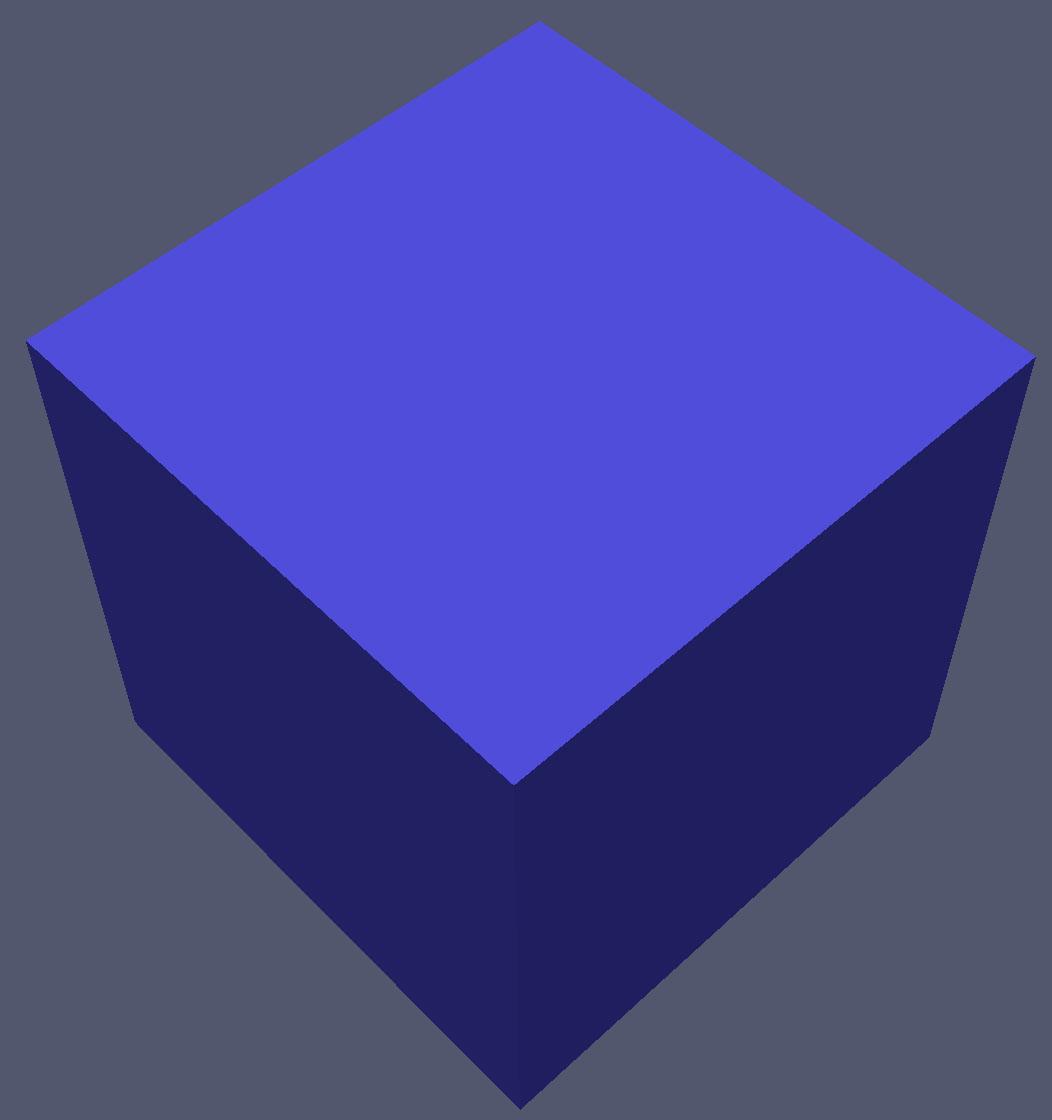} &
    \includegraphics[clip,trim=90 90 90 90,totalheight=\fHeightS]{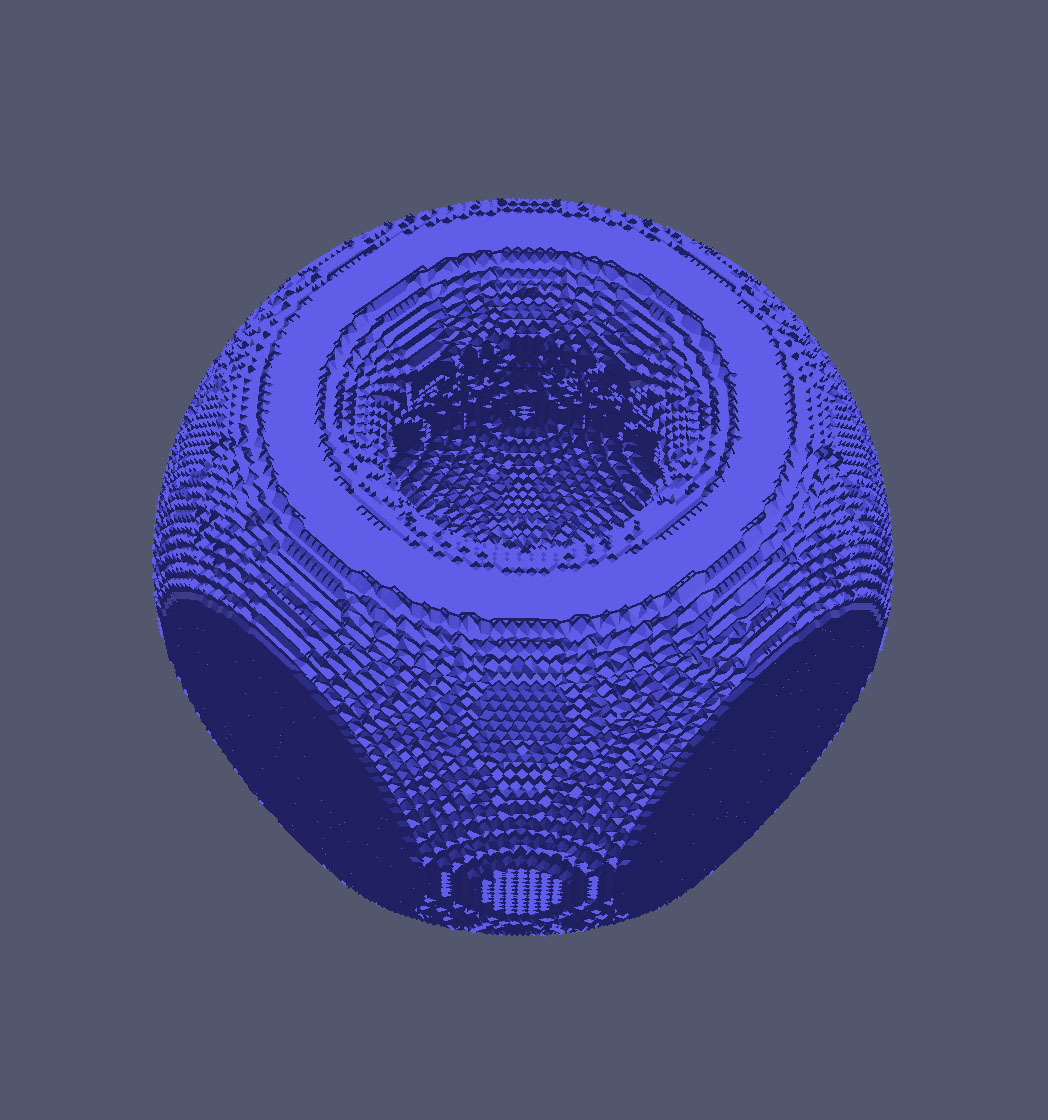} &
    \includegraphics[clip,trim=90 90 90 90,totalheight=\fHeightS]{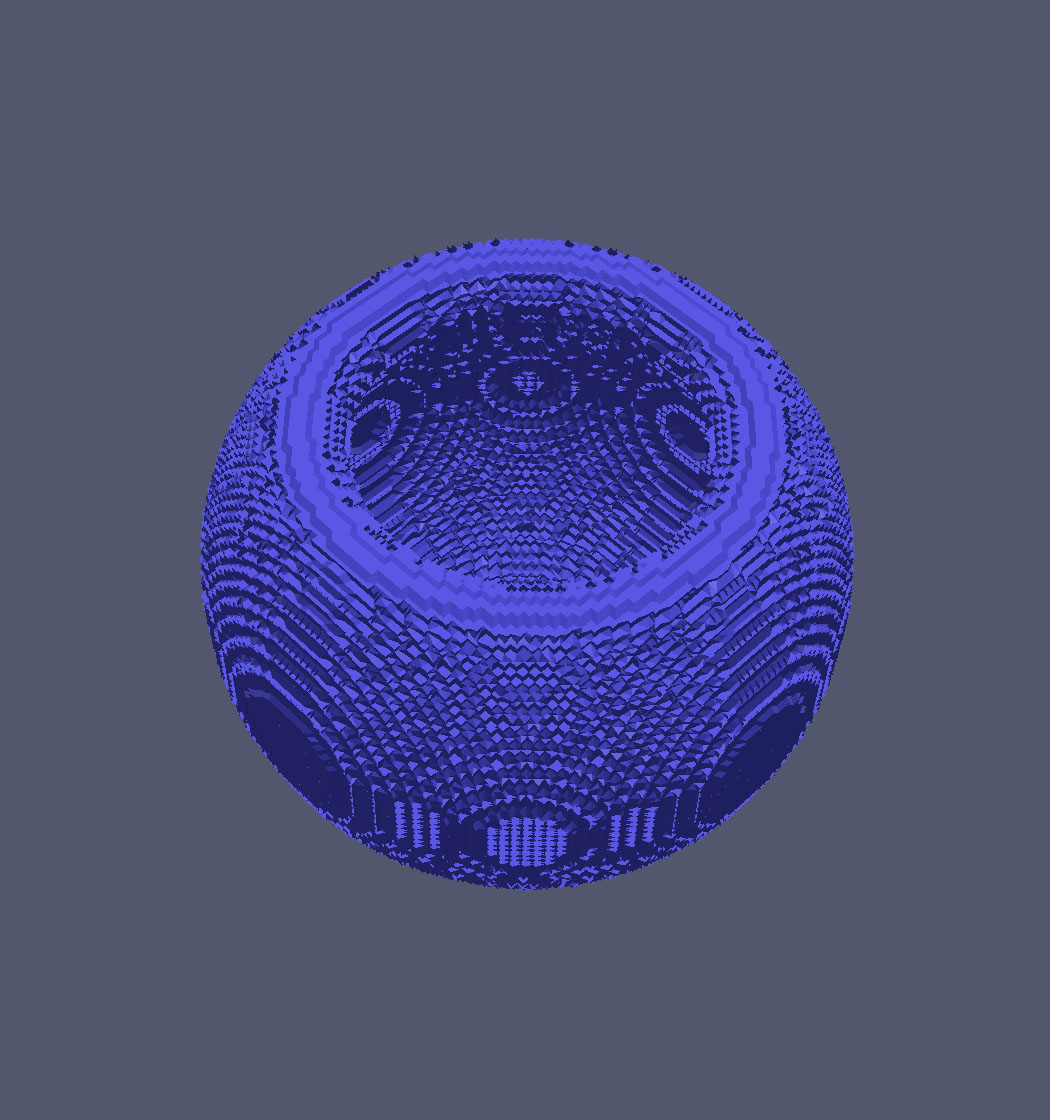} &
    \includegraphics[clip,trim=40 40 40 40,totalheight=\fHeightS]{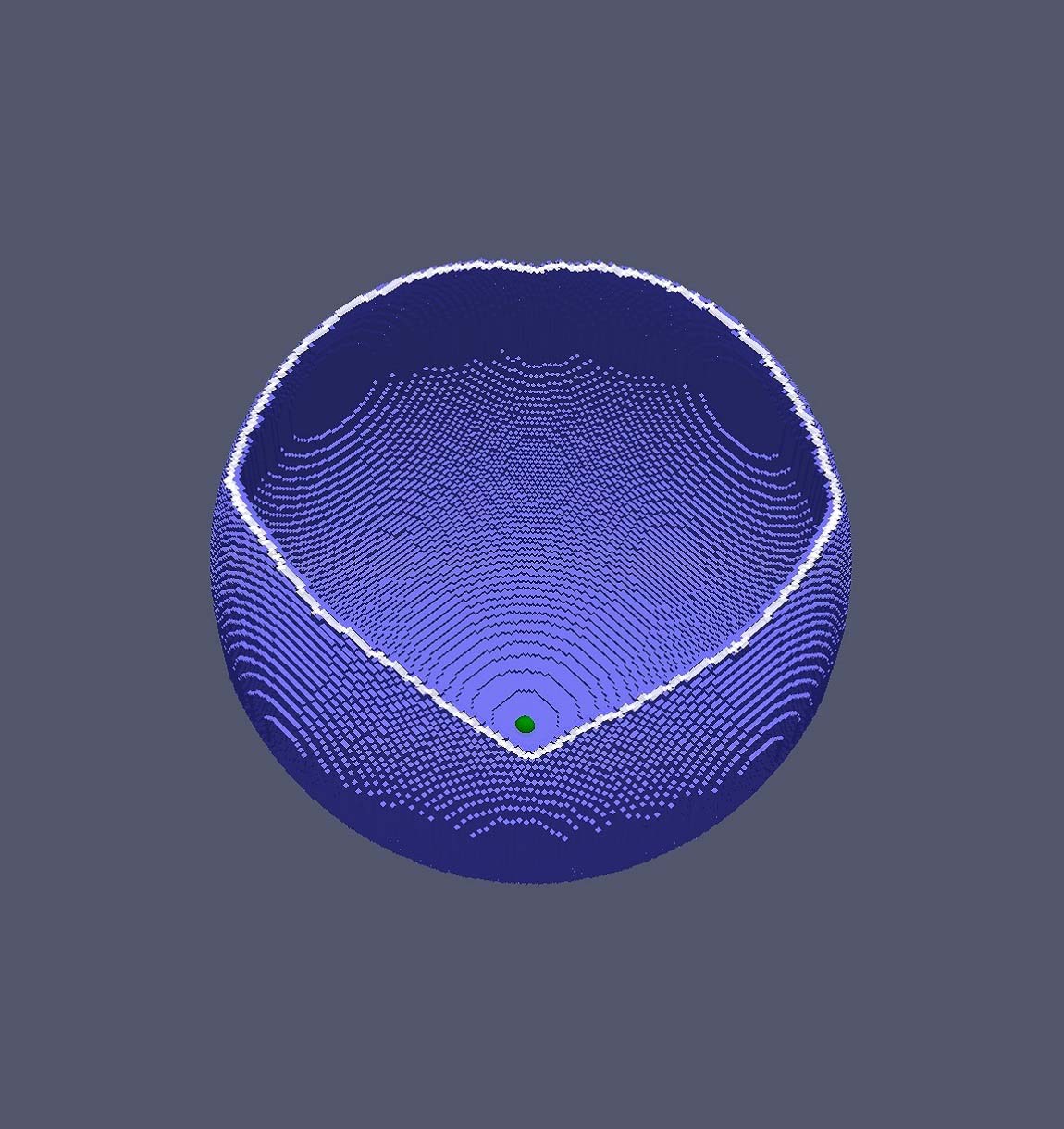}
    \end{tabular}\\
}
\caption{Synthetic example of extracting a sphere with top cut such
  that the boundary is four arcs. The image (not shown) is a noisy
  image of the cut sphere with holes. Ridge curves are
  extracted via Algorithm~1 (top left). The final cut of ridge curves
  (top, middle left), the final surface extracted via Algorithm 2
  (top, middle right), and the ground truth (top, right) are
  shown. Snapshots in the removal of faces in Alg.~2 are shown
  (bottom), resulting in the surface (right).}
\label{fig:sythetic_example}
\end{figure}

Using the last property, we can show the surface extracted by our
algorithm is a collection of minimal paths to $p$.
\begin{prop}
  Suppose $V$ is a valley of $U$, the solution of the eikonal
  equation, containing the seed point $p$ used to define $U$. Then $V$
  is a union of minimal paths to $p$.
\end{prop}
\begin{proof}
  Let $x\in V$ then the path $\gamma_x$ formed from the gradient
  descent of $U$ starting from $x$ stays in $V$ by
  Proposition~\ref{prop:valley_grad}. The path $\gamma_x$ is also a
  minimal path since gradient paths of $U$ are minimal paths. Note
  that $\gamma_x$ ends at $p$. Therefore, we see that $V$ is the union
  of $\gamma_x$ over all $x$.
\end{proof}

\section{Experiments}

\label{sec:expts}

Supplementary video are
available\footnote{\url{https://sites.google.com/site/surfacecut/pami}}. We
qualitatively and quantitatively assess our method by comparing
against competing algorithms.

\subsection{Datasets and Parameters}
We evaluate our method on three datasets of 3D images.

{\bf Synthetic Dataset}: We construct a synthetic dataset consisting
of 20 different surfaces with boundary at three different image
resolutions, $100\times 100\times 100$, $500\times 500\times 500$ and
$800\times 800\times 800$. Each of the surfaces have different 3D
boundary curves of different shape, and surfaces that have various
degrees of coarse and fine features. Example surfaces are shown in
Fig.~\ref{fig:synthetic_surfaces}. The images are formed by setting
pixels not within distance 1 to the surface to 1 and all other pixels
to 0. The surfaces meshes are downsampled for the lower resolution
images. Noise with level $\sigma=0.1$ is then added to the images.

\def\fHeightSynth{1.3in}
\begin{figure}
  \centering
  \includegraphics[width=\fHeightSynth]{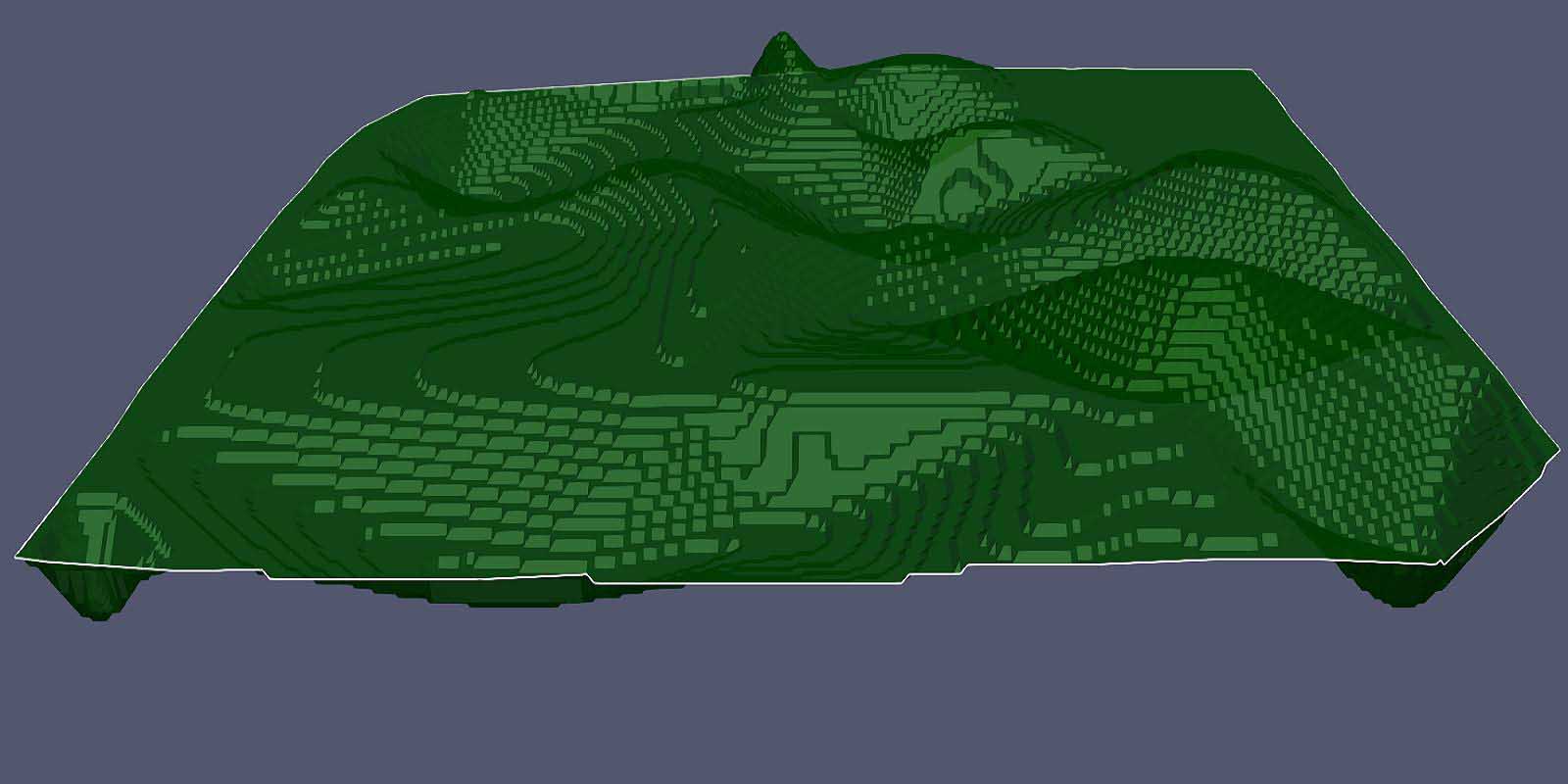}
  \includegraphics[clip, trim=0 31 0 18,
  width=\fHeightSynth]{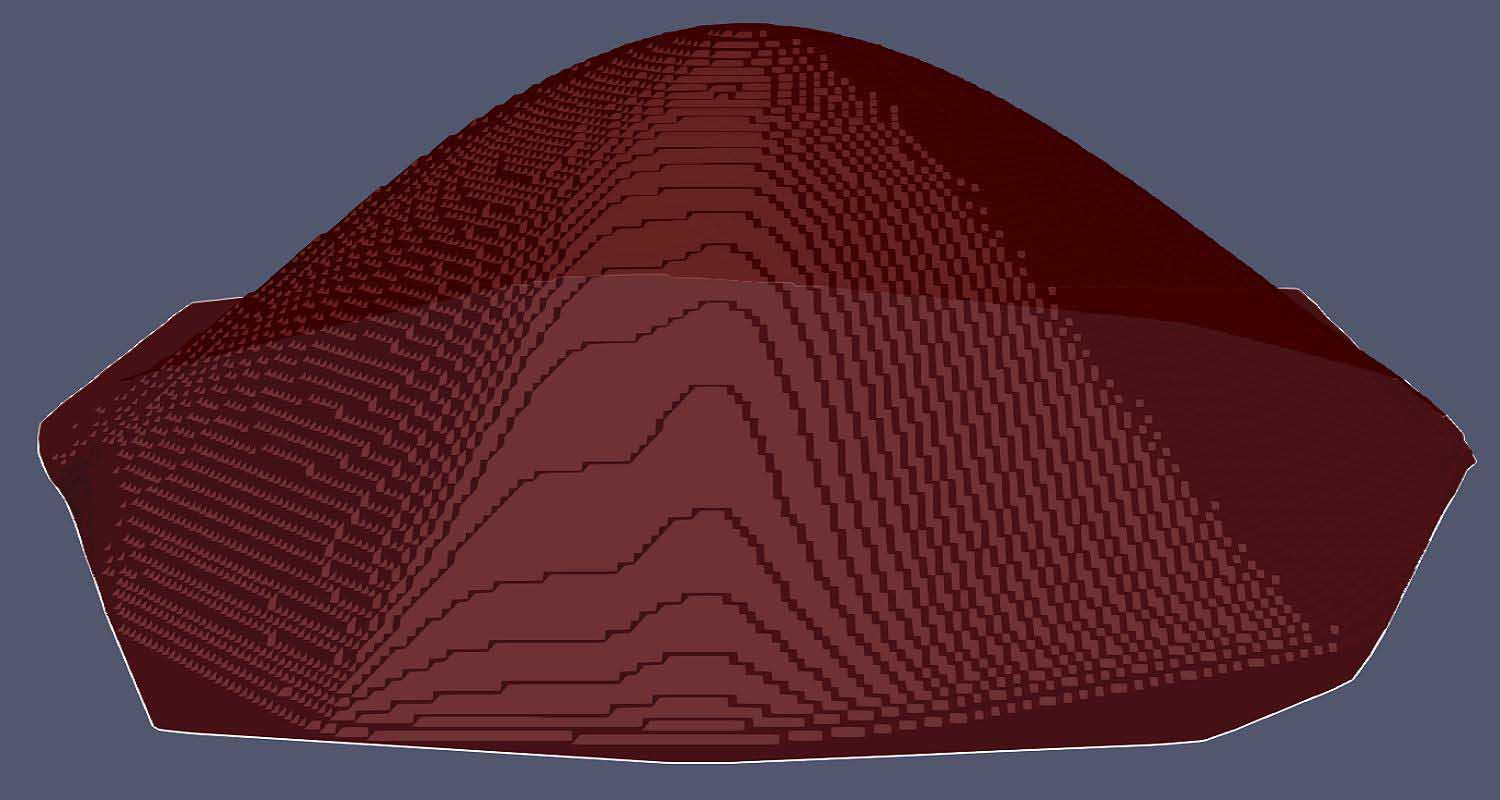}\\\vspace{1mm}%
  \includegraphics[width=\fHeightSynth]{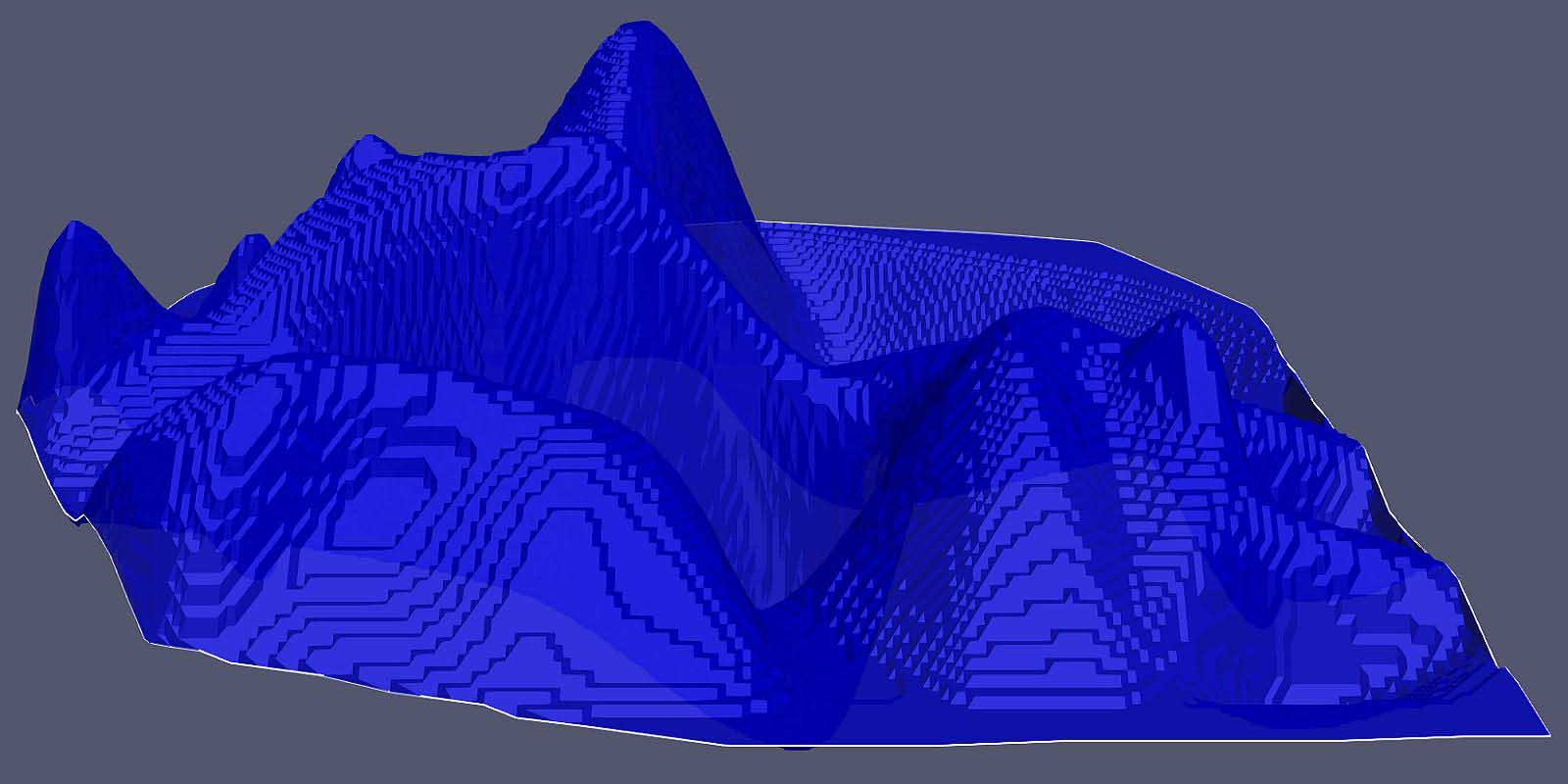}
  \includegraphics[width=\fHeightSynth]{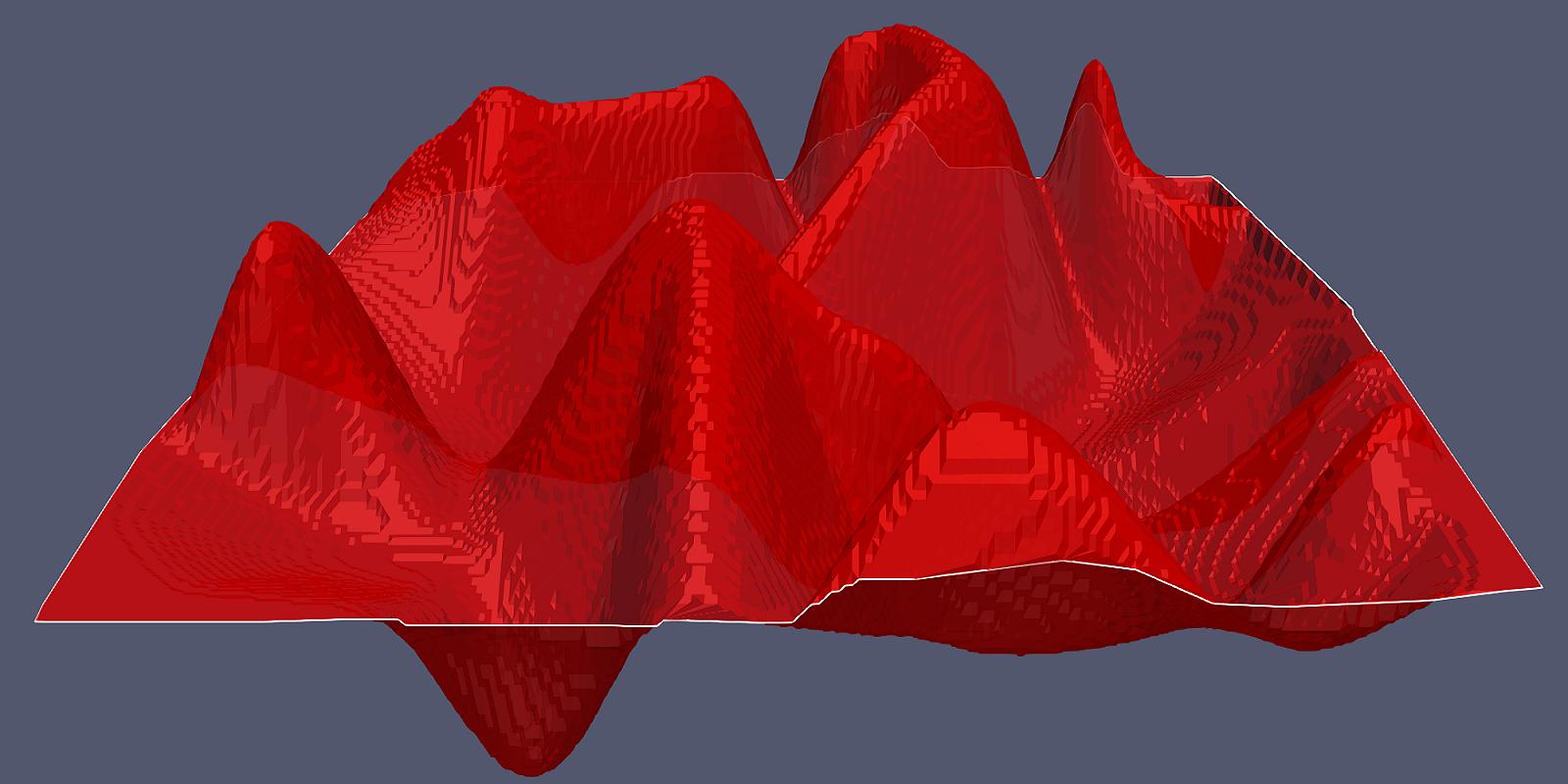}
  \caption{Example surfaces in our synthetic dataset. Each surface has
    a different boundary curve, and the surfaces are of different
    shape, exhibiting various degrees of randomness.}
  \label{fig:synthetic_surfaces}
\end{figure}

{\bf Seismic Dataset}: Seismic images are formed from measurements of
seismic pulses reflected back from the earth's sub-surface. They are
3D images, and are used to measure geological structures. We have a
dataset of three volumes with dimensions $463\times 951 \times
651$. The goal is to extract fault surfaces, which form free-boundary
surfaces within the volume. Faults may have significant curvature, and
the boundaries are non-planar. The images are cluttered and noisy, and
faults can be found by locating discontinuities, which is difficult
due to subtle edges. Each image consists of multiple faults. We have
obtained ground truth segmentations (human annotated) of two faults
within each image for each slice.

{\bf Lung CT Dataset}: We use a dataset of 10 3D computed tomography
(CT) of the lung of cancer patients from the Cancer Imaging Archive
(TCIA) \cite{hugo2016}. Each image has size $512\times 512 \times Z$,
where $Z$ varies between $300$ and $700$, depending on the
patient. Our goal is to segment lung fissures (e.g.,
\cite{lassen2013automatic,xiao2016pulmonary}), which are the
boundaries between sections of the lung. They are very thin, subtle
structures, and form free-boundary surfaces. Each of the lung fissures
in each image is human annotated, for every slice.

%https://wiki.cancerimagingarchive.net/display/Public/Collections

{\bf Parameters}: Our algorithm, given the local surface likelihood
$\phi$, requires only one parameter, the threshold on the cut cost. In
all experiments, we choose this to be $T=5$. This is not sensitive to
the data (see Supplementary).

\subsection{Evaluation Methodology}

We validate our results with quantification measures for both the
accuracy of the surface boundary and the surface using quantities
analogous to the precision, recall and F-measure. We represent the
surface and its boundary as voxels. Let $S_{r}$ denote the surface
returned by an algorithm and let $S_{gt}$ be the ground truth
surface. Denote by $\partial S_{r}$ and $\partial S_{gt}$ the
respective boundaries. We define
\begin{align*}
  N( r\rightarrow gt) &= |\{ v \in S_r \,:\, d_{S_{gt}}(v) <
                        \varepsilon \} | \\
  N( gt\rightarrow r)  &= |\{ v \in S_{gt} \,:\, d_{S_r}(v) <
                         \varepsilon \}| \\
  P_S &= N( r\rightarrow gt)   /  |S_r |, \\
  R_S &= N( gt\rightarrow r) / |S_{gt}|, \\
  F_S &= 2 P_S R_S / (P_S + R_S) \\
  \text{GT Cov.} &=   ( N( r\rightarrow gt) + N( gt\rightarrow r) ) / (  |S_r | +  |S_{gt}| )
\end{align*}
where $d_S(v)$ denotes the distance between $v$ and the closet point
to $S$ using Euclidean distance, $|\cdot|$ denotes the number of
elements of the set, and $\varepsilon > 0$. The precision measures how
close the returned surface matches to the ground truth surface.  The
recall defined above measures how close the ground truth matches to
the surface. The $F$-measure provides a single quantity summarizing
both precision and recall. GT-Cov.~is another metric summarizing both
the precision and recall. All quantities are between 0 and 1 (higher
is more accurate). The precision and recall are similar to accuracy and
completness for closed surfaces in evaluating stereo reconstruction
algorithms \cite{seitz2006comparison}.  We similarly define precision
$P_{\partial S}$, recall $R_{\partial S}$ and $F$-measure for
$\partial S_r$ and $\partial S_{gt}$ using the same formulas but with
the surfaces replaced with their boundaries. We set $\varepsilon = 3$
to account for inaccuracies in the human annotation.

\subsection{Evaluation}

\subsubsection{Synthetic Data: Surface Extraction Given Boundary}

We first evaluate three methods for surface extraction given a
3D-boundary curve of the surface, discrete-minimal surface computed
with linear programming (LP) \cite{grady2010minimal}, discrete-minimal
surface approximated with Minimum-Cost Network Flow (MCNF)
\cite{grady2010minimal, goldberg1997efficient,
  brunsch2015smoothed,ford2015flows}, and our surface extraction,
described in Section~\ref{subsec:surface_extraction}.  We use Gurobi's
state-of-the-art linear programming implementation, to implement LP.
We use the Lemon library \cite{dezsHo2011lemon} to implement MCNF.
There are no other methods that solve this problem. We choose $\phi$
to be the image. All methods are provided the ground truth 3D boundary
curves. We evaluate the methods in terms of computational time, and in
terms of surface accuracy. A summary of results are provided in
Table~\ref{tab:synth}. Average of results over all the images are
provided. Our method is computationally faster than all other methods
at all resolutions. LP is unable to perform in a reasonable time frame
for images sizes above $100^3$, and MCNF is unable to perform for
image sizes above $500^3$. At all resolutions, our method is
faster. Speeds are reported on a single Pentium 2.3 GHz processor. The
accuracy of our method is also the highest on all measures, but all
have similar accuracies. The advantage of our method is clearly speed,
and ability to deal with high resolution images.  Note that the
analysis was not extended to the real datasets as they have high
resolution, making it too computationally expensive to test, and
down-sampling the images destroys the structures to be extracted.

\begin{table}
  \caption{Comparison of methods for surface extraction given the
    surface boundary on the synthetic dataset. Speed (in seconds),
    surface precision (P), recall (R), F-measure (F), and ground truth
    covering (GT-cov) are reported. Higher P, R, F, GT-Cov.~indicate
    better fidelity to the ground truth.}
  \label{tab:synth}
  \centering
  $100 \times 100 \times 100$ pixel images \\ \vspace{2mm}
  \begin{tabular}{ l | l | l | l | l | l }
    Method & Time  & $F$ &  GT-Cov. &  $P$ &  $R$ \\
    \hline
    LP & 1167 & 0.93$\pm$0.01 & 0.94$\pm$0.01 & 0.91$\pm$0.02 & 0.96$\pm$0.01    \\
    MCNF & 12.75 & 0.92$\pm$0.01 & 0.90$\pm$0.01 & 0.93$\pm$0.01 & 0.92$\pm$0.02    \\
    Surfcut & {\bf 1.87} & {\bf 0.95$\pm$0.02} & {\bf 0.95$\pm$0.02} &
                                                                       {\bf
                                                                       0.96$\pm$0.02}
                                           & {\bf 0.94$\pm$0.03}    \\
    \hline
  \end{tabular}\\ \vspace{3mm}
  $500 \times 500 \times 500$ pixel images \\ \vspace{2mm}
  \begin{tabular}{ l | l | l | l | l | l }
    Method & Time & $F$ &  GT-Cov. &  $P$ &  $R$ \\
    \hline
    LP & $>$24hr & NA & NA & NA & NA    \\
    MCNF & 35614 & 0.94$\pm$0.01 & 0.92$\pm$0.01 & 0.94$\pm$0.01 & 0.93$\pm$0.01    \\
    Surfcut & {\bf 421} &  {\bf 0.96$\pm$0.01} &  {\bf 0.96$\pm$0.01} &  {\bf 0.97$\pm$0.01} &  {\bf 0.94$\pm$0.01}    \\
    \hline
  \end{tabular}\\ \vspace{3mm}
  $800 \times 800 \times 800$ pixel images \\ \vspace{2mm}
  \begin{tabular}{ l | l | l | l | l | l }
    Method & Time & $F$ &  GT-Cov. &  $P$ &  $R$ \\
    \hline
    LP & $>$24hr & NA & NA & NA & NA \\
    MCNF & $>$24hr & NA & NA & NA & NA \\
    Surfcut &  {\bf 2227} &  {\bf 0.96$\pm$0.01} &  {\bf 0.97$\pm$0.01} &  {\bf 0.98$\pm$0.01} &  {\bf 0.95$\pm$0.02}    \\
    \hline
  \end{tabular}
\end{table}

\subsubsection{Seismic Data: Surface and Boundary Extraction}

We now compare against the competing method for free boundary surface
extraction.  To the best of our knowledge, there is no other general
algorithm that extracts both the boundary of the free-surface and the
surface given a seed point. Therefore, we compare our method in an
interactive setting and automated setting (with seed points
automatically initialized) to Crease Surfaces
\cite{schultz2010crease}. It computes the smoothed Hessian of $\phi$,
and computes a modified matrix based on the relative difference in the
first and second highest eigenvalues. It then forms the surface by
determining locations where the eigenvector aligns with the gradient,
and constructs connected surfaces. In an interactive setting, we
choose the surface returned by \cite{schultz2010crease} that is near
to the user provided seed point (and best fits ground truth) to
provide comparison to our method. In an automated setting, we use a
seed point extraction algorithm (described later) to initialize our
surface extraction.

\comment{
\cite{schultz2010crease} returns all surfaces by detecting connected
components of maxima of the local surface likelihood map, which is
computed from $\phi$. It constructs a likelihood map from of the
eignevalues of the smoothed Hessian of $\phi$ at a pixel, and setting
the likelihood proportional to the relative difference in the first
and second highest eignevalues.
}

We choose $\phi(x)$ to be the semblance measure in
\cite{hale2013methods}; this along with \cite{schultz2010crease} is
state-of-the-art for seismic data.

{\bf Robustness to Smoothing Degradations}: The semblance $\phi$
contains a smoothing parameter, which must be tuned to achieve a
desirable segmentation. Therefore, it is important that the surface
extraction algorithm be robust to changes in the parameter of the
likelihood. Thus, we evaluate our algorithm as we vary the smoothing
parameter.  The smoothing parameter is varied from
$\sigma = 0, 2, 3, \ldots, 14$.  We initialize our algorithm with a
user specified seed point. Quantitative results are shown in
Figure~\ref{fig:smoothing_expt}, where we plot the F-measure versus
the smoothing amount both in terms of surface and boundary
measures. Some visual results of the surfaces are shown in
Fig.~\ref{fig:smoothing_degrad_vis}.  Notice our method degrades only
gradually and maintains consistently high accuracy in both measures in
contrast to \cite{schultz2010crease}.

\def\fHeightFm{2.2cm}
\begin{figure}
  \centering
  \includegraphics[totalheight=\fHeightFm]{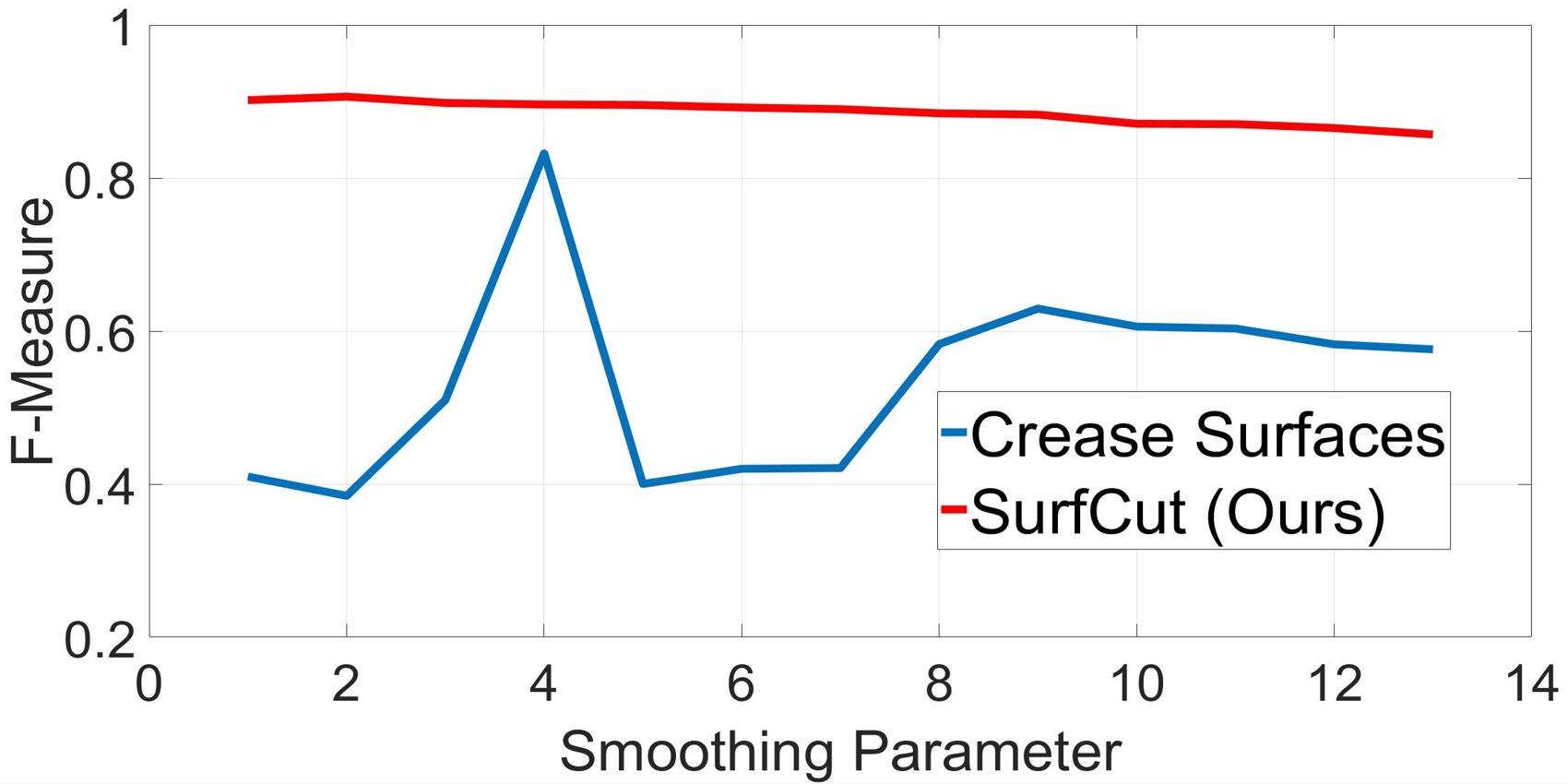}
  \includegraphics[totalheight=\fHeightFm]{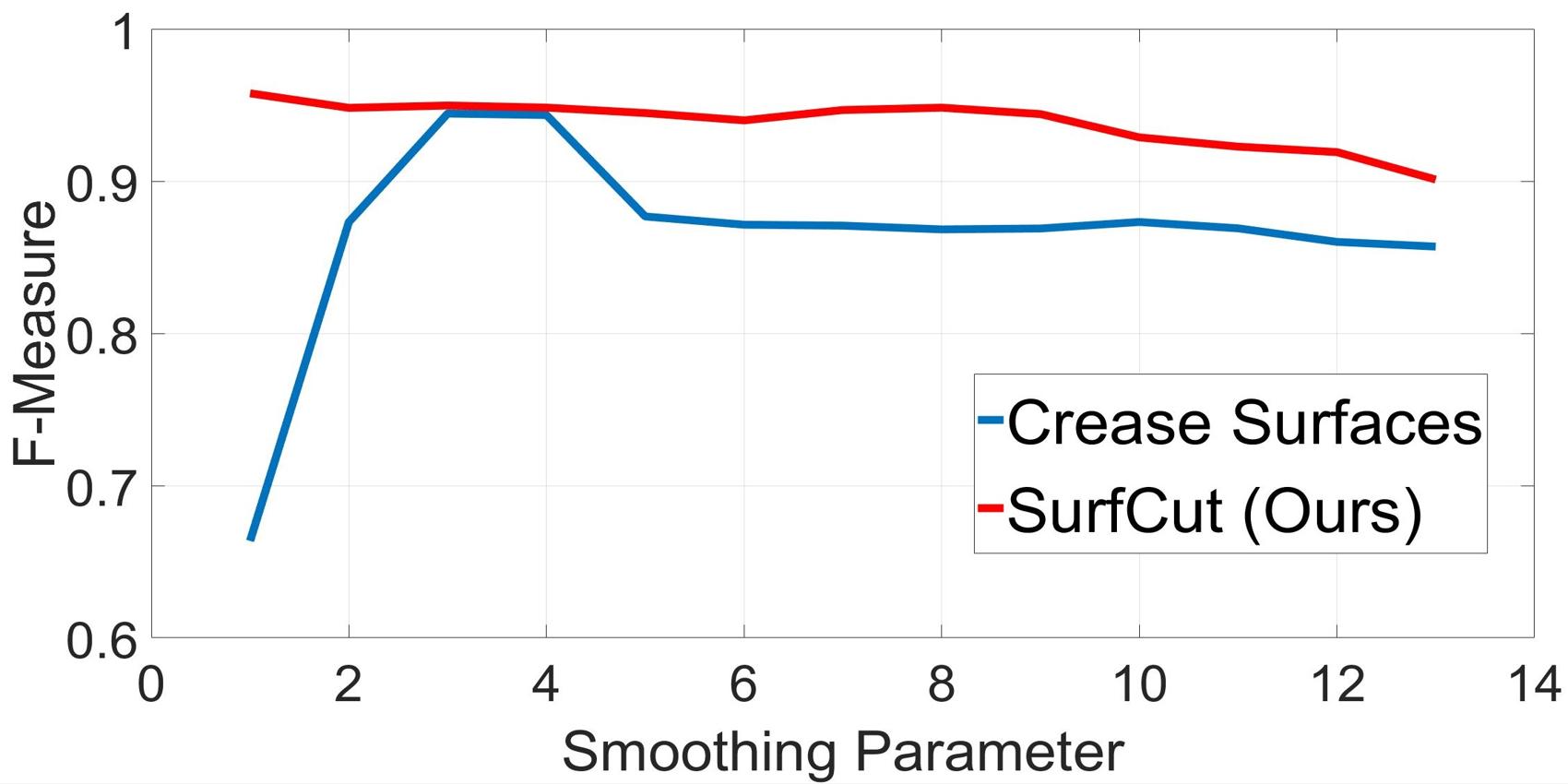}
  \caption{{\bf Quantitative Analysis of Smoothing Parameter}
    Boundary (left) and surface (right) F-measure versus smoothing degradations for our method and \cite{schultz2010crease}. }
  \label{fig:smoothing_expt}
\end{figure}

\def\fHeightFaultSmoothing{0.45in}
\begin{figure}
  \centering
  \footnotesize
  {Crease Surface (for increasing smoothing parameter $\rightarrow$ ) \\
    \includegraphics[totalheight=\fHeightFaultSmoothing]{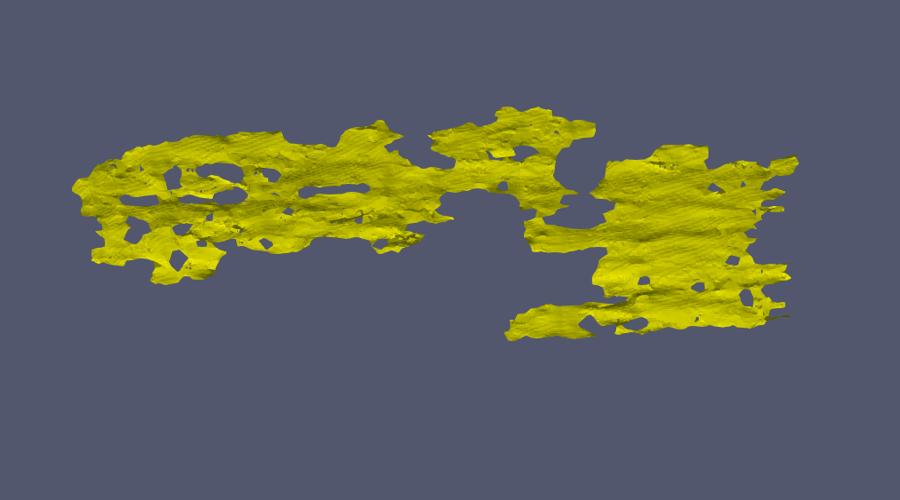} 
    \includegraphics[totalheight=\fHeightFaultSmoothing]{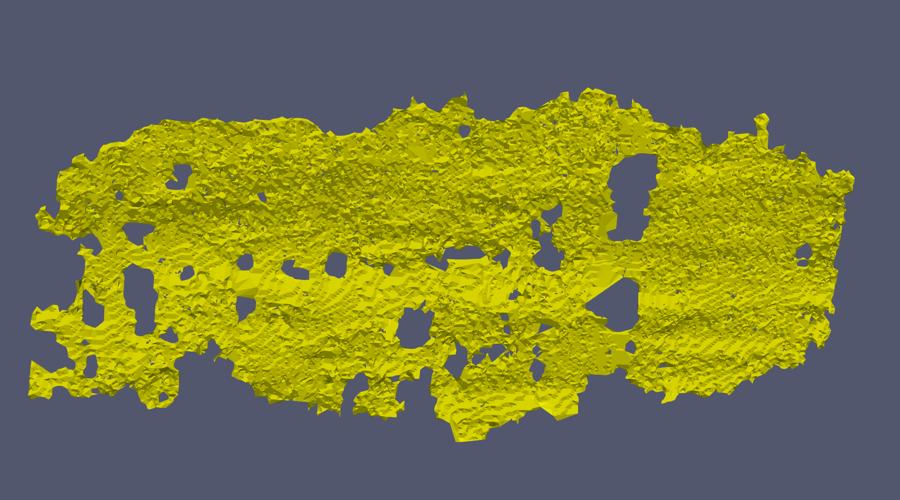} 
    \includegraphics[totalheight=\fHeightFaultSmoothing]{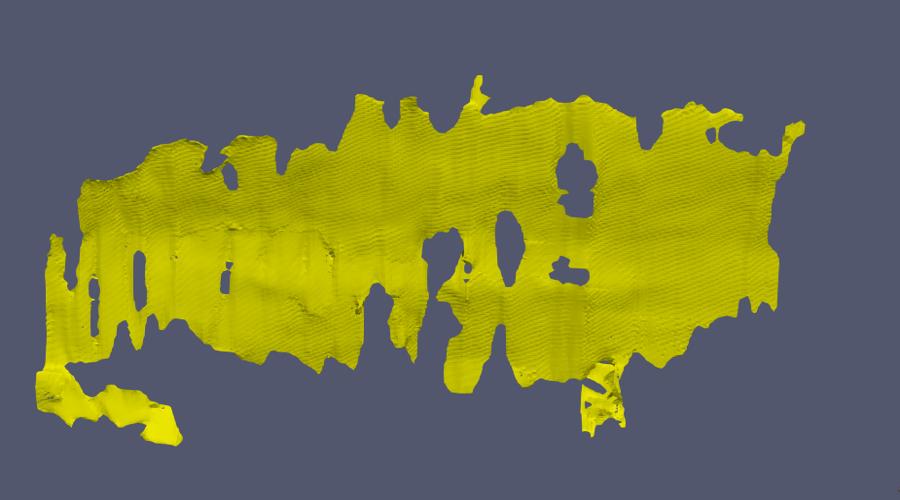} 
    \includegraphics[totalheight=\fHeightFaultSmoothing]{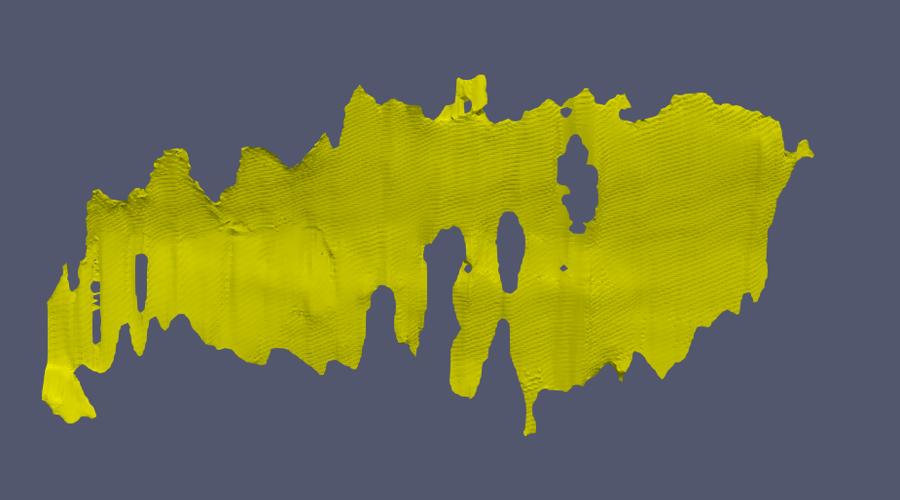} \\
    SurfCut \\
    \includegraphics[totalheight=\fHeightFaultSmoothing]{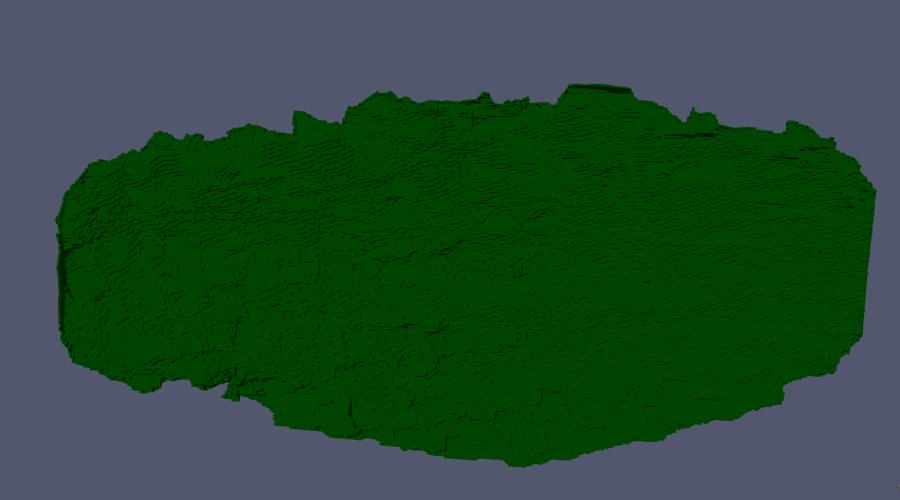} 
    \includegraphics[totalheight=\fHeightFaultSmoothing]{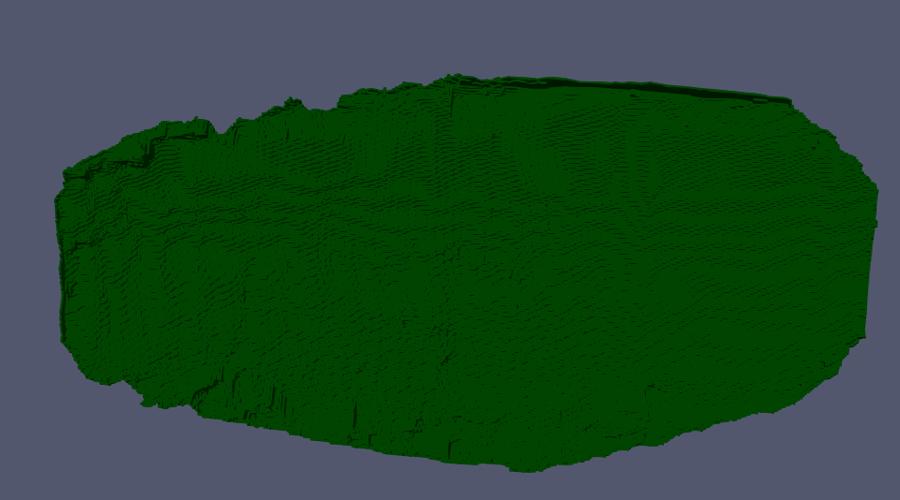} 
    \includegraphics[totalheight=\fHeightFaultSmoothing]{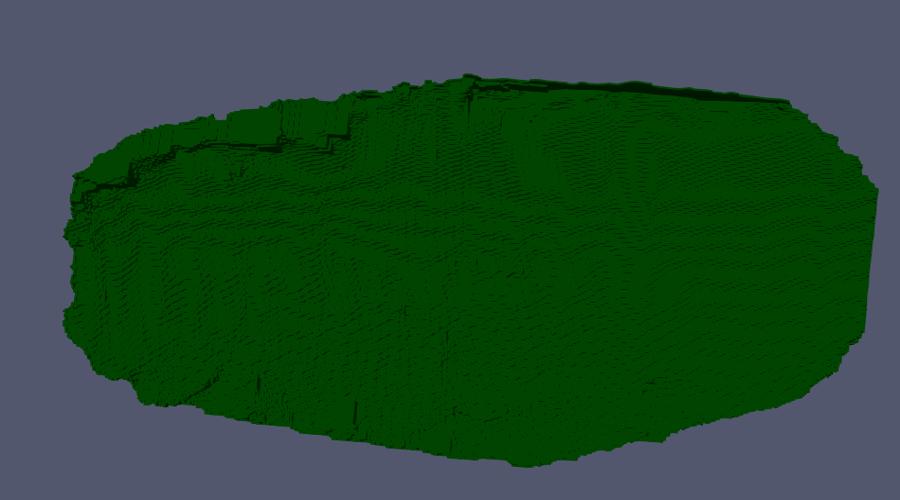} 
    \includegraphics[totalheight=\fHeightFaultSmoothing]{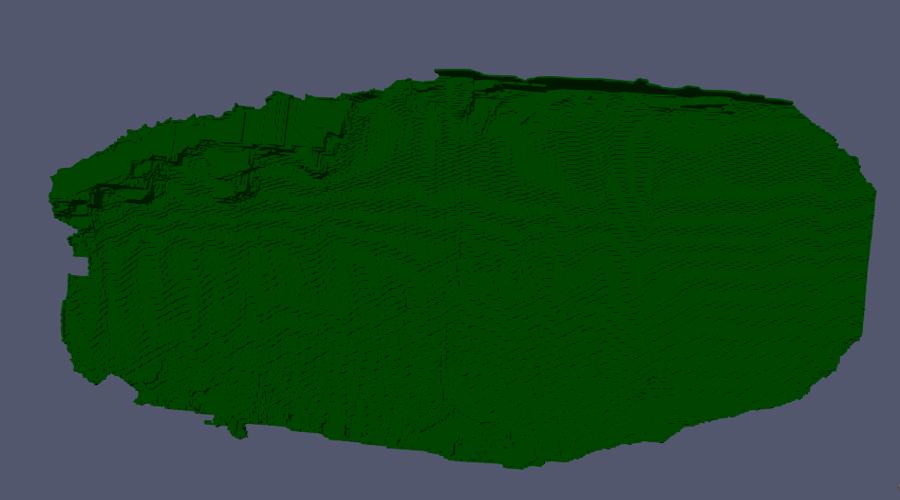} \\
    Ground Truth \\
    \includegraphics[totalheight=\fHeightFaultSmoothing]{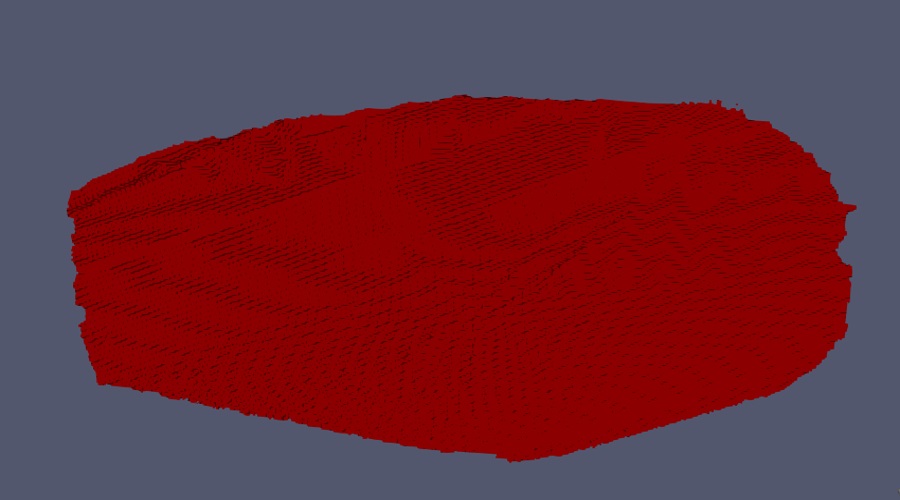} 
    \includegraphics[totalheight=\fHeightFaultSmoothing]{figures/fault_smooth/gt} 
    \includegraphics[totalheight=\fHeightFaultSmoothing]{figures/fault_smooth/gt} 
    \includegraphics[totalheight=\fHeightFaultSmoothing]{figures/fault_smooth/gt} \\
}
\caption{{\bf Qualitative Analysis of Smoothing Parameter}. Results
  displayed by varying the parameter in $\phi$ (larger towards the
  right). Surfaces extracted by Crease surfaces and our method are
  displayed with the ground truth. }
\label{fig:smoothing_degrad_vis}
\end{figure}

{\bf Robustness to Noise}: In applications, the image may be distorted
by noise (this is the case in seismic images where the SNR may be
low), and thus we evaluate our algorithm as we add noise to the image,
and we fix the smoothing parameter of the semblance $\phi(x)$ to the
one with highest F-measure in the previous experiment. We choose noise
levels as follows: $\sigma^2 = 0, 0.05, \ldots, 0.5$. Quantitative
results are shown in Figure~\ref{fig:noise_expt}, and some
visualizations of the surfaces are shown in
Fig.~\ref{fig:noise_expt_vis}. Results show that our method
consistently returns an accurate result in both measures, and degrades
only slightly.

\begin{figure}
  \centering
  \includegraphics[totalheight=\fHeightFm]{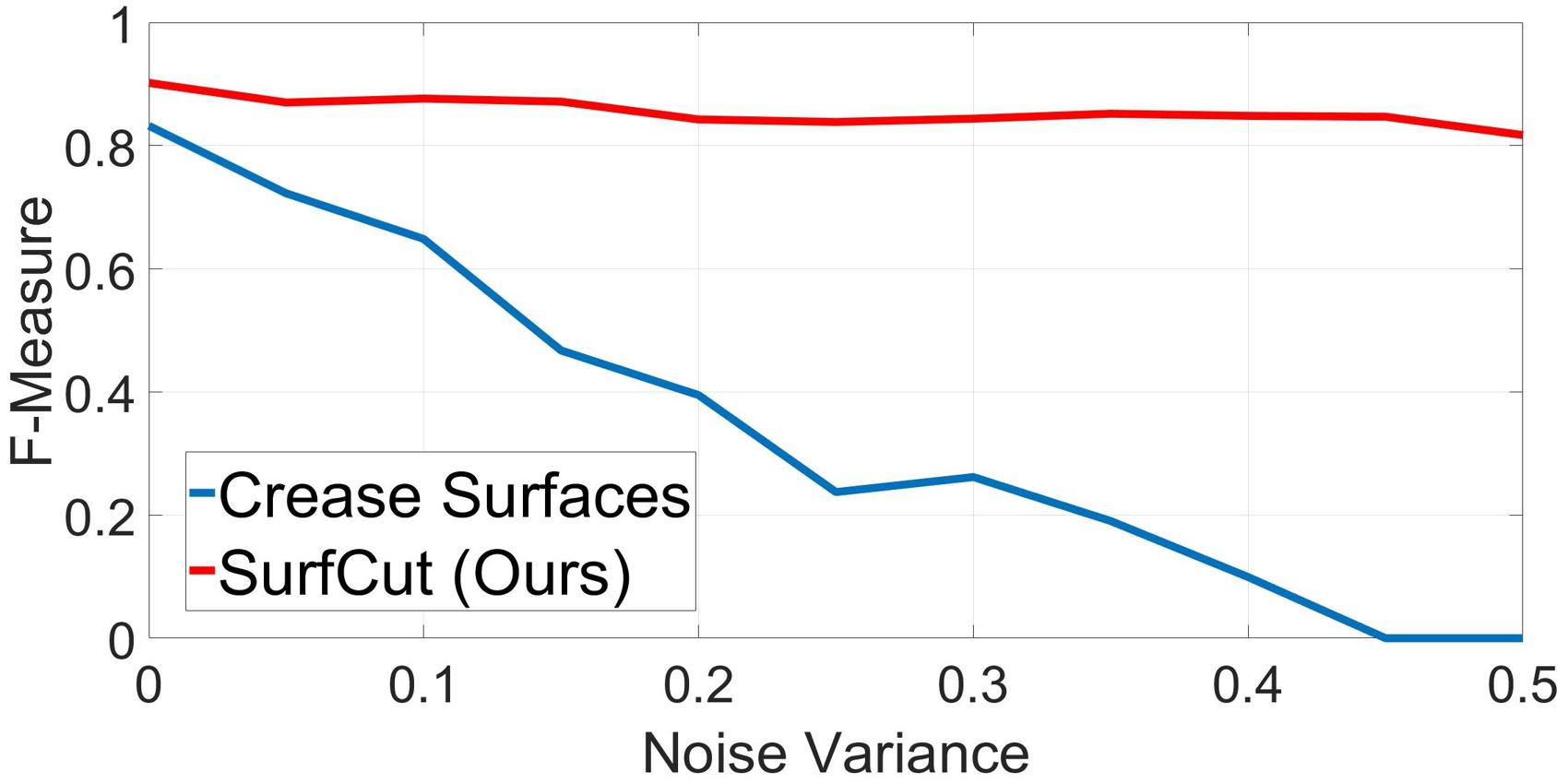}
  \includegraphics[totalheight=\fHeightFm]{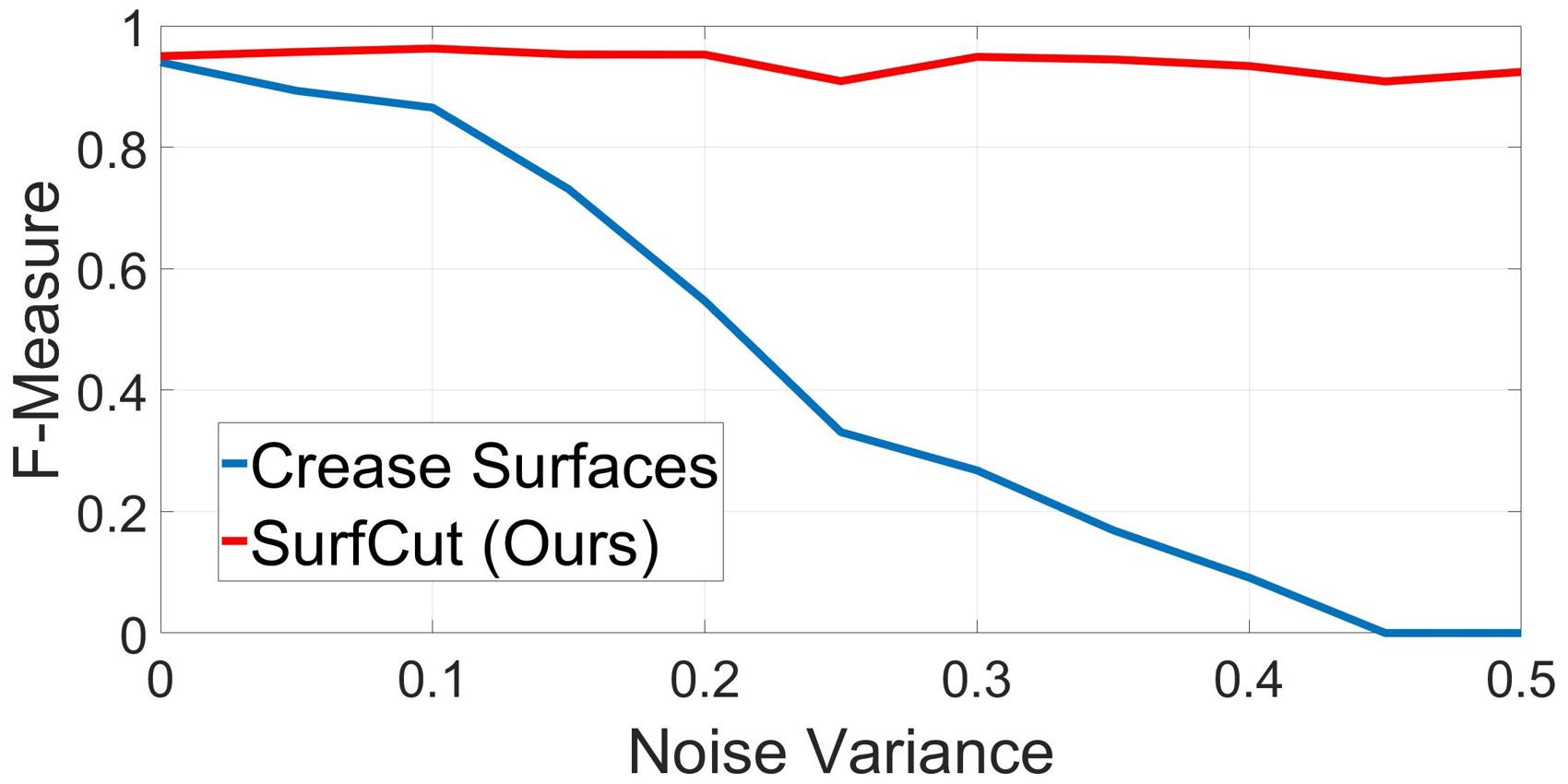}
  \caption{ {\bf Quantitative Analysis of Noise Degradations}
    Boundary (left) and surface (right)  F-measure versus the noise degradation plots for our method and \cite{schultz2010crease}.
  }
  \label{fig:noise_expt}
\end{figure}

\def\fHeightFaultNoise{0.45in}
\begin{figure}
  \centering
  \footnotesize
  {Crease Surface (for increasing noise $\rightarrow$) \\
    \includegraphics[totalheight=\fHeightFaultNoise]{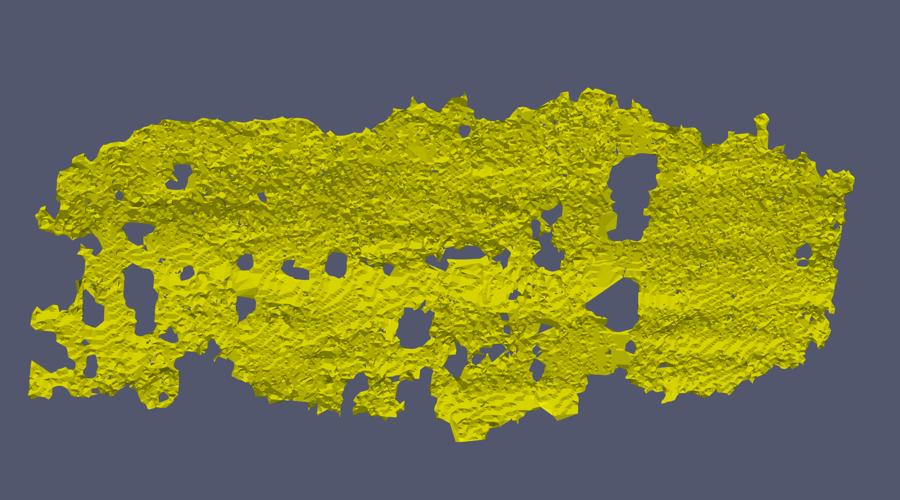} 
    \includegraphics[totalheight=\fHeightFaultNoise]{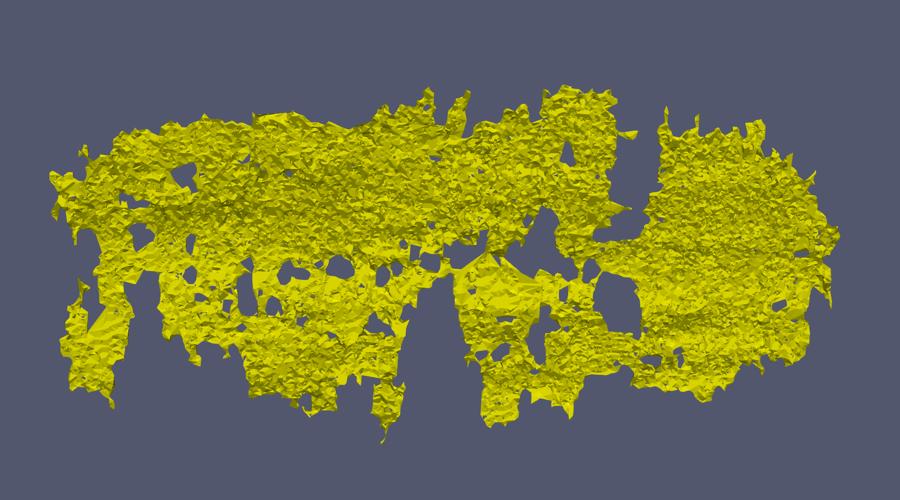} 
    \includegraphics[totalheight=\fHeightFaultNoise]{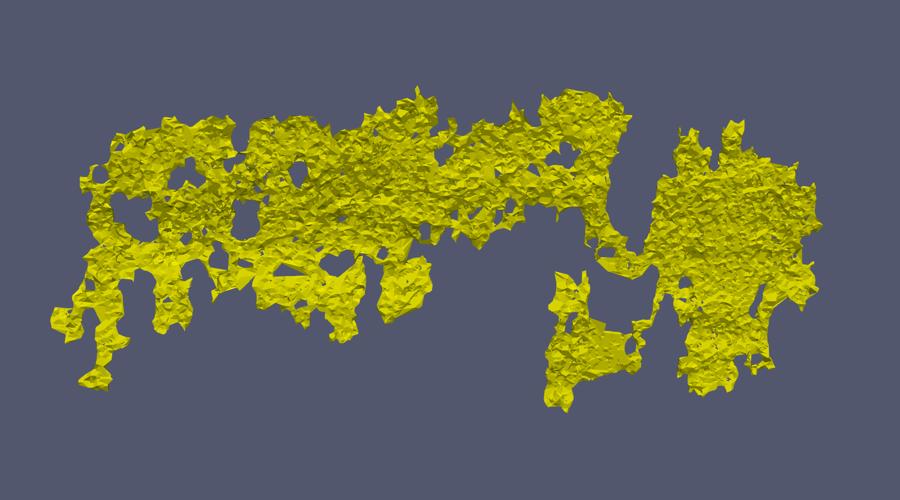} 
    \includegraphics[totalheight=\fHeightFaultNoise]{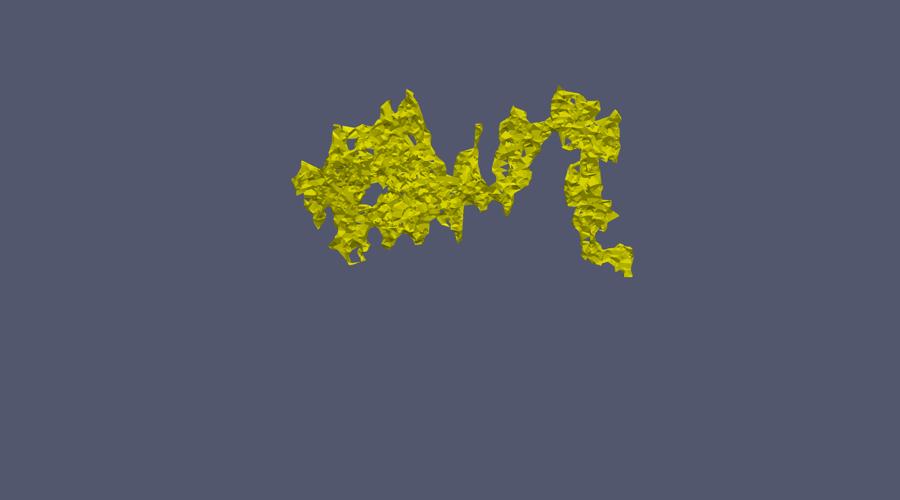} \\
    SurfCut \\
    \includegraphics[totalheight=\fHeightFaultNoise]{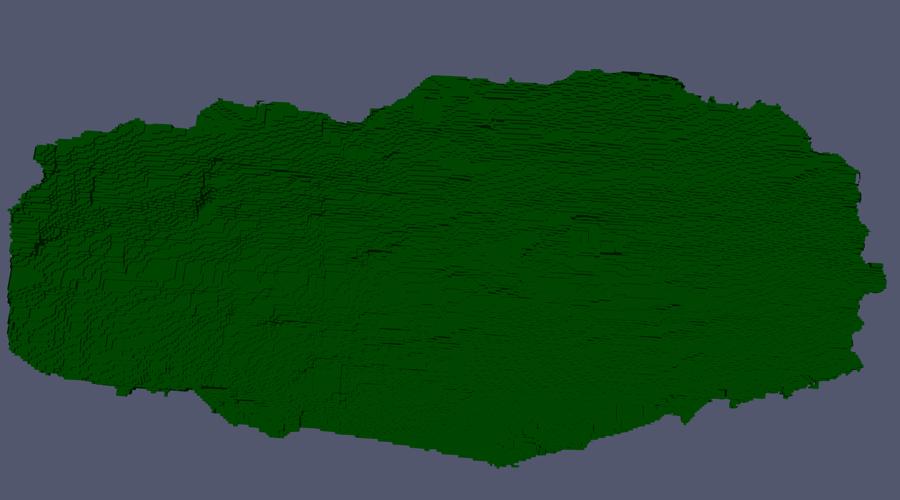}
    \includegraphics[totalheight=\fHeightFaultNoise]{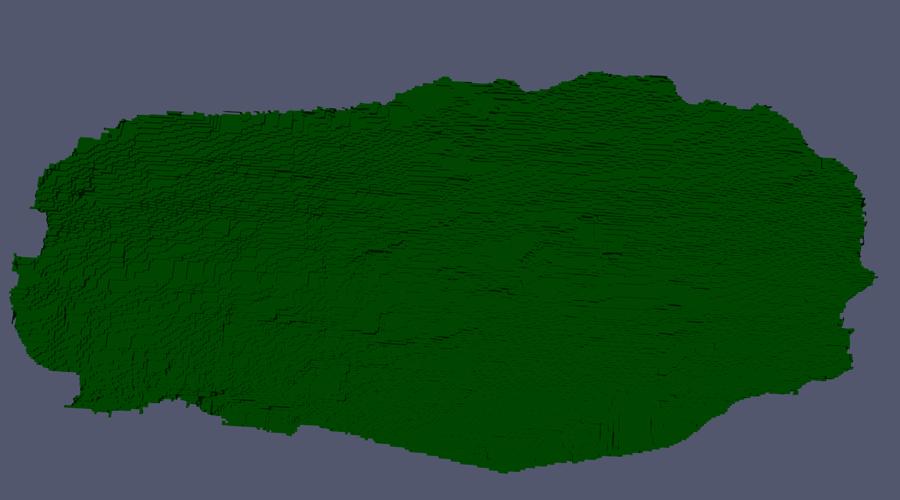}
    \includegraphics[totalheight=\fHeightFaultNoise]{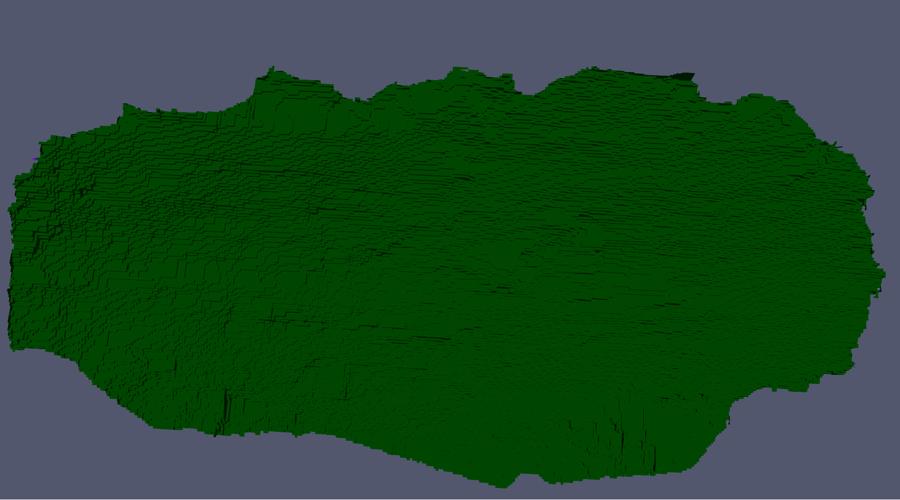}
    \includegraphics[totalheight=\fHeightFaultNoise]{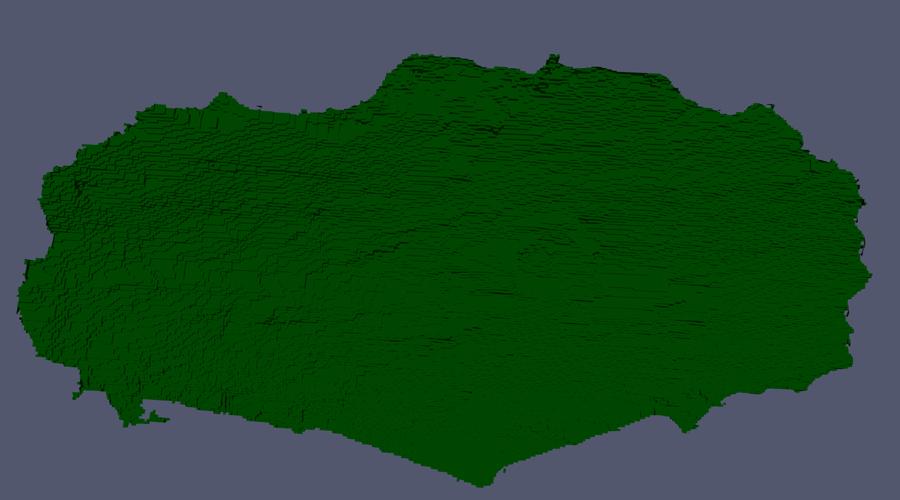} \\
    Ground Truth \\
    \includegraphics[totalheight=\fHeightFaultNoise]{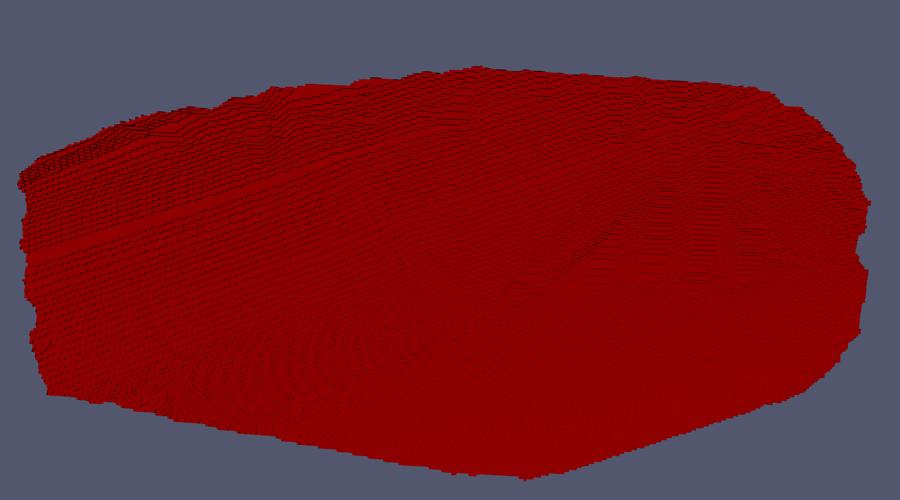} 
    \includegraphics[totalheight=\fHeightFaultNoise]{figures/fault_noise/gt} 
    \includegraphics[totalheight=\fHeightFaultNoise]{figures/fault_noise/gt} 
    \includegraphics[totalheight=\fHeightFaultNoise]{figures/fault_noise/gt} \\
  }
  \caption{{\bf Qualitative Analysis of Noise Degradations}. Results
    displayed by varying the additive noise to $\phi$ (larger towards
    the right). Surfaces extracted by Crease
    surfaces and our method are displayed with the ground truth.}
  \label{fig:noise_expt_vis}
\end{figure}

{\bf Slice-wise Validation}: We now show some visual validation of our
method by showing that the surface intersects with slices of the image
in locations where there is a fault, and thus the value of $\phi$ is
low. This is shown in Figure~\ref{fig:fault-intersection4}.

\def\fHeightFIaa{0.83in}
\def\fHeightFIbb{0.83in}
\def\fHeightFIcc{0.83in}
\begin{figure}
  \centering
  \includegraphics[width=\fHeightFIaa]{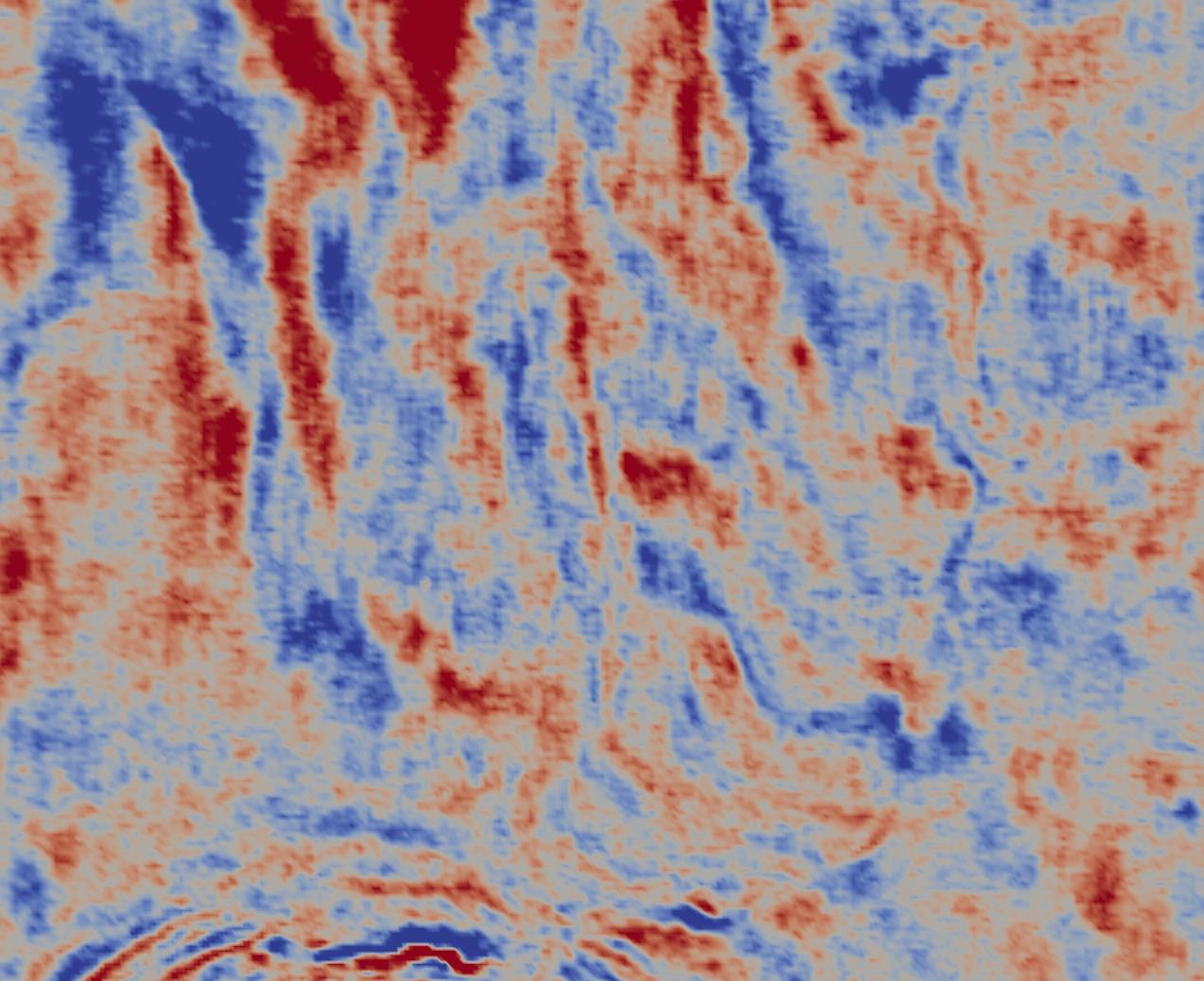}
  \includegraphics[width=\fHeightFIaa]{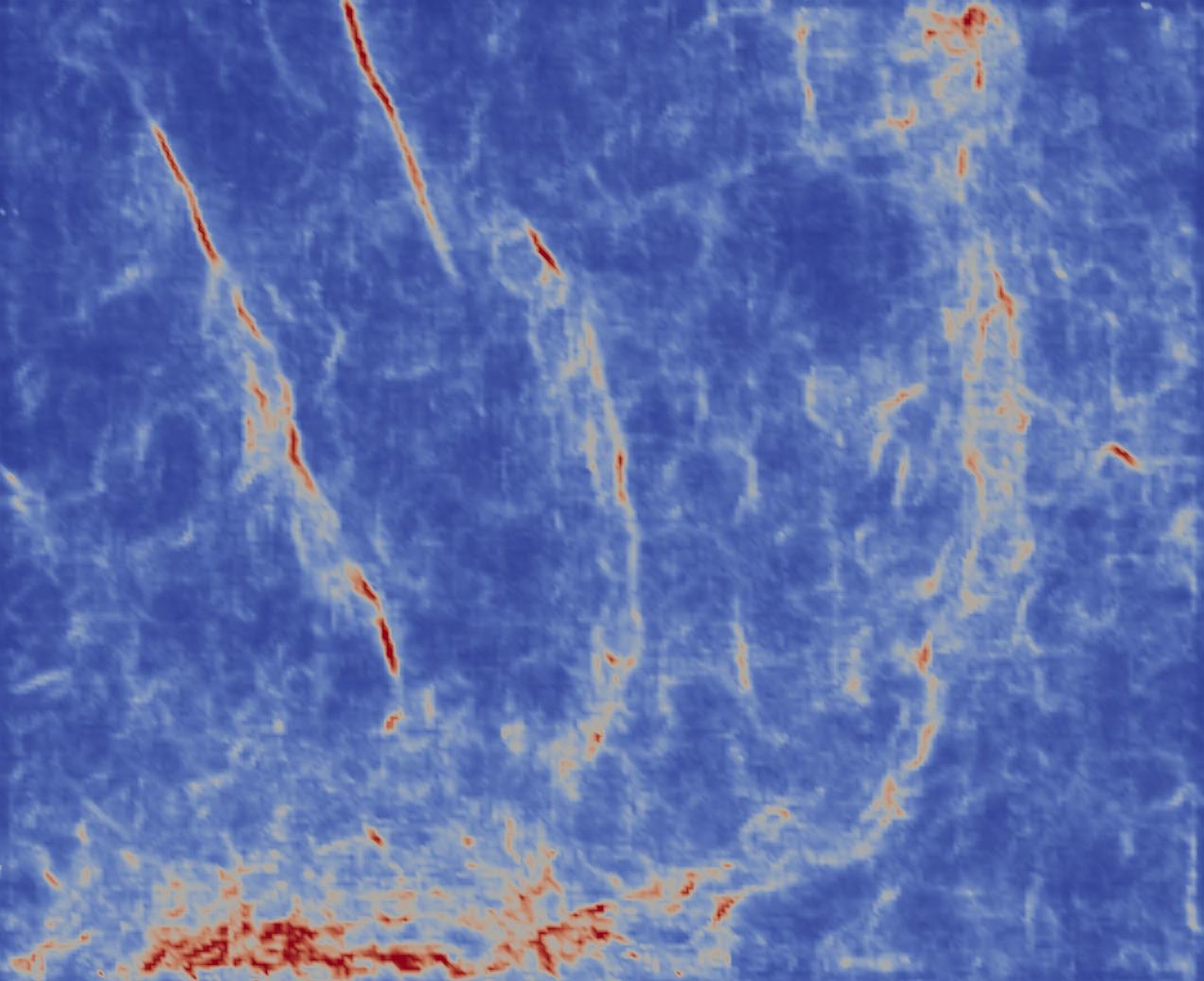}
  \includegraphics[width=\fHeightFIaa]{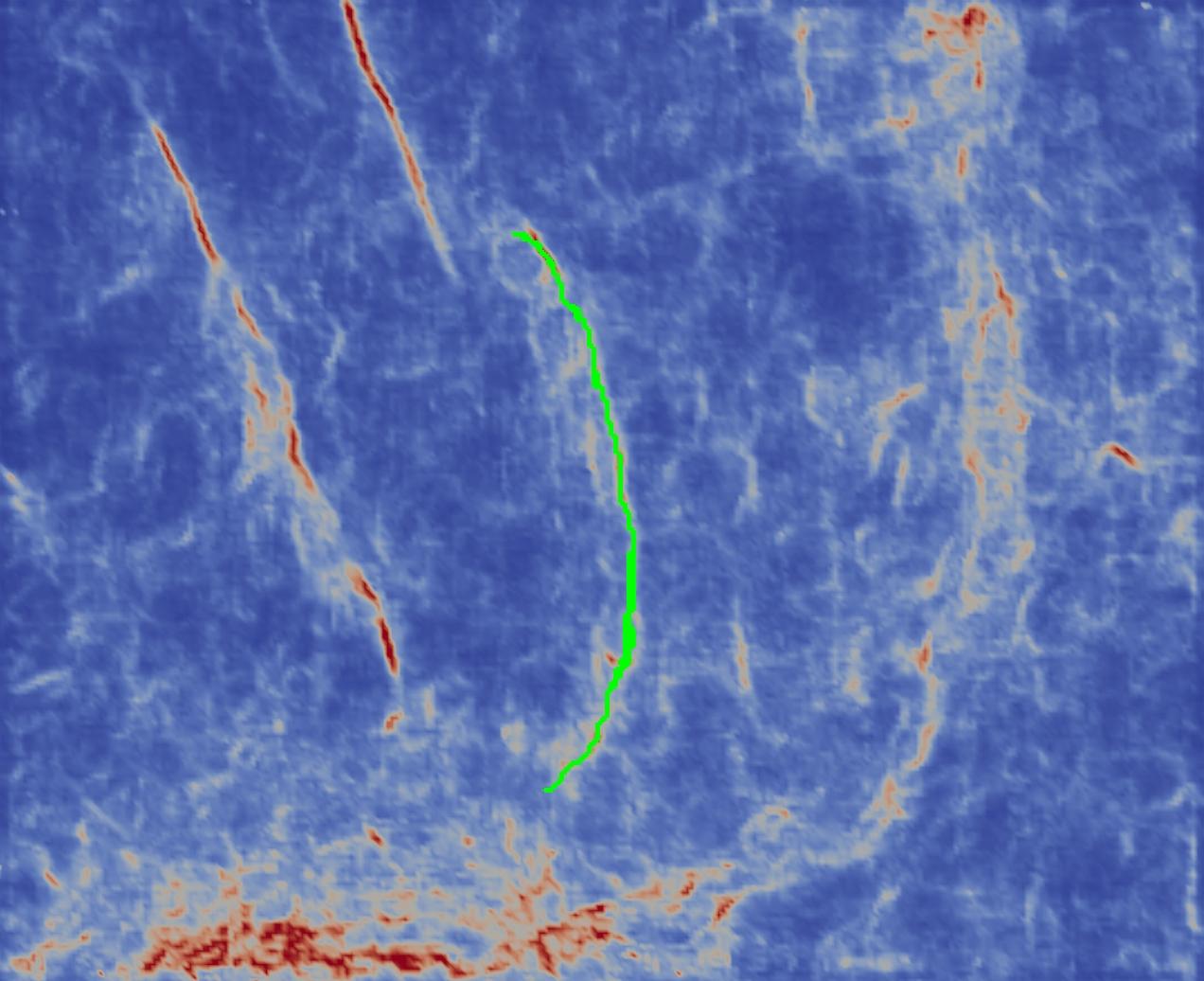}
  \includegraphics[width=\fHeightFIaa]{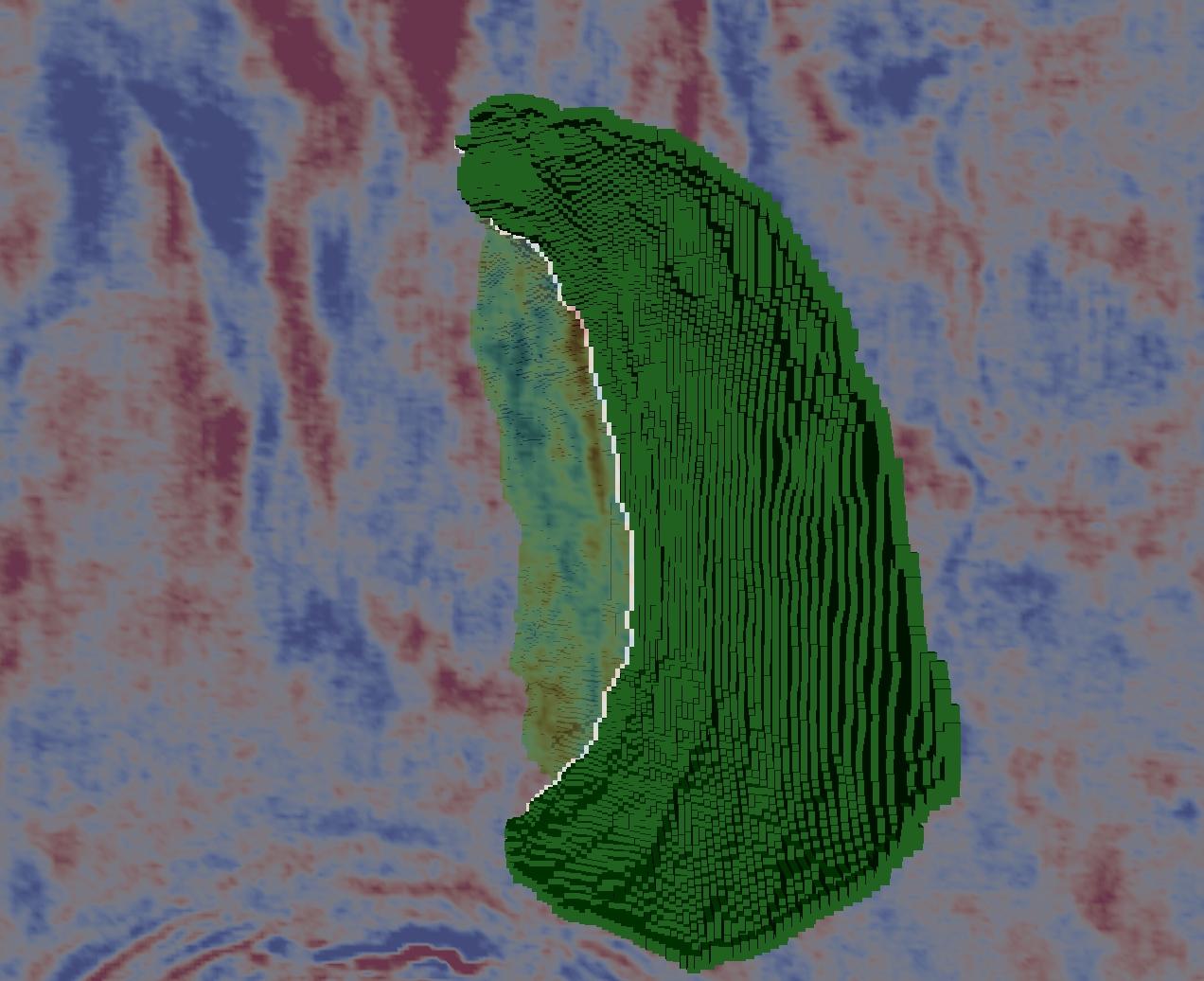}  \\ \vspace{1mm}
  \includegraphics[width=\fHeightFIbb]{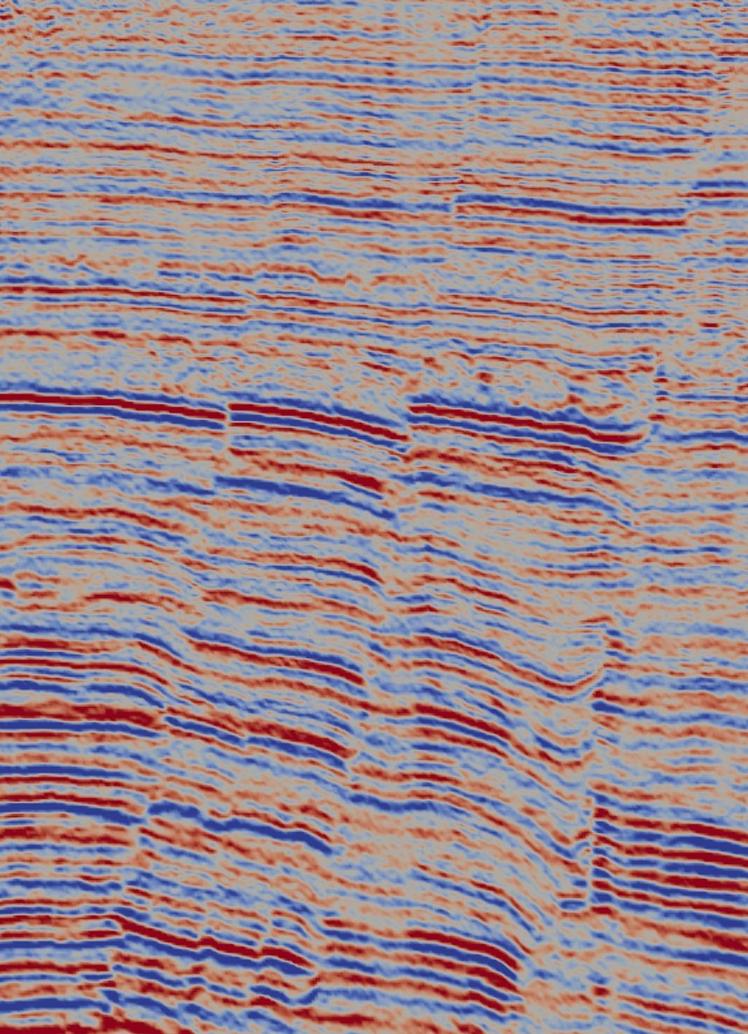}
  \includegraphics[width=\fHeightFIbb]{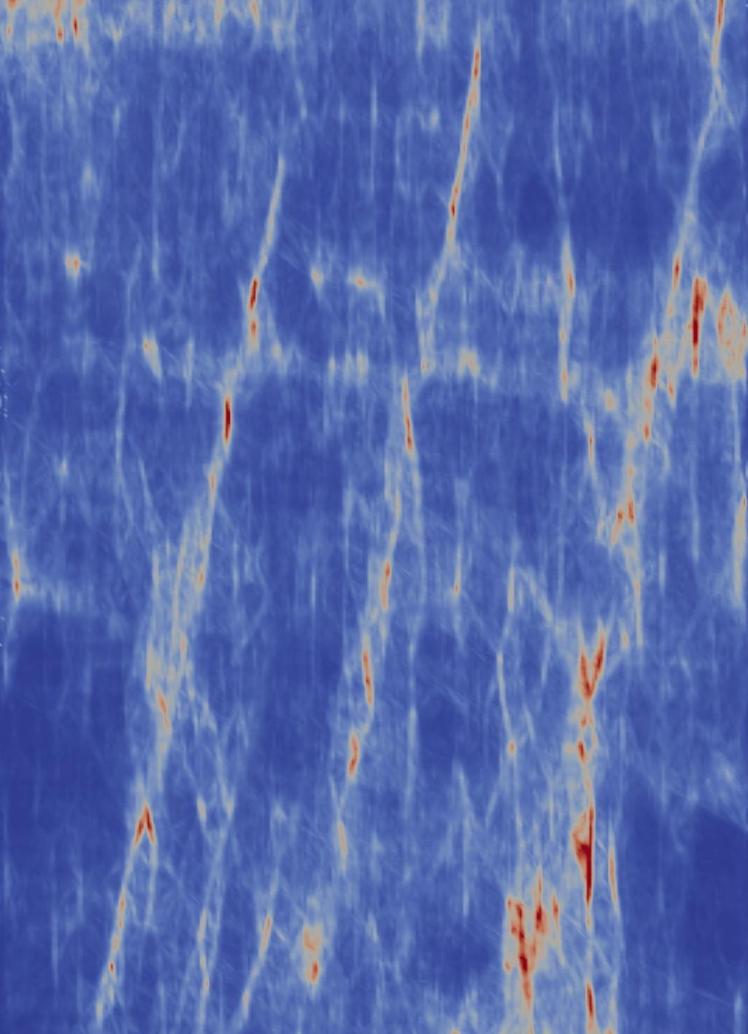}
  \includegraphics[width=\fHeightFIbb]{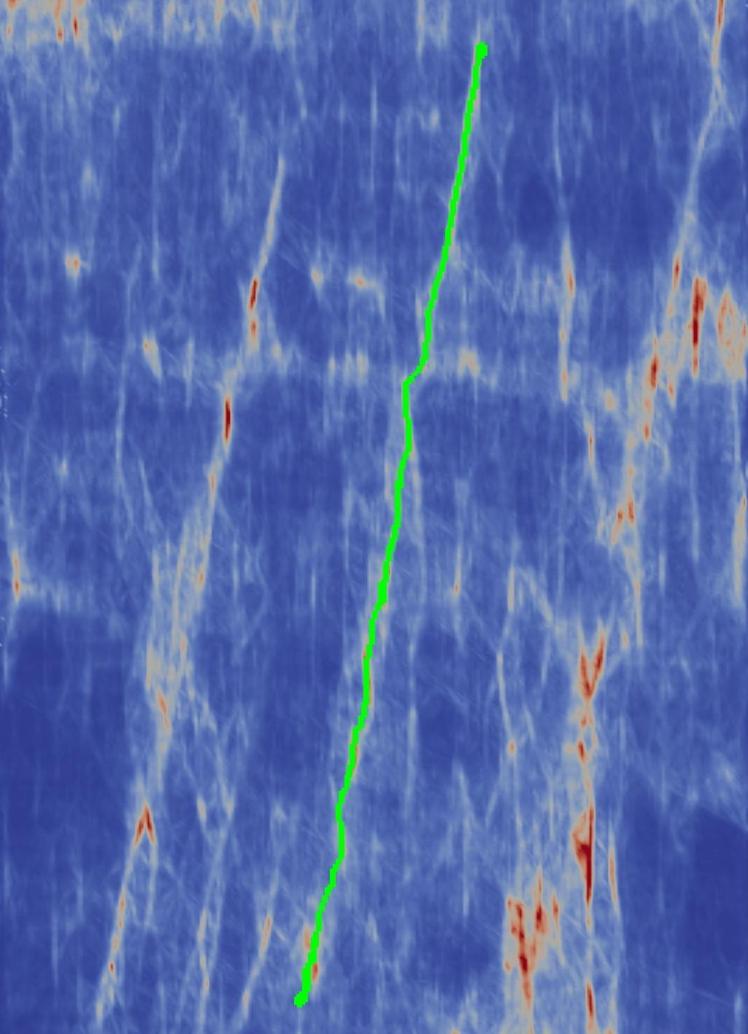}
  \includegraphics[width=\fHeightFIbb]{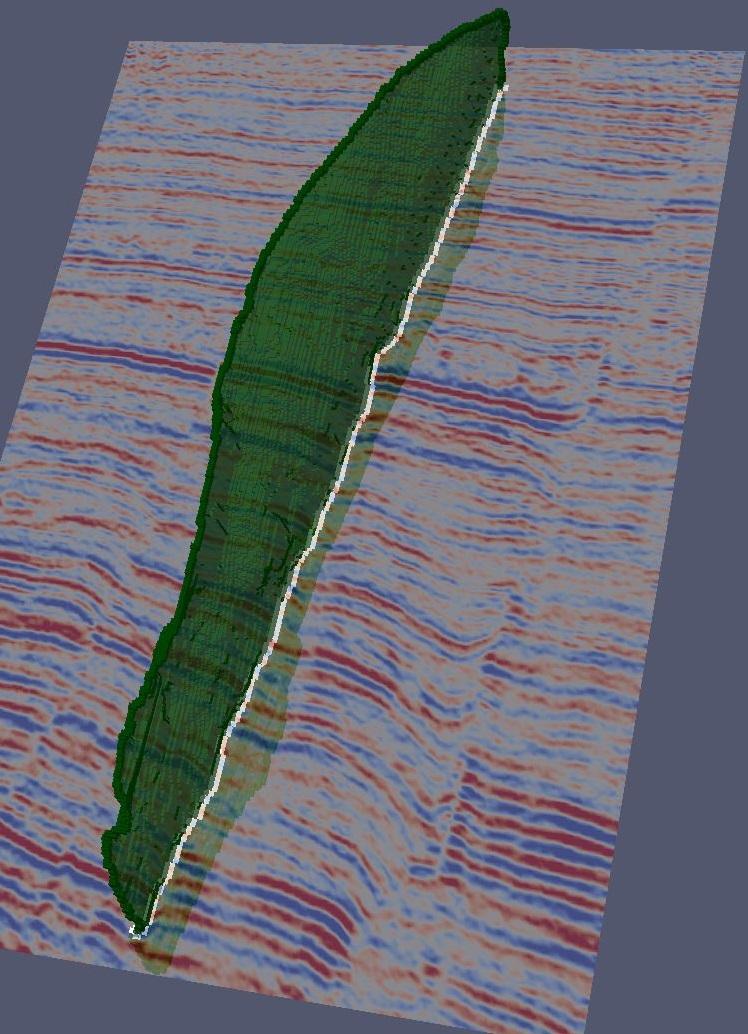}  \\ \vspace{1mm}
  \includegraphics[height=\fHeightFIcc,angle=-90,origin=c]{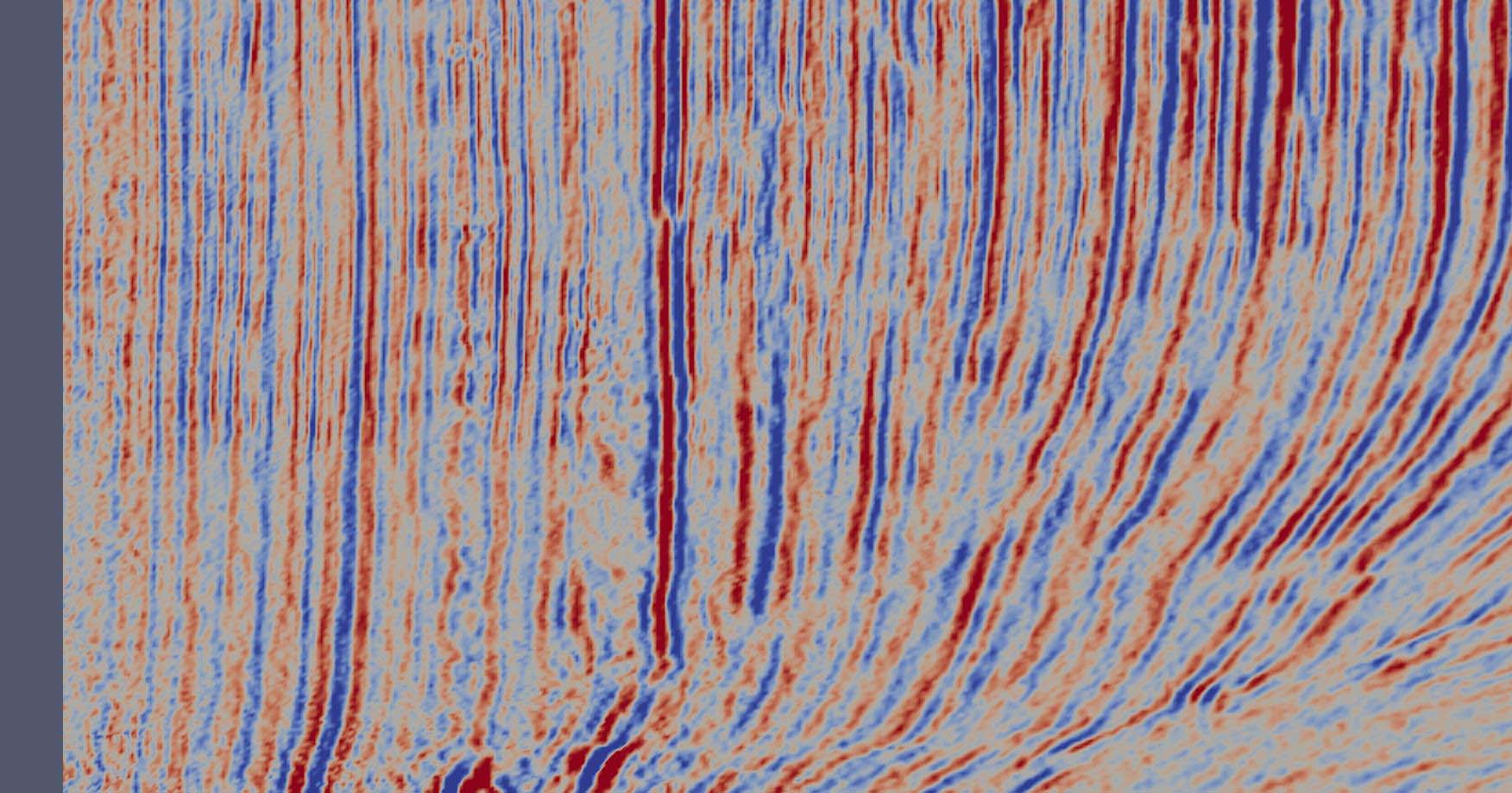}
  \includegraphics[height=\fHeightFIcc,angle=-90,origin=c]{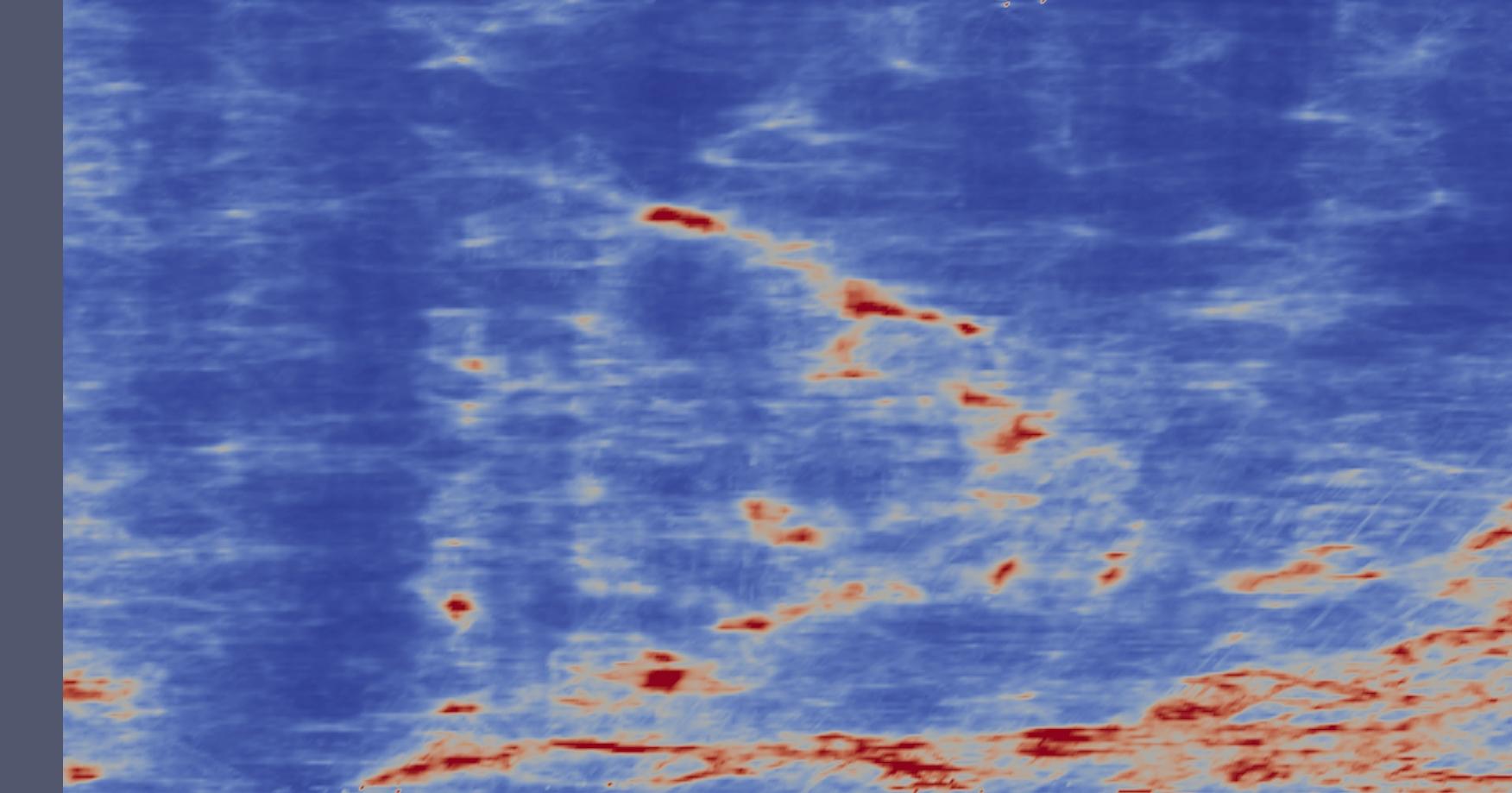}
  \includegraphics[height=\fHeightFIcc,angle=-90,origin=c]{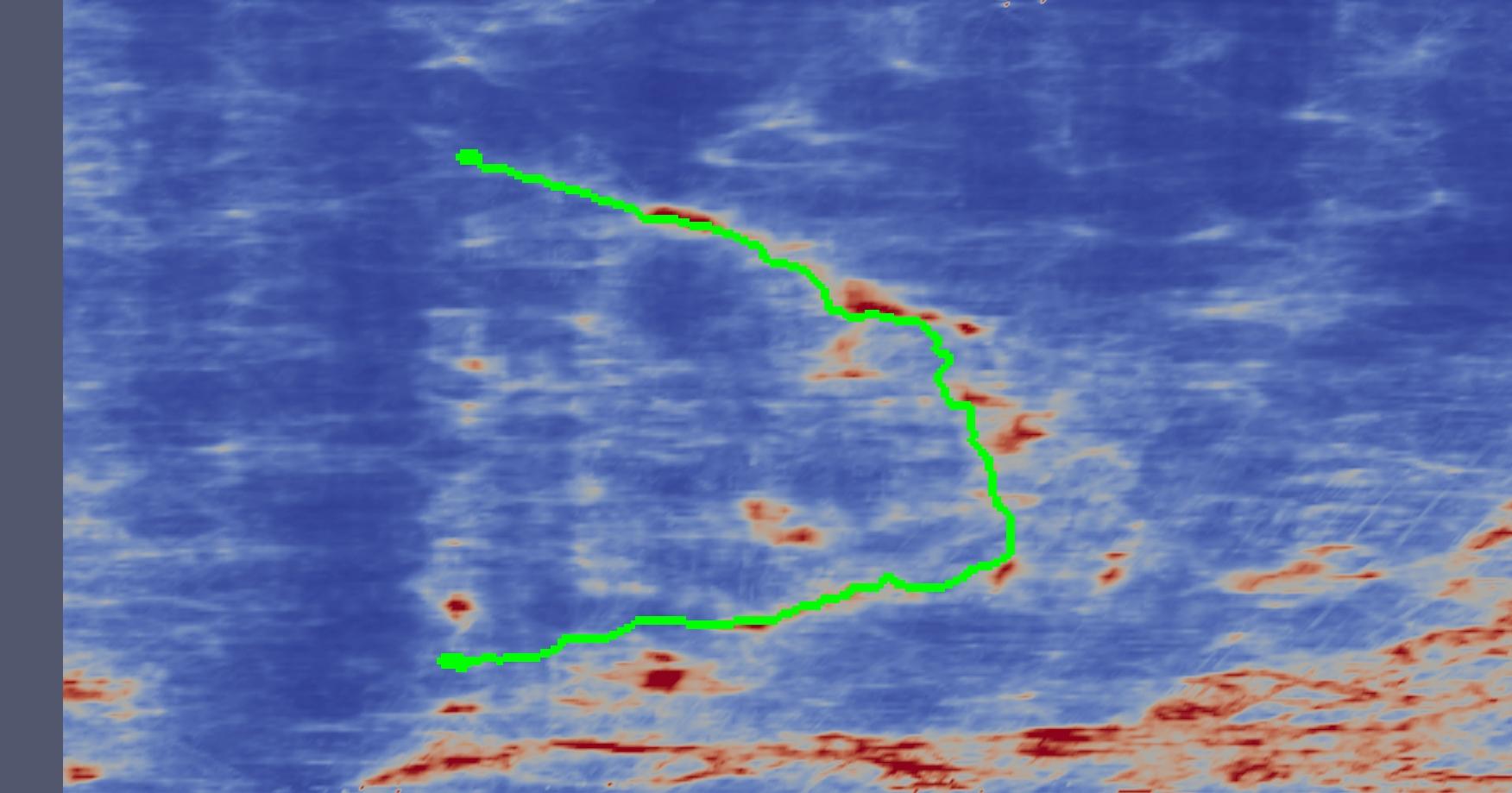}
  \includegraphics[height=\fHeightFIcc,angle=-90,origin=c]{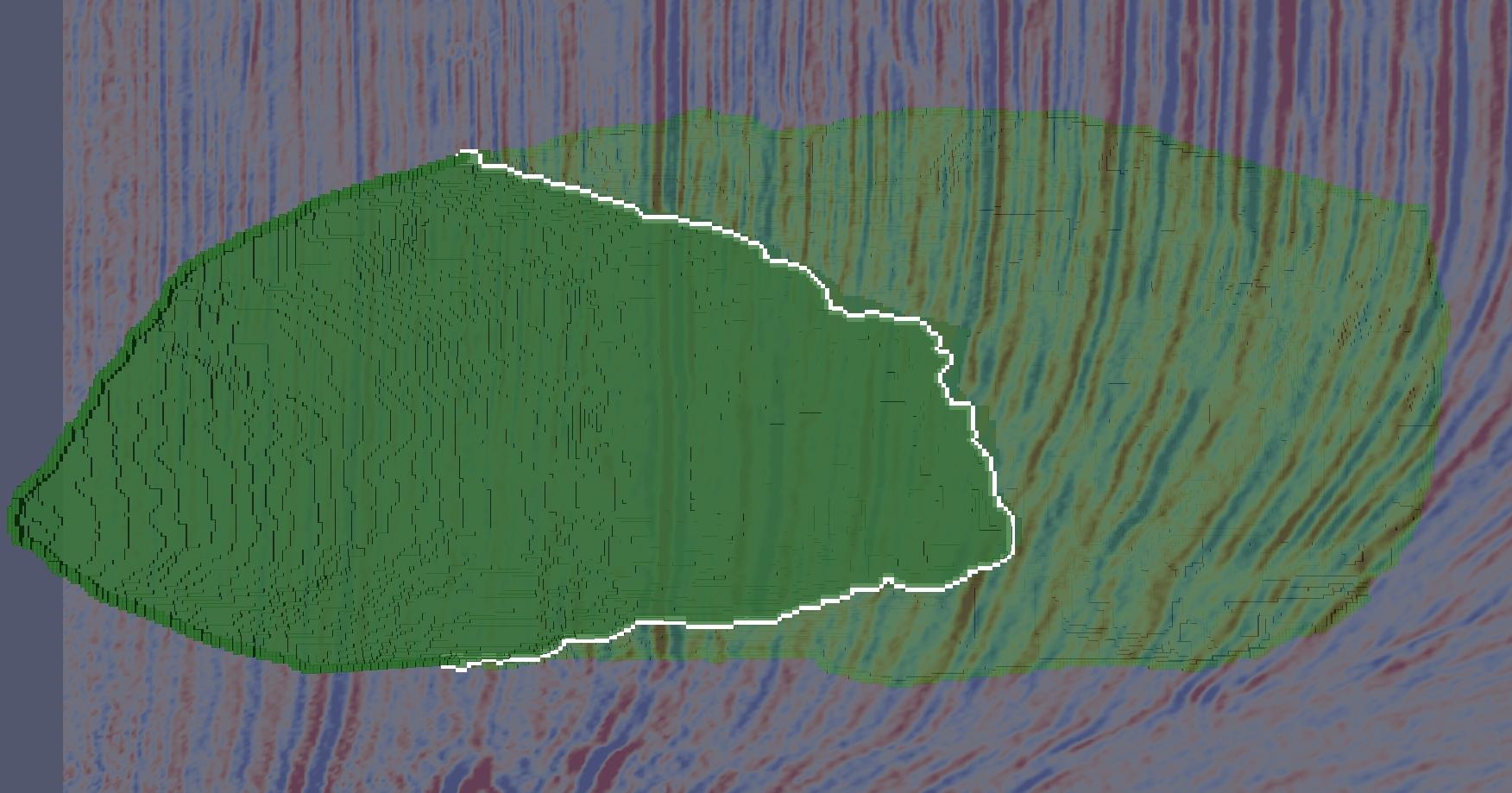}  \\ \vspace{1mm}
  \caption{{\bf Slice-wise Validation on Seismic Data}: [Left column]:
    Slices corresponding to x-y, x-z, and y-z planes, [Middle, left]:
    Local surface likelihood ($1/\phi$), [Middle right]: Intersection
    of SurfCut result (green) with slice, [Right column]: surface from
    SurfCut at certain viewpoint.}
  \label{fig:fault-intersection4}
\end{figure}

{\bf Robustness to Seed-Point Location}: We demonstrate that our
surface extraction method is robust to the choice of the seed point
location. To this end, we randomly sample 30 points (with high local
likelihood) from the ground truth surface. We use each of the points
as seed points to initialize our algorithm. We measure the boundary
and surface accuracy for each of the extracted surfaces. Results are
displayed in Figure~\ref{fig:seedpoint_robust}. They show our
algorithms consistently returns a boundary and surface of similar
accuracy regardless of the seed point location.

\begin{figure}
  \centering
  \includegraphics[totalheight=\fHeightFm]{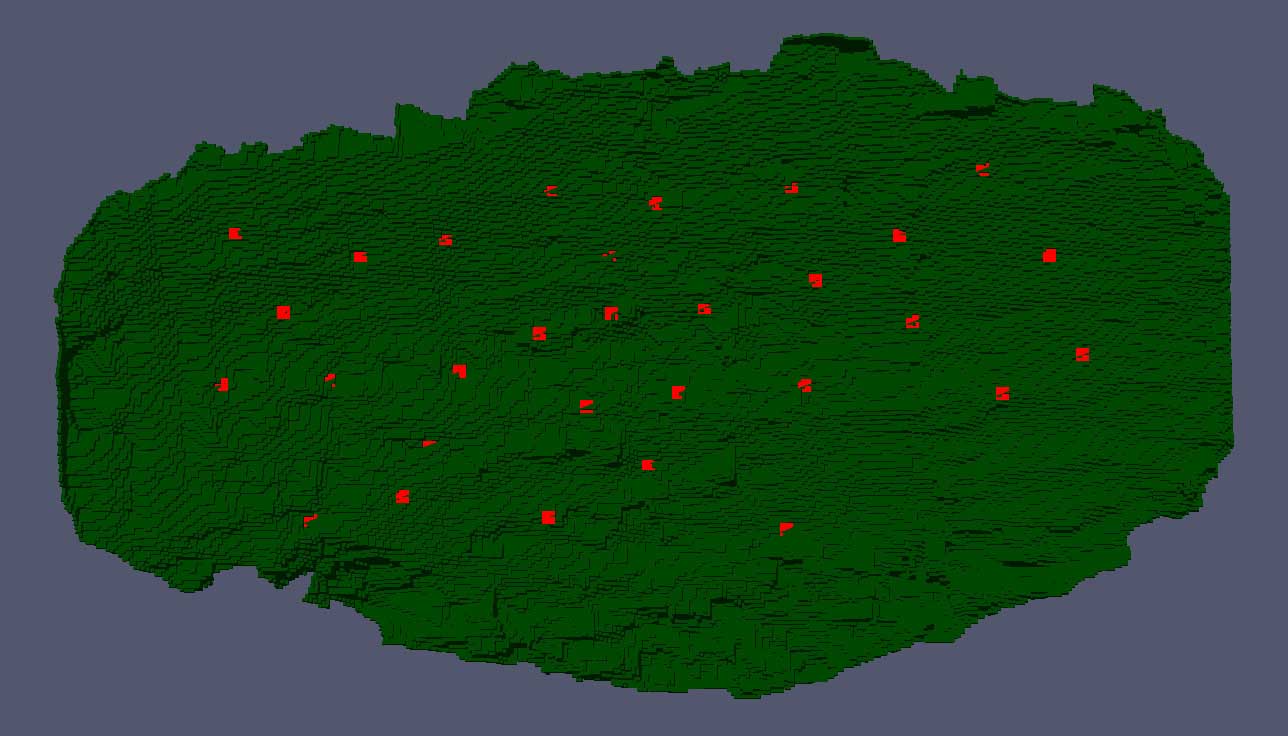}
  \includegraphics[totalheight=\fHeightFm]{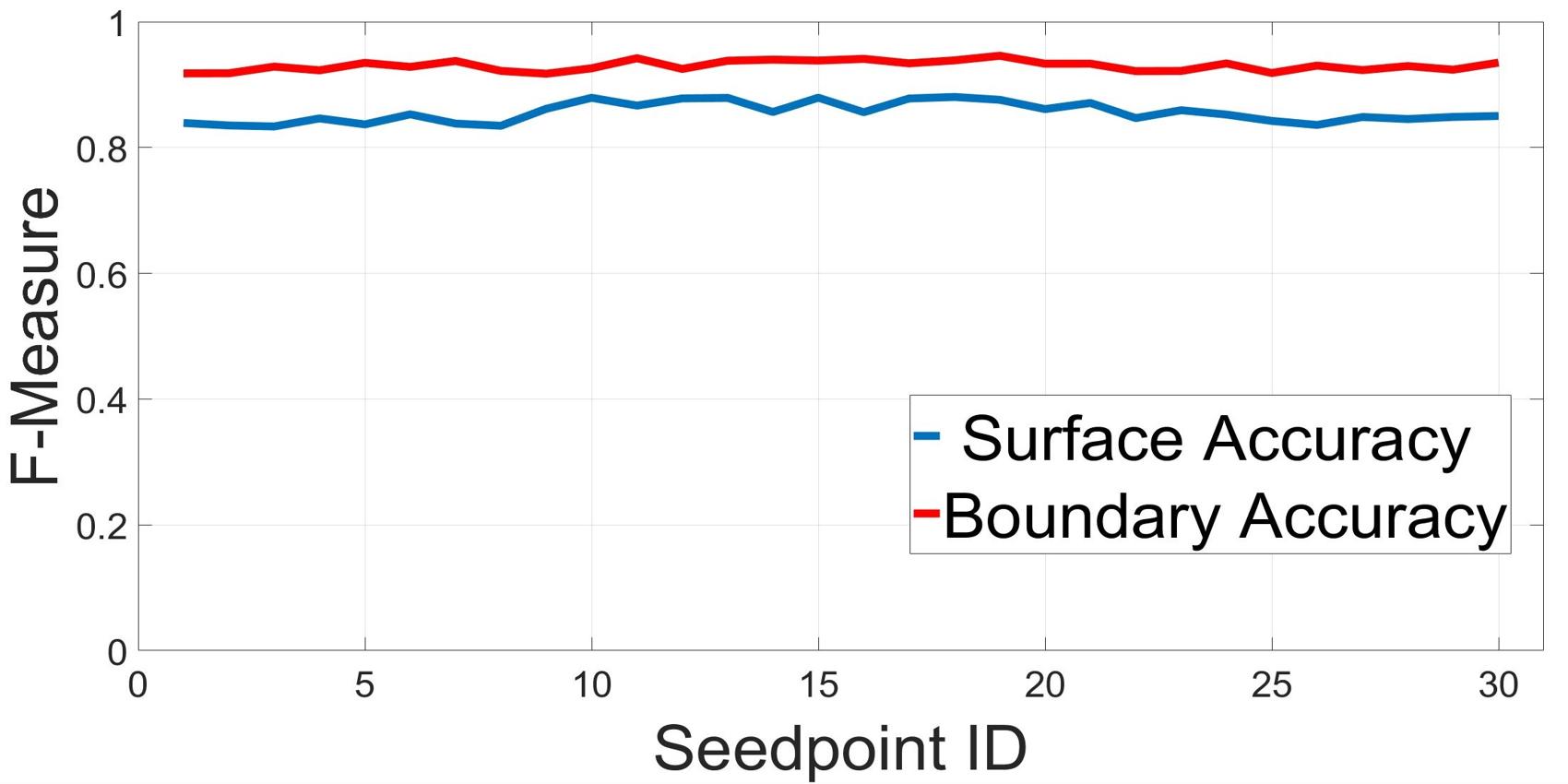}
  \caption{{\bf Robustness to Seed Point Choice}: [Left]: A
    visualization of the seed points chosen.  [Right]: Boundary
    F-measure versus various seed point indices. The same boundary and
    surface accuracy is maintained no matter the seed point location.}
  \label{fig:seedpoint_robust}
\end{figure}

{\bf Analysis of Automated Algorithm}: Even though our contribution is
in the surface and boundary extraction from a seed point, we show with
a seed point initialization, our method can be automated. We
initialize our algorithm with a simple automated detection of seeds
points. We extract seed points by finding extrema of the Hessian and
then running a piece-wise planar segmentation of these points using
RANSAC \cite{rusu20113d} successively; the point on each of the
segments located closest to other points on the segment are seed
points. This operates under the assumption that the surfaces are
roughly planar.  If not, there could possibly be redundant seed points
on the same surface, which would result in repetitions in surfaces in
our final output. This could easily be filtered out.  We run our
boundary curve extraction followed by surface extraction for each of
the seed points on the original datasets. We compare to
\cite{schultz2010crease}. There are 6 ground truth surfaces in this
dataset. Our algorithm correctly extracts 6 surfaces, while
\cite{schultz2010crease} extracts 4 surfaces (2 pairs of faults are
merged together each as a single connected component). Results on a
dataset are visualized in Figure~\ref{fig:fully_automated_results}
(each connected component in different color). Notice that Crease
Surfaces has holes, captures clutter, and connects separate faults.

\begin{figure}
  \centering
  \includegraphics[width=4cm,height=3cm]{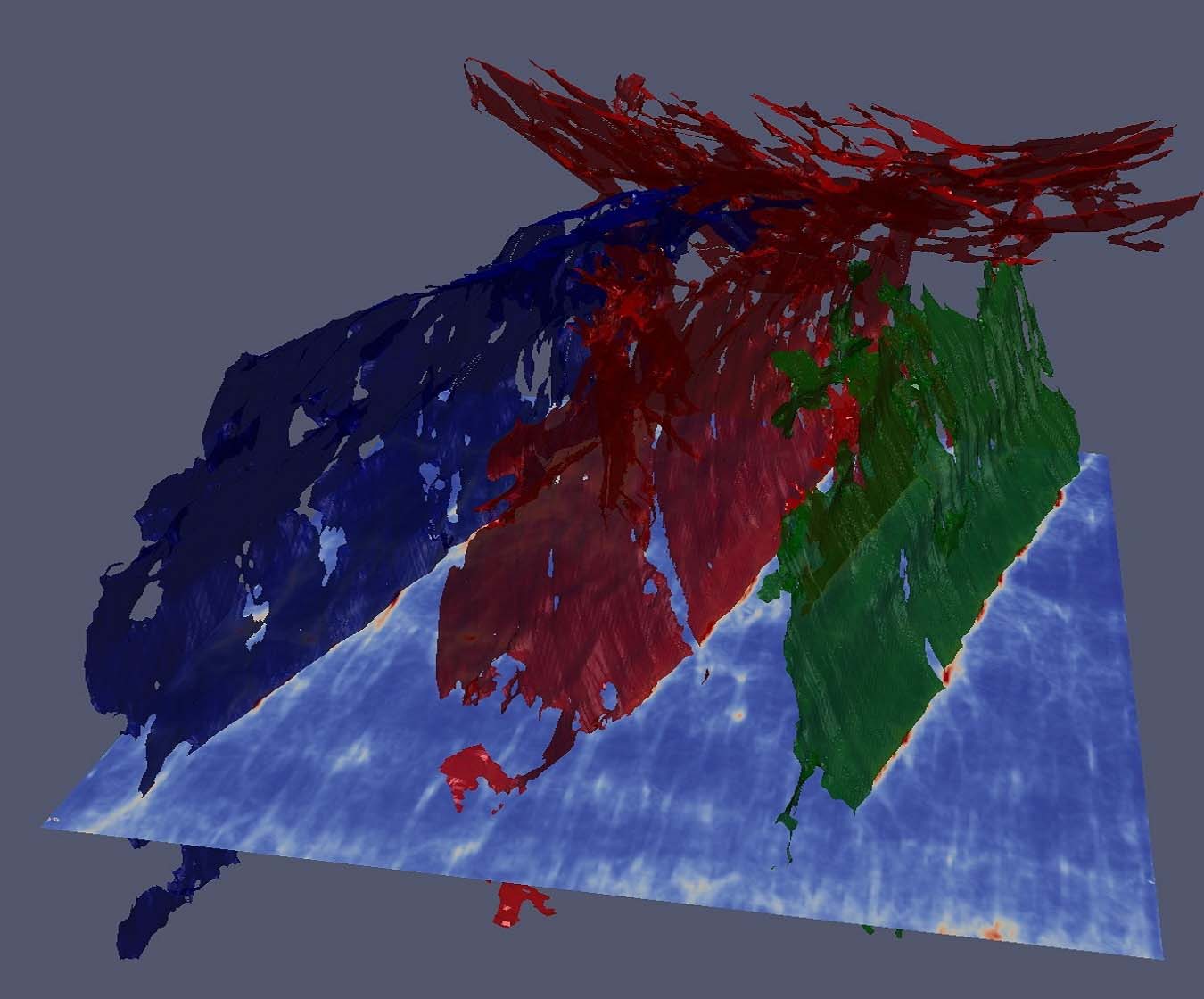}
  \includegraphics[width=4cm,height=3cm]{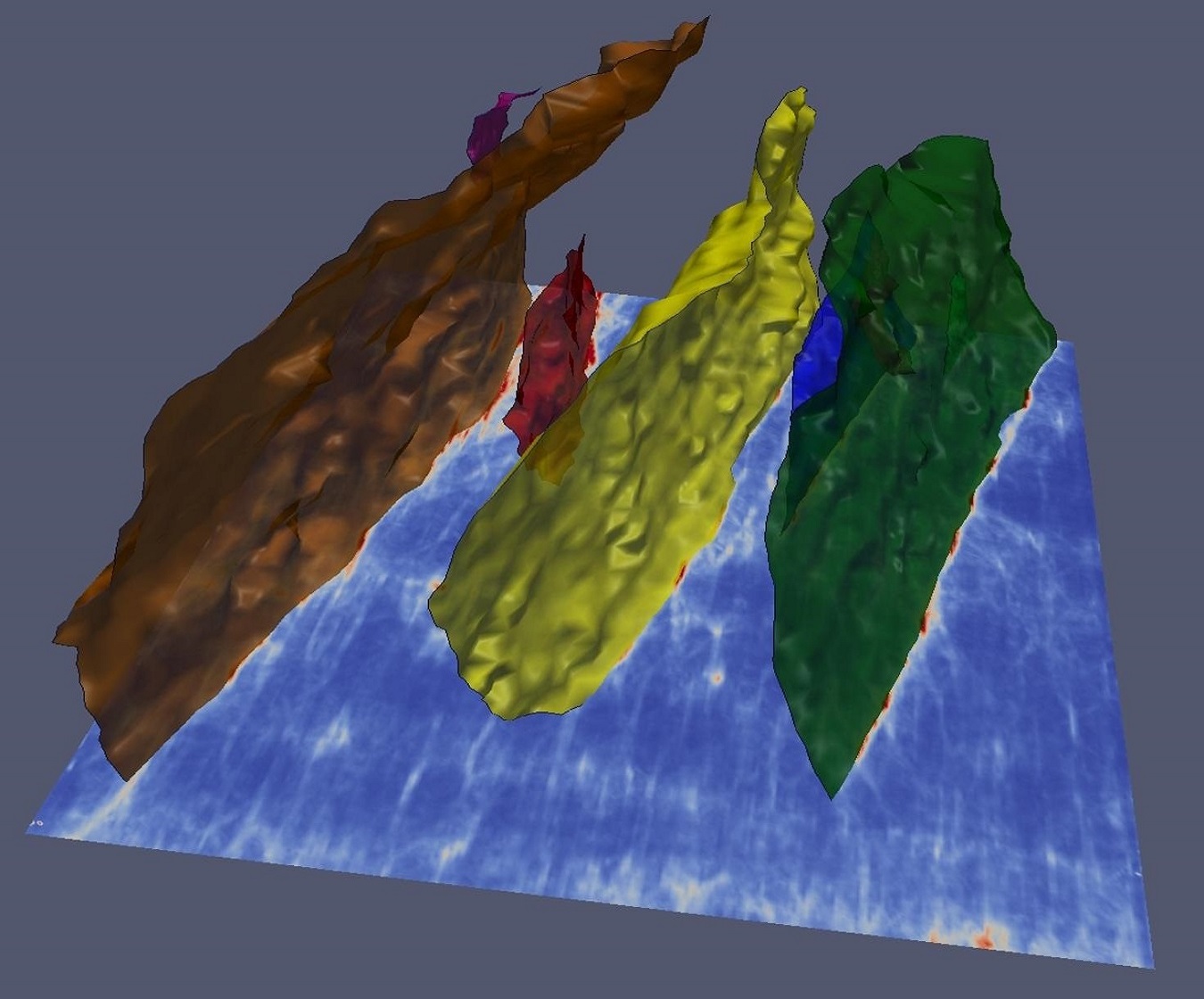} \vspace{0.05in} \\%
  \includegraphics[width=4cm,height=3cm]{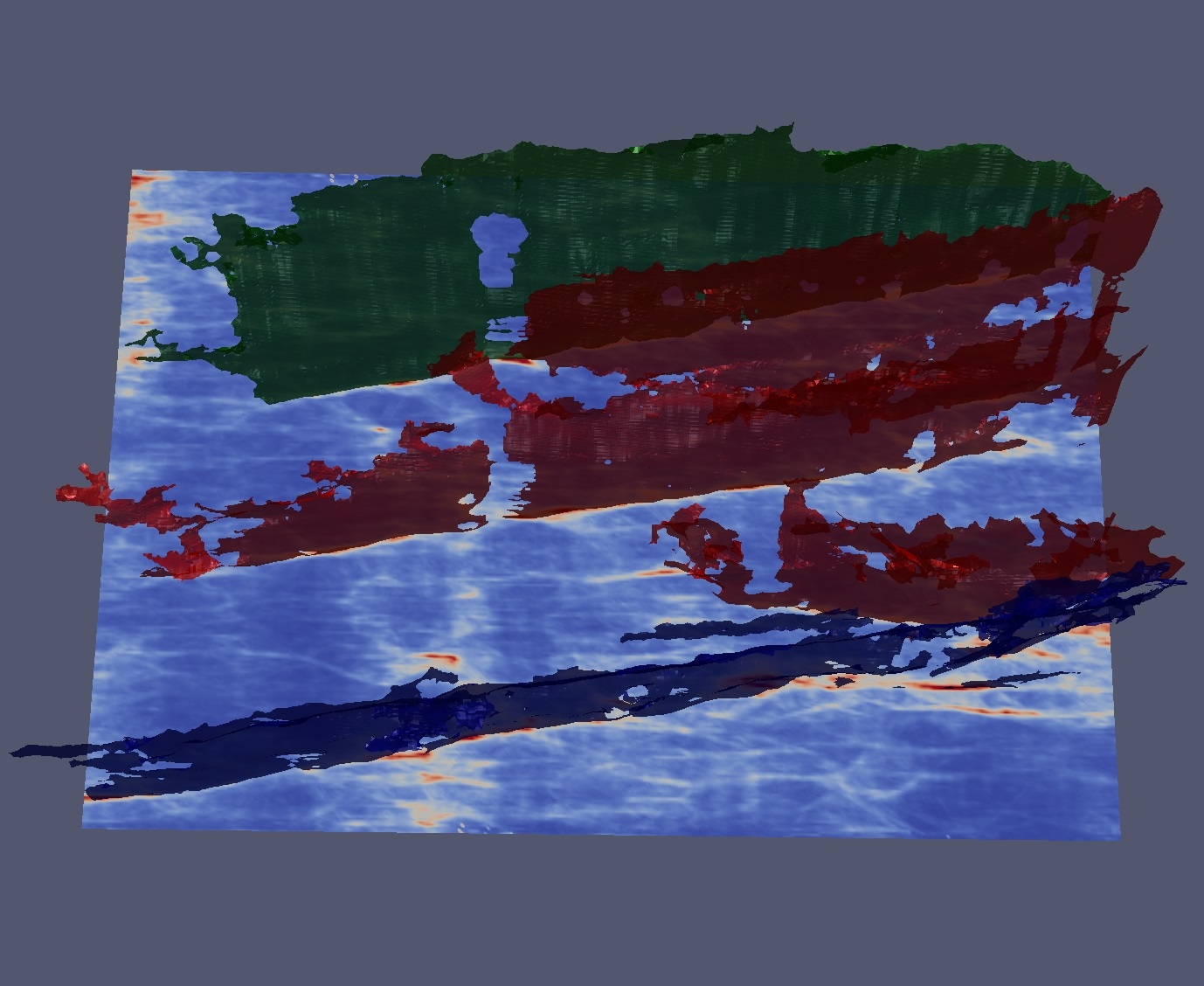}
  \includegraphics[width=4cm,height=3cm]{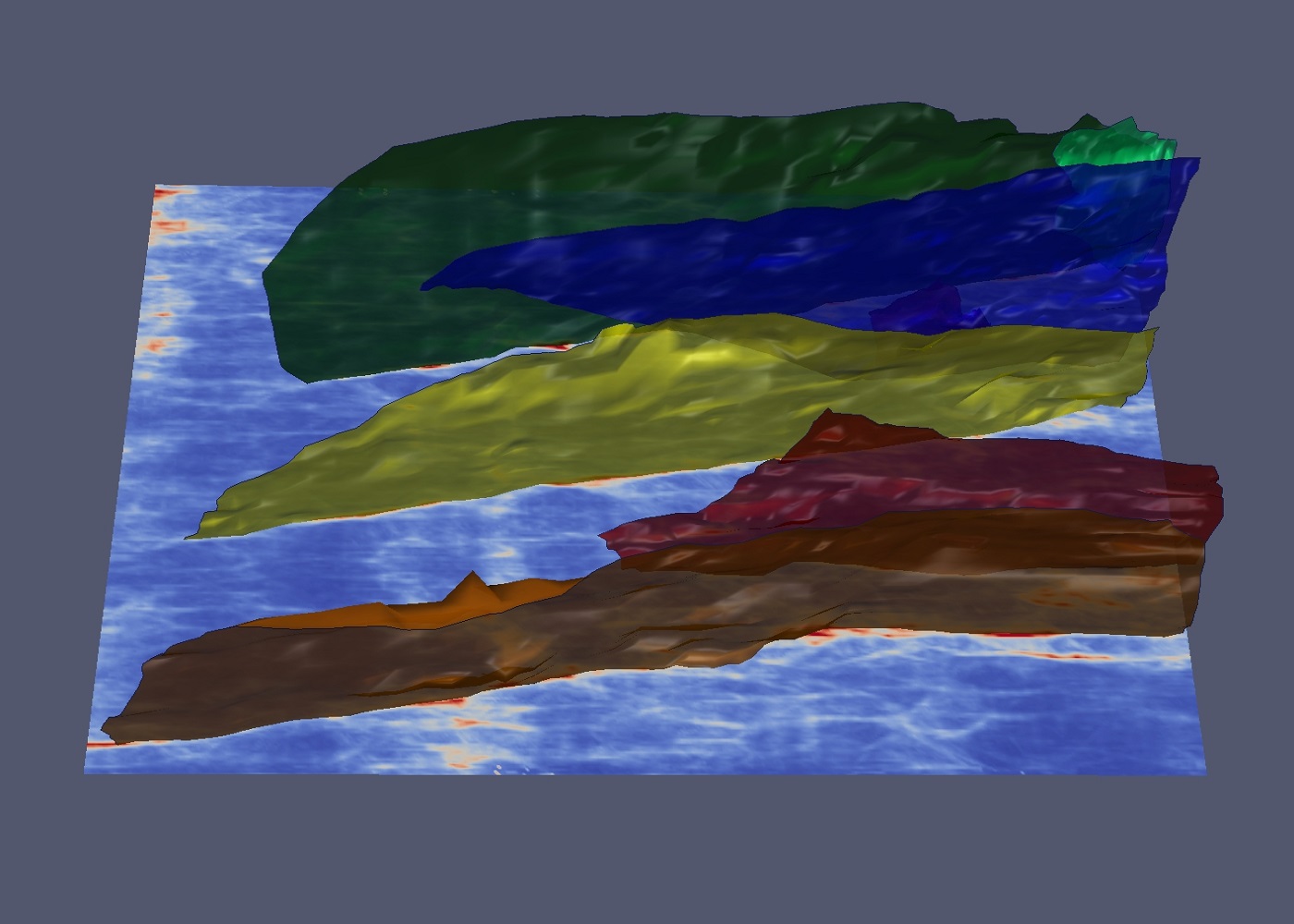}
  \caption{Example result in an automated setting. [Left]: Result by
    Crease Surfaces, which contains holes and incorrectly detects
    clutter (top, red) due to noise in the data. [Right]: Results of
    SurfCut, which extracts the correct number of surfaces and
    produces smooth simple surfaces. \comment{ ({\bf See Supplementary
        Material for video and further visual validation.}) } }
  \label{fig:fully_automated_results}
\end{figure}

{\bf Computational Cost}: We analyze run-times on a dataset of size
$463\times 951 \times 651$. The run-time of our algorithm depends on
the size of the surface. To extract one surface, our algorithm takes
on average $10$ minutes ($9$ minutes for the boundary extraction and
$1$ minute for the surface extraction). Automated seed point
extraction takes about 3 minutes. Therefore, the total cost of our
algorithm for extracting 6 faults is about 1 hour. We note that after
seed point extraction, the computation of surfaces can be
parallelized. In comparison, \cite{schultz2010crease} takes about
Pentium 2.3 GHz processor.
$2$ hours on the same dataset.  Speeds are reported on a single

\comment{
Even though the method
\cite{grady2010minimal} requires manual input of the boundary curve of
the surface, we state the time of \cite{grady2010minimal} for surface
extraction. Using Gurobi's state-of-the-art linear programming
implementation, the method takes over 10 hours for a \emph{single}
surface (and the time grows drastically with increasing image
sizes). Ours takes 1 minute given the boundary (both achieve similar
accuracies). Our solution may not achieve the minimal surface as in
\cite{grady2010minimal}, but it does achieve a surface with high
fidelity to the surface of interest. 
}

\subsubsection{Lung CT Data: Surface and Boundary Extraction}

We now compare to Crease Surfaces for the Lung CT dataset. We compare
the methods under the settings described in the previous section. For
medical data, we modify the matrix based on the Hessian in Crease
Surfaces to another matrix based on closeness to a plate-like
structure as common in lung fissure detection
\cite{wiemker2005unsupervised,van2008supervised,xiao2016pulmonary}. State-of-the-art
methods in fissure extraction use a method similar to Crease surfaces
to extract the surface. We choose $\phi$ to be the plate-ness measure in
our method. Quantitative results on the entire dataset are summarized
in Table~\ref{tab:lung_data}. Both in terms of surface and boundary
accuracy, our method is more accurate with respect to all
measures. Visual validation of our method on slice-wise views of the
surface and image is shown in Fig.~\ref{fig:lung-intersection}.  Some
visualizations of the surface results are shown in
Fig.~\ref{fig:lung_result}.  Various slices are shown to help
visualize features of the image. Crease surface generates surfaces
with incorrect holes and many times cannot capture the entire fissure,
hence low recall and precision on the boundary metrics. SurfCut does
not contain any holes and accurately captures very fine and thin
structures near the boundaries of the fissures.

\comment{
The eignvalues of
the Hessian are caluclated at each pixel, and $\phi$ is computed by
the difference of the first and second highest eigenvalues divided by
the sum.
}

\begin{table}
  \caption{{\bf Quantitative Evaluation on Lung Dataset}. Comparison
    of methods in terms of surface and boundary accuracy. Precision
    (P), recall (R), F-measure (F), and ground truth covering (GT-cov)
    are reported. Higher P, R, F, GT-Cov.~indicate better fidelity to the ground truth.}
  \label{tab:lung_data} 
  \begin{center}
    Surface accuracy \\
    \begin{tabular}{ l | l | l | l  | l }
      Method & $F$ &  GT-Cov. &  $P$ &  $R$ \\
      \hline
      Crease Surfaces & 0.76$\pm$0.08 & 0.70$\pm$0.10 & 0.67$\pm$0.11 & 0.91$\pm$0.06    \\
      Surfcut & {\bf 0.91$\pm$0.04} &  {\bf 0.87$\pm$0.06} &  {\bf 0.86$\pm$0.06} &  {\bf 0.95$\pm$0.02}   \\
      \hline
    \end{tabular}\\\vspace{0.1in}
    Boundary accuracy \\
    \begin{tabular}{ l | l | l | l  | l }
      Method & $F$ &  GT-Cov. &  $P$ &  $R$ \\
      \hline
      Crease Surfaces & 0.70$\pm$0.11 & 0.72$\pm$0.08 & 0.69$\pm$0.10 & 0.71$\pm$0.12   \\
      Surfcut &  {\bf 0.86$\pm$0.04} &  {\bf 0.86$\pm$0.06} &  {\bf 0.85$\pm$0.06} &  {\bf 0.87$\pm$0.05}   \\
      \hline
    \end{tabular}\\
  \end{center}
\end{table}

\comment{
\def\fHeightLI{1.2in}
\begin{figure}
  \centering
  \includegraphics[totalheight=\fHeightLI]{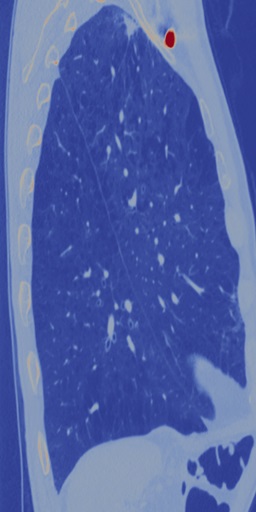}
  \includegraphics[totalheight=\fHeightLI]{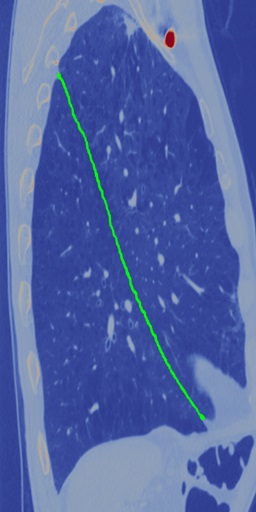}
  \includegraphics[totalheight=\fHeightLI]{./figures/lung_intersection2/3d_001} \\ \vspace{1mm}
  \includegraphics[totalheight=\fHeightLI]{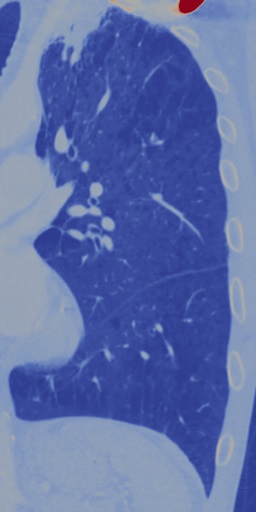}
  \includegraphics[totalheight=\fHeightLI]{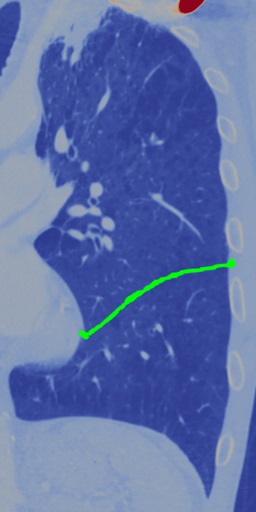}
  \includegraphics[totalheight=\fHeightLI]{./figures/lung_intersection2/3d_002} \\ \vspace{1mm}
  \includegraphics[totalheight=\fHeightLI]{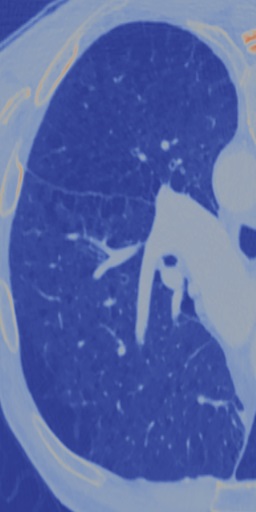}
  \includegraphics[totalheight=\fHeightLI]{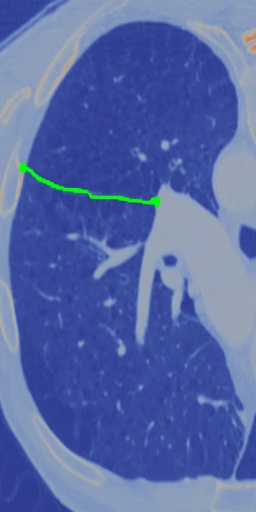}
  \includegraphics[totalheight=\fHeightLI]{./figures/lung_intersection2/3d_003}
  \caption{Validation on slices (first column) that intersect the
    surface computed with SurfCut (green, second column) passes
    through locations of high likelihood of the true surface (red regions).}
  \label{fig:lung-intersection}
\end{figure}
}

\def\fHeightLI{1.1in}
\begin{figure}
  \centering
  \includegraphics[totalheight=\fHeightLI]{./figures/lung_intersection2/x370_slice}
  \includegraphics[totalheight=\fHeightLI]{./figures/lung_intersection2/y255_slice}
  \includegraphics[totalheight=\fHeightLI]{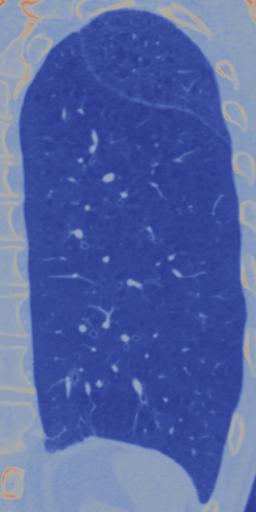}
  \includegraphics[totalheight=\fHeightLI]{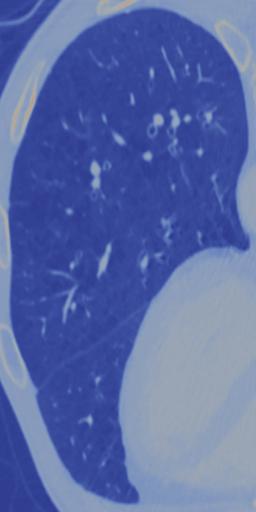}
  \includegraphics[totalheight=\fHeightLI]{./figures/lung_intersection2/z350_slice}
\\ \vspace{1mm}
  \includegraphics[totalheight=\fHeightLI]{./figures/lung_intersection2/x370_surfcut}
  \includegraphics[totalheight=\fHeightLI]{./figures/lung_intersection2/y255_surfcut}
  \includegraphics[totalheight=\fHeightLI]{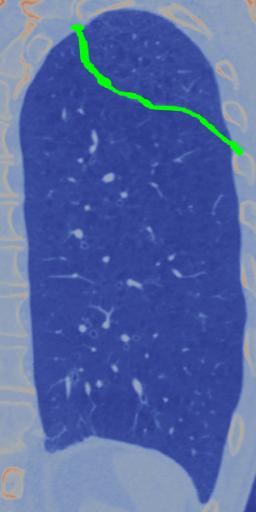}
  \includegraphics[totalheight=\fHeightLI]{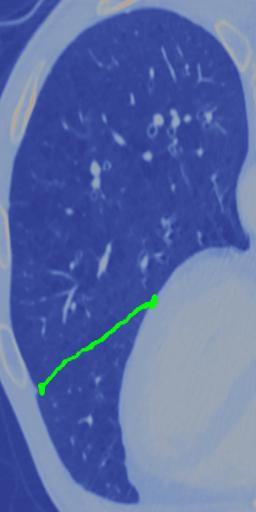}
  \includegraphics[totalheight=\fHeightLI]{./figures/lung_intersection2/z350_surfcut}
  \caption{ {\bf Slice-wise Validation in Lung CT Dataset}. [Top]:
    Various slices of an image of a patient, [Bottom]: Surface
    generated with SurfCut intersected with the slice above (green)
    superimposed on the slice. Notice the structure of interest is a
    subtle thin lines in the slices (top).}
  \label{fig:lung-intersection}
\end{figure}

\def\fHeightLungResult{0.88in}
\begin{figure}
  \centering
  {Ground Truth \\
    \includegraphics[totalheight=\fHeightLungResult]{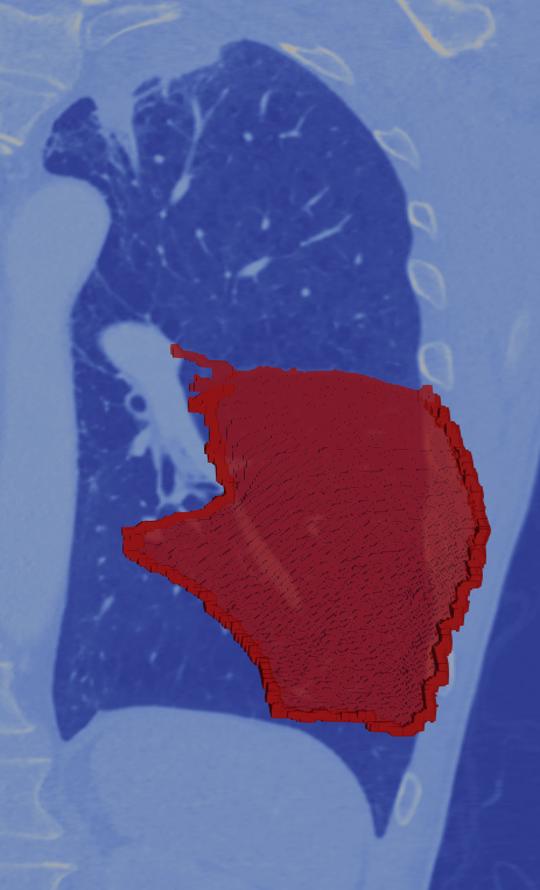} 
    \includegraphics[totalheight=\fHeightLungResult]{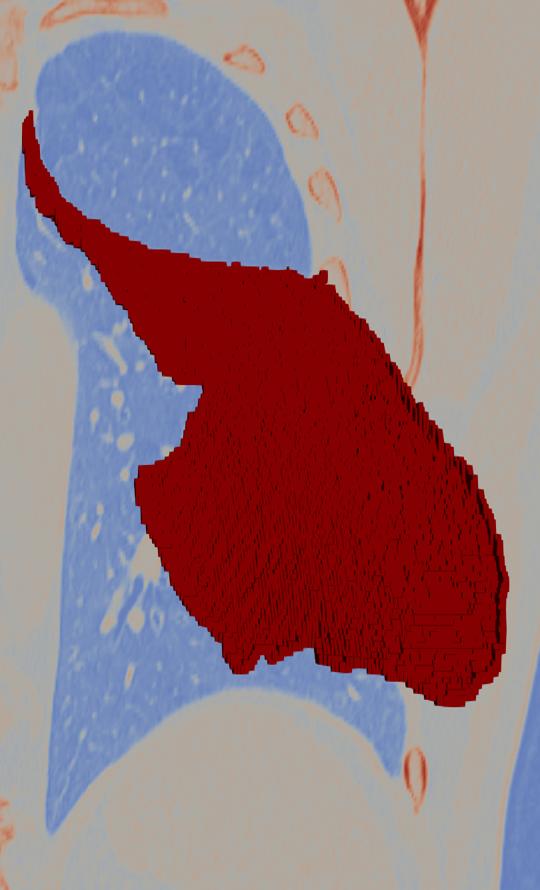} 
    \includegraphics[totalheight=\fHeightLungResult]{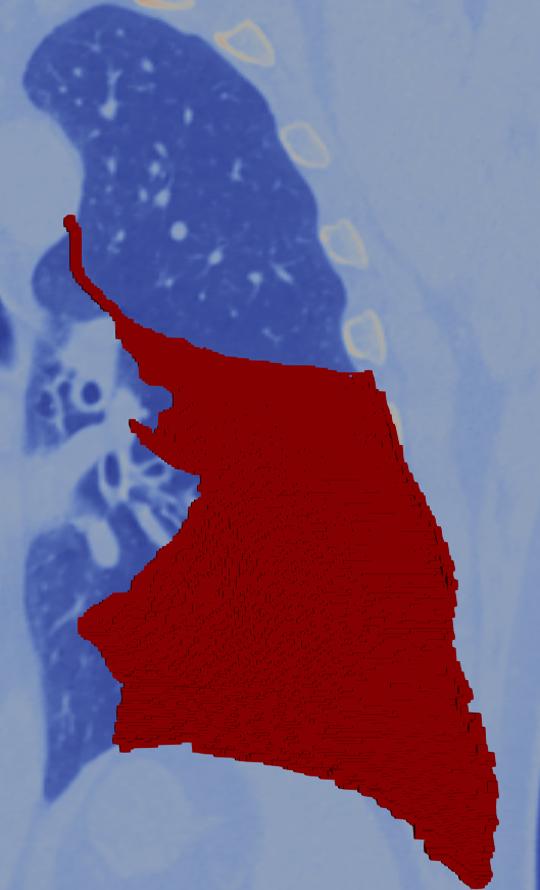} 
    \includegraphics[totalheight=\fHeightLungResult]{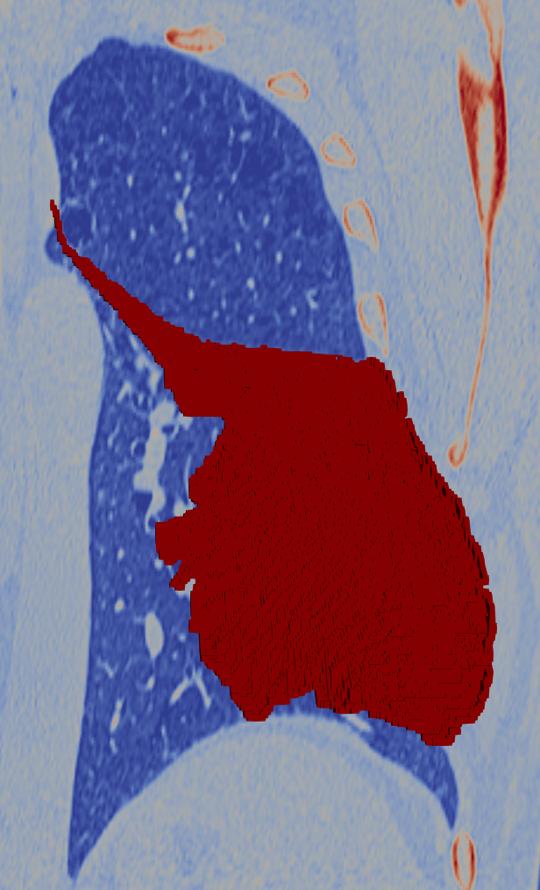} 
    \includegraphics[totalheight=\fHeightLungResult]{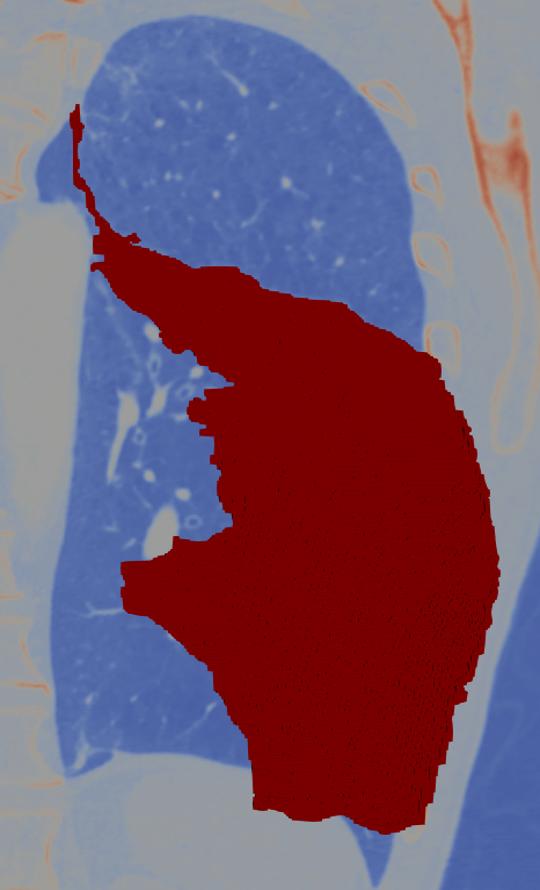} 
    \includegraphics[totalheight=\fHeightLungResult]{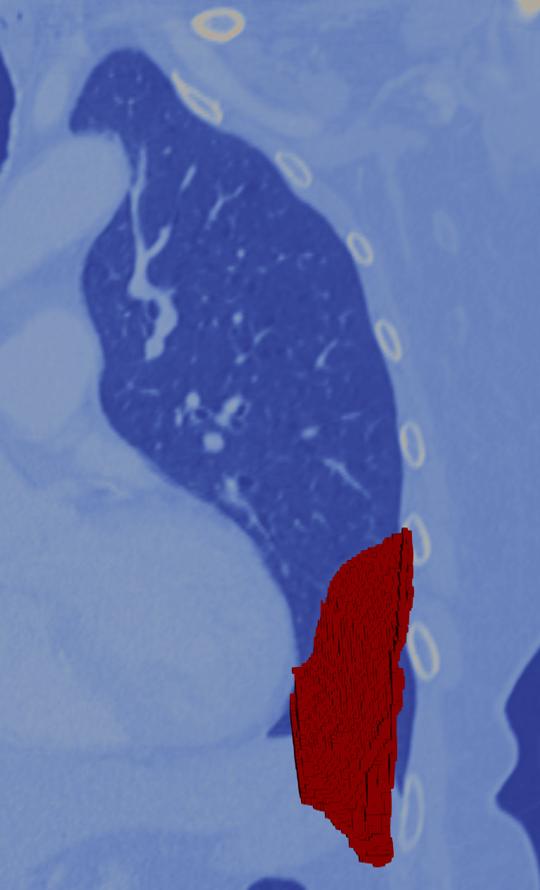} \\
    Crease Surface \\
    \includegraphics[totalheight=\fHeightLungResult]{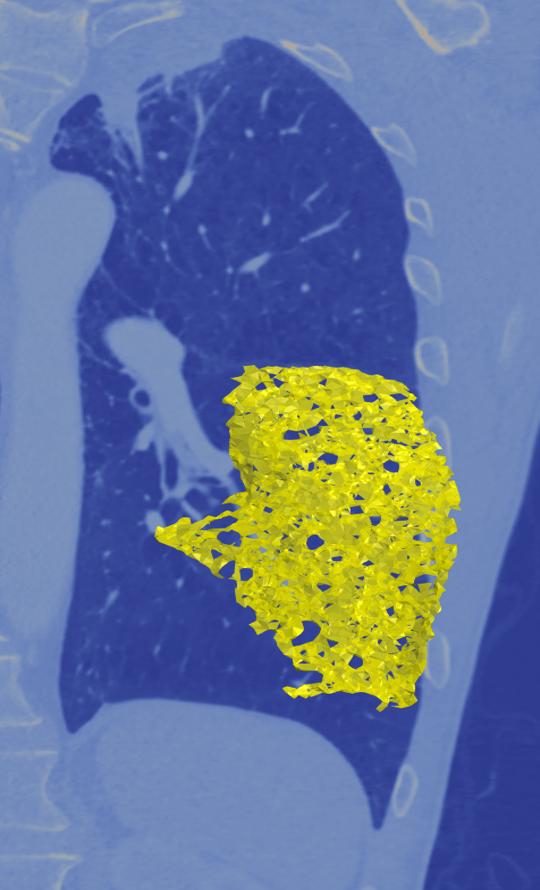} 
    \includegraphics[totalheight=\fHeightLungResult]{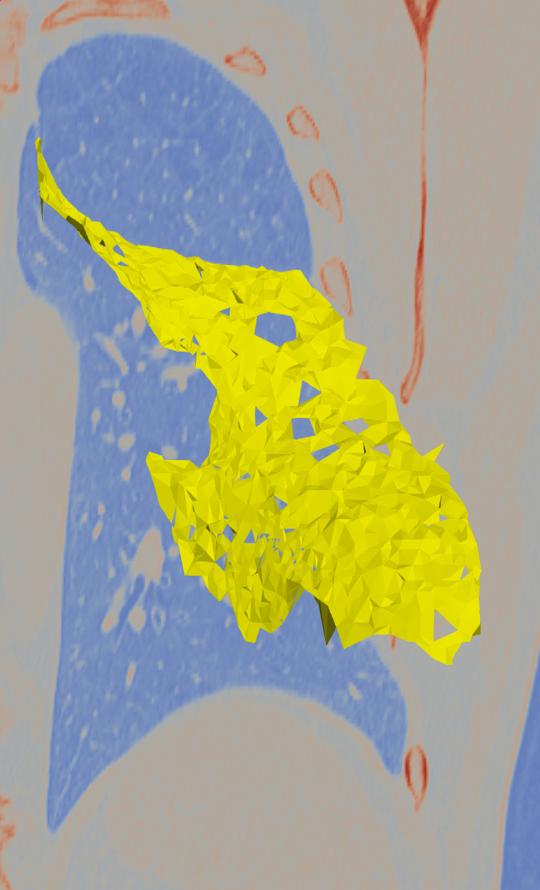} 
    \includegraphics[totalheight=\fHeightLungResult]{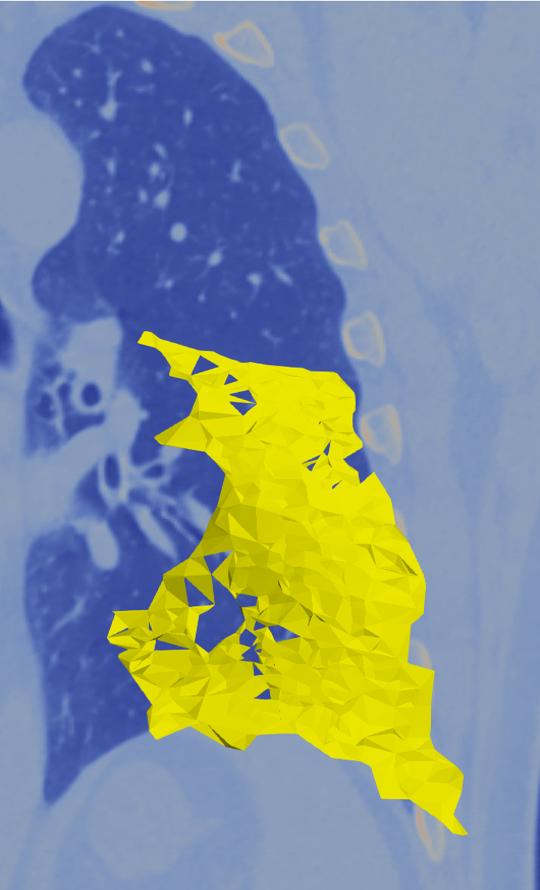} 
    \includegraphics[totalheight=\fHeightLungResult]{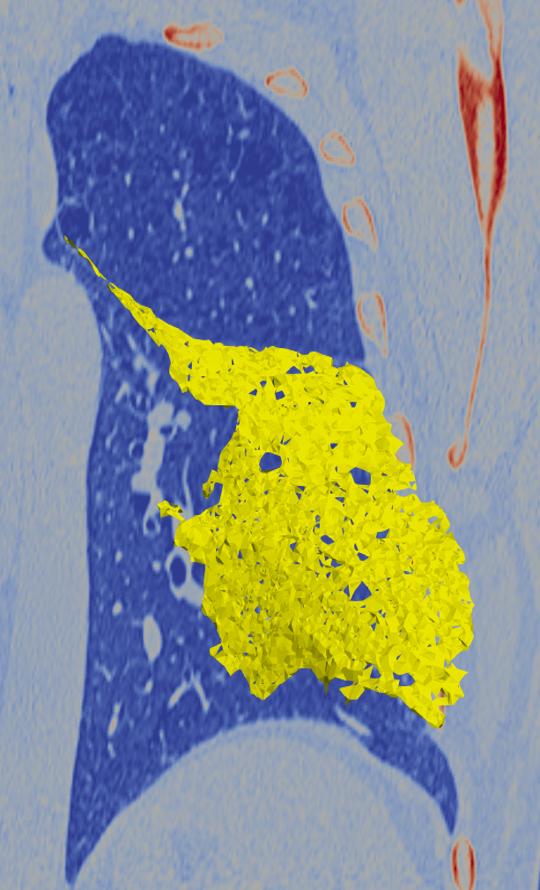} 
    \includegraphics[totalheight=\fHeightLungResult]{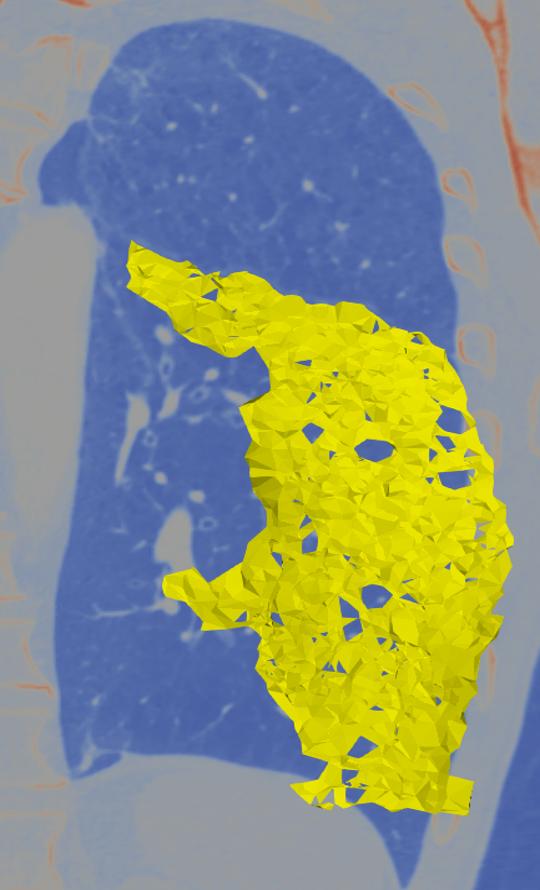} 
    \includegraphics[totalheight=\fHeightLungResult]{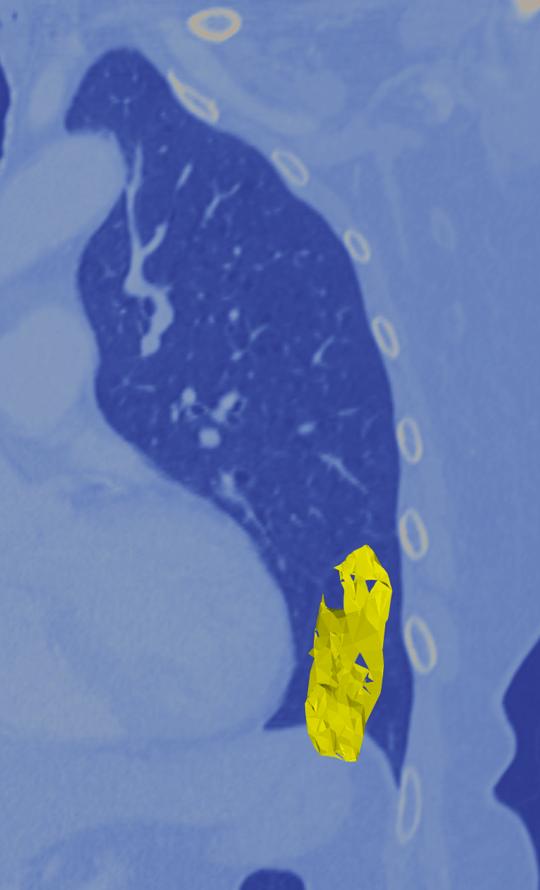} \\
    SurfCut \\
    \includegraphics[totalheight=\fHeightLungResult]{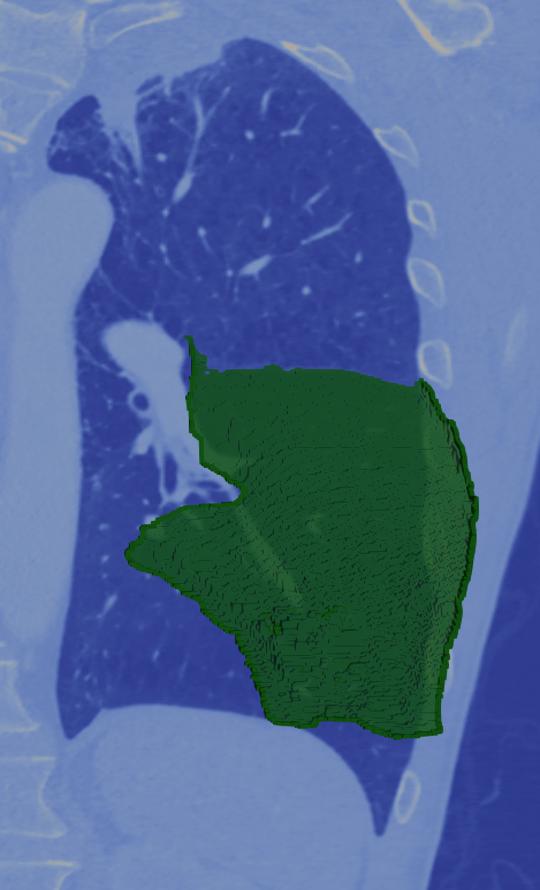} 
    \includegraphics[totalheight=\fHeightLungResult]{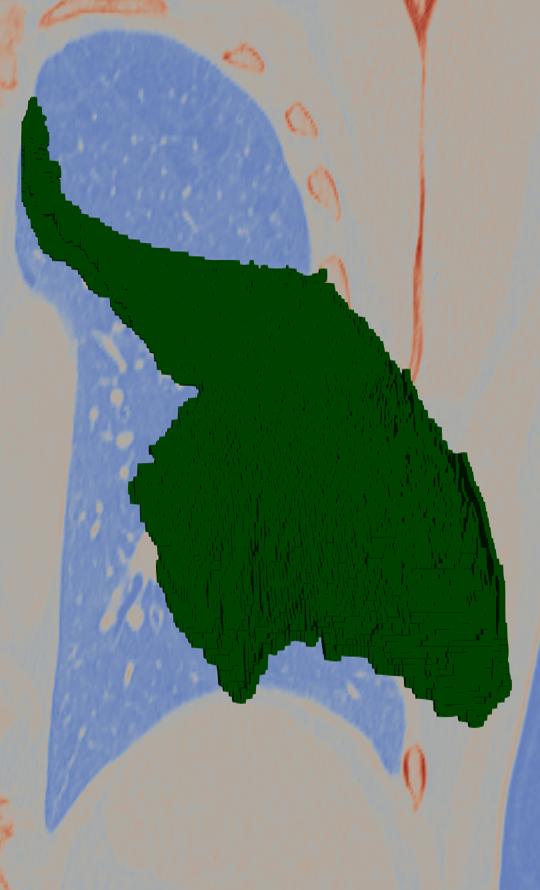} 
    \includegraphics[totalheight=\fHeightLungResult]{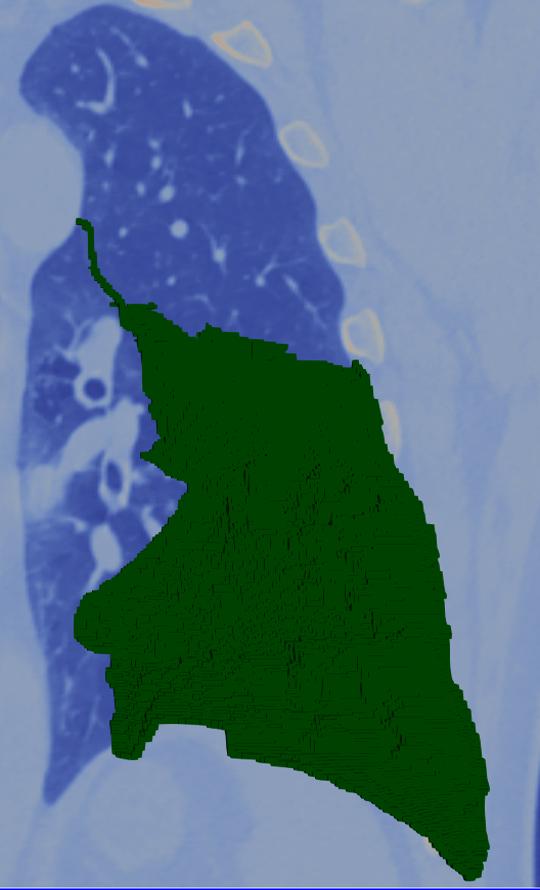} 
    \includegraphics[totalheight=\fHeightLungResult]{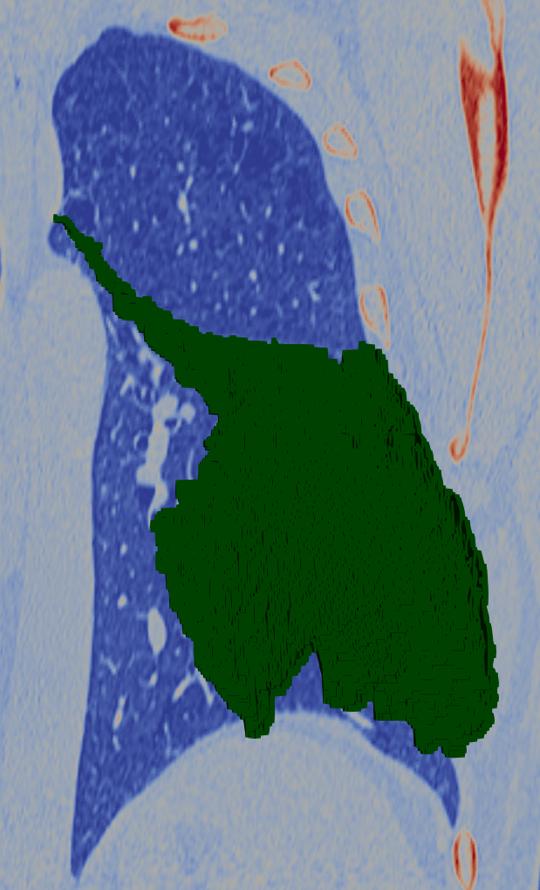} 
    \includegraphics[totalheight=\fHeightLungResult]{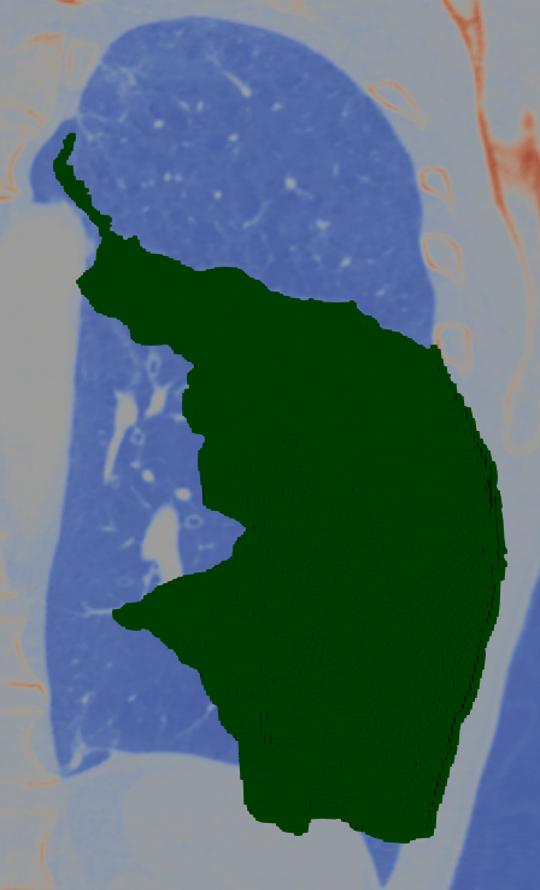} 
    \includegraphics[totalheight=\fHeightLungResult]{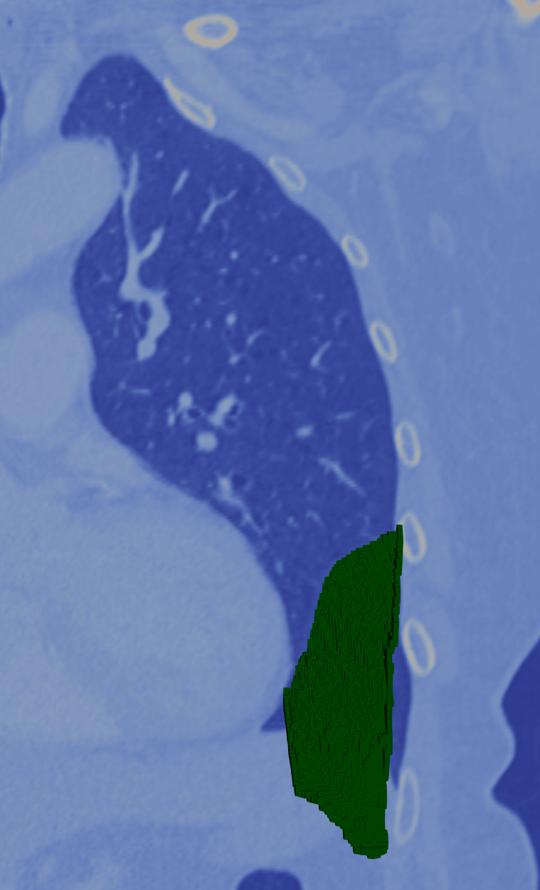} \\
}
\caption{{\bf Qualitative Results on Lung CT}. Columns show the
  surfaces on the same slice on the same patient for various
  methods. Moving through a row shows the surface for different
  patients, and a slice of the image is shown for various different
  slices. SurfCut extracts more of the fine structure of the fissures,
  better estimates the boundary, and recovers more of the surfaces
  than Crease surfaces.}
\label{fig:lung_result}
\end{figure}

\section{Conclusion}
We have provided a general method for extracting a smooth simple
(without holes) surface with unknown boundary in a 3D image with noisy
local measurements of the surface, e.g., edges. Our novel method takes
as input a single seed point, and extracts the unknown boundary that
may lie in 3D. It then uses this boundary curve to determine the
entire surface efficiently. We have demonstrated with extensive
experiments on noisy and corrupted data with possible interruptions
that our method accurately determines both the boundary and the
surface, and the method is robust to seed point choice. In comparison
to an approach which extracts connected components of edges in 3D
images, our method is more accurate in both surface and boundary
measures. The computational cost of our algorithm is less than
competing approaches.

A limitation of our method (as with competing methods) is extracting
intersecting surfaces. Our boundary extraction method may extract
boundaries of one or both of the surfaces depending on the
data. However, if given the correct boundary of one of the surfaces,
our surface extraction produces the relevant surface. This limitation
of our boundary extraction is the subject of future work. This is
important in seismic images, where surfaces can intersect.

% or


\ifCLASSOPTIONcompsoc
  \section*{Acknowledgments}
\else
  \section*{Acknowledgment}
\fi

This work was supported by KAUST OCRF-2014-CRG3-62140401, and the
Visual Computing Center at KAUST.

\ifCLASSOPTIONcaptionsoff
  \newpage
\fi



%\bibliographystyle{IEEEtran}

\bibliographystyle{IEEEtran}
\bibliography{References}



%\vfill

%[{\includegraphics[totalheight=1.25in]{figures2/ganesh_pami}}]

\begin{IEEEbiography}{Marei Algarni}
  received the BS degree in Computer Science from KAU, Saudi Arabia,
  and MSc with Merit from the University of Bradford/UK, 2008. He then
  worked at Saudi Aramco in the Exploration Application Service
  Department. He is currently working toward his PhD degree in the
  Department of Computer Science, KAUST (King Abdullah University of
  Science and Technology). His research interests include computer
  vision with particular interest in segmentation of 3D scientific
  datasets.
\end{IEEEbiography}

\begin{IEEEbiography}{Ganesh Sundaramoorthi}
  received the PhD in Electrical and Computer Engineering from Georgia
  Institute of Technology, Atlanta, USA. He was then a postdoctoral
  researcher in the Computer Science Department at the University of
  California, Los Angeles between 2008 and 2010. In 2011, he was
  appointed Assistant Professor of Electrical Engineering and
  Assistant Professor of Applied Mathematics and Computational Science
  at King Abdullah University of Science and Technology (KAUST). His
  research interests include computer vision and its mathematical
  foundations with recent interest in shape and motion analysis, video
  analysis, invariant representations for visual tasks, and
  applications in medical and scientific imaging. He is an area chair for
  IEEE ICCV 2017.
\end{IEEEbiography}

\end{document}